\documentclass{article}

\usepackage{microtype}
\usepackage{graphicx}
\usepackage{subcaption}
\usepackage{booktabs} 

\usepackage{hyperref}

\usepackage[preprint]{style}

\usepackage{amsmath}
\usepackage{amssymb}
\usepackage{mathtools}
\usepackage{amsthm}
\usepackage{physics}
\usepackage{dblfloatfix}

\usepackage[capitalize,noabbrev]{cleveref}

\theoremstyle{plain}
\newtheorem{theorem}{Theorem}[section]
\newtheorem{proposition}[theorem]{Proposition}
\newtheorem{lemma}[theorem]{Lemma}
\newtheorem{corollary}[theorem]{Corollary}
\theoremstyle{definition}
\newtheorem{definition}[theorem]{Definition}
\newtheorem{assumption}[theorem]{Assumption}
\theoremstyle{remark}

\newcommand{\PL}{\mathrm{PL}}
\newcommand{\Curv}{\mathrm{Curv}}

\newcommand{\Sspace}{\mathcal{S}}
\newcommand{\Aspace}{\mathcal{A}}

\newcommand{\R}{\mathbb{R}}
\newcommand{\E}{\mathbb{E}}

\newcommand{\Wone}{\mathsf{W}_1}

\newcommand{\Tcal}{\mathcal{T}}

\DeclareMathOperator*{\argmax}{arg\,max}

\usepackage[textsize=tiny]{todonotes}

\icmltitlerunning{Geometry of Drifting MDPs with Path-Integral Stability Certificates}

\begin{document}

\twocolumn[

\icmltitle{Geometry of Drifting MDPs with Path-Integral Stability Certificates}

  \begin{icmlauthorlist}
    \icmlauthor{Zuyuan Zhang}{gwu}
    \icmlauthor{Mahdi Imani}{northeastern}
    \icmlauthor{Tian Lan}{gwu}
  \end{icmlauthorlist}

  \icmlaffiliation{gwu}{Department of Electrical and Computer Engineering, 
The George Washington University, Washington, DC 20052 USA }
  \icmlaffiliation{northeastern}{Department of Electrical and Computer Engineering, Northeastern University, Boston, MA, USA}

  \icmlcorrespondingauthor{Zuyuan Zhang}{zuyuan.zhang@gwu.edu}
  \icmlcorrespondingauthor{Tian Lan}{tlan@gwu.edu}

  \vskip 0.3in
]

\printAffiliationsAndNotice{}  

\begin{abstract}
Real-world reinforcement learning is often \emph{nonstationary}: rewards and dynamics drift, accelerate, oscillate, and trigger abrupt switches in the optimal action. Existing theory often represents nonstationarity with coarse-scale models that measure \emph{how much} the environment changes, not \emph{how} it changes locally---even though acceleration and near-ties drive tracking error and policy chattering. We take a geometric view of nonstationary discounted Markov Decision Processes (MDPs) by modeling the environment as a differentiable homotopy path and tracking the induced motion of the optimal Bellman fixed point. This yields a length--curvature--kink signature of intrinsic complexity: cumulative drift, acceleration/oscillation, and action-gap-induced nonsmoothness. We prove a solver-agnostic path-integral stability bound and derive gap-safe feasible regions that certify local stability away from switch regimes. Building on these results, we introduce \textit{Homotopy-Tracking RL (HT-RL)} and \textit{HT-MCTS}, lightweight wrappers that estimate replay-based proxies of length, curvature, and near-tie proximity online and adapt learning or planning intensity accordingly. Experiments
show improved tracking and dynamic regret over matched static baselines, with the largest gains in oscillatory and switch-prone regimes.
\end{abstract}

\section{Introduction}

Reinforcement learning (RL) theory is largely developed for a \emph{fixed} Markov decision process (MDP), where the goal is to converge to a single Bellman fixed point  \citep{puterman2014markov,levin2017markov,sutton1998reinforcement,auer2008near,tsitsiklis1996analysis,watkins1992q,mnih2015human,kocsis2006bandit,silver2016mastering}.
In many real-world systems, however, environments are \emph{nonstationary}: rewards drift as goals and user populations evolve, dynamics change as platforms are updated, and operating conditions may accelerate, oscillate, or shift regimes \citep{even2009online,garivier2008upper,cheung2020reinforcement,zou2024distributed,mao2021near,wei2021non,zhang2024distributed,lecarpentier2019non}.
In these settings, RL becomes a \emph{tracking} problem---the MDP may change in the same time-scale as RL steps---and persistent error or instability can arise even when such per-step changes are small~\citep{zinkevich2003online,hazan2016introduction}.

Tracking difficulty depends not only on \emph{how much} the MDP changes, but on \emph{how it changes over time}. Environments with similar total variation can behave very differently: smooth drift may be trackable, while acceleration/oscillation can render value estimates stale between updates, and near-ties can flip the greedy action and induce instability~\citep{bellemare2016increasing}. Most nonstationary-RL theory uses coarse-scale measures (e.g., piece-wise variation, switch counts, worst-case, and adversarial drift), and the resulting methods often rely on generic templates (restarts, sliding windows, forgetting schedules)~\citep{even2009online,cheung2020reinforcement,lecarpentier2020nonstationarymarkovdecisionprocesses,mao2021near,wei2021non}.
These coarse-scale abstractions blur distinct patterns and provide little solver-independent guidance for \emph{when} aggressive adaptation is needed versus \emph{when} it is unnecessary or destabilizing. This motivates a structural question: \emph{which intrinsic features of a drifting MDP determine how the optimal solution moves, and how can those features be used to control tracking errors and inspire new algorithms with provable guarantees?}~\citep{allgower2012numerical,krantz2002implicit}

We model nonstationarity as a differentiable homotopy path of discounted MDPs, $\tau \mapsto \mathcal{M}(\tau)$, and studying the induced motion of the optimal Bellman fixed point $Q^\star_\tau$. We characterize this motion through a length--curvature--kink decomposition of tracking difficulty. \emph{Length} captures cumulative, value-relevant drift in rewards and transitions; \emph{curvature} captures changes in the drift rate (acceleration/oscillation) that govern how quickly value estimates become stale; and \emph{kinks} capture near-tie regimes where small perturbations can switch the identity of the optimal action, inducing nonsmooth changes in the optimal value. These components yield solver-agnostic guarantees that bound optimal-value displacement along the path and certify gap-safe neighborhoods where the optimal value remains locally stable.

We then turn these quantities into practical algorithm designs for learning along stochastic paths.  
We introduce two lightweight \emph{Homotopy-Tracking (HT)} wrappers: \textsc{HT-RL} for deep RL and \textsc{HT-MCTS} for Monte Carlo Tree Search.
Both estimate replay-based proxies for length, curvature, and near-tie proximity online and adapt algorithmic intensity accordingly. In \textsc{HT-RL}, the proxies modulate step sizes, target-network update inertia, and regularization with smoothing/hysteresis to avoid chattering. In \textsc{HT-MCTS}, the same signals control search depth and simulation budgets, allocating more planning effort under high curvature or near switching boundaries. We show that both HT wrappers ensure stability of the scheduled hyperparameter processes and lead to an upper-bound guarantee on the dynamic regret.

Our contributions are summarized as follows:
\begin{itemize}
\vspace{-0.15in}
    \item \textbf{Quantifying path geometry.} We formulate nonstationarity as a differentiable homotopy path of MDPs and introduce a length--curvature--kink decomposition of tracking difficulty.
    \vspace{-0.1in}
    \item \textbf{Solver-agnostic guarantees.} Under mild regularity, we prove a path-integral stability bound for optimal Bellman values along the path and derive gap-safe feasible regions that certify local stability away from switching regimes.
    \vspace{-0.1in}
    \item \textbf{From geometry to algorithms.} We derive a dynamic-regret decomposition under a contraction-with-noise abstraction and use it to motivate \textsc{HT-RL} and \textsc{HT-MCTS}, ensuring the hyperparameter process stability and an upper-bound guarantee on dynamic regret.
    \vspace{-0.1in}
    \item \textbf{Empirical validation.} Experiments on synthetic homotopy MDPs and drifting-control benchmarks show improved tracking and lower dynamic regret than matched static baselines, with gains most pronounced under oscillatory and switch-prone nonstationarity.
\end{itemize}

\section{Related Work}

\paragraph{Nonstationary RL via global budgets and adaptation templates.}
A dominant line models nonstationarity as a time-varying MDP sequence controlled by variation budgets/switch counts and uses optimism, restarts, or sliding-window/forgetting schemes \citep{even2009online,garivier2008upper,lecarpentier2019non,cheung2020reinforcement,mao2021near,wei2021non,zhang2025lipschitz}.

\vspace{-0.17in}
\paragraph{Dynamic regret and tracking decompositions.}
Dynamic regret separates loss into environment-driven movement and algorithmic error \citep{zinkevich2003online,hazan2016introduction,zhang2024modeling}, while RL analyses often use contraction-style arguments to separate tracking from sampling noise/approximation \citep{tsitsiklis1996analysis,watkins1992q,mnih2015human}.

\vspace{-0.17in}
\paragraph{Sensitivity/regularity, near-ties, and planning under drift.}
Classical MDP sensitivity bounds and continuation/implicit-function tools characterize local regularity of parameterized fixed points \citep{puterman2014markov,levin2017markov,qiao2024br,krantz2002implicit,zhang2024collaborative,allgower2012numerical}, and small action gaps expose brittleness near ties \citep{bellemare2016increasing}.
Planning methods such as MCTS provide compute-allocation knobs and strong decision-making performance, but drift handling is often heuristic \citep{kocsis2006bandit,tang2025malinzero,silver2016mastering,zhang2025tail}.

\vspace{-0.17in}
\paragraph{Our approach.}
We model environment evolution as a differentiable homotopy path and study the induced motion of the optimal Bellman fixed point.
This yields a length--curvature--kink characterization that separates cumulative drift, drift-rate variation, and switch-induced nonsmoothness, enabling solver-independent stability bounds and a tracking/dynamic-regret decomposition under a contraction-with-noise abstraction.

\section{Preliminaries}
\label{sec:prelim}
We study discounted MDPs evolving along a differentiable homotopy path $M(\tau)$.
To quantify non-stationarity in a value-relevant way, we use Lipschitz test functions on $(\Sspace,d_\Sspace)$ and the dual $W_1^\ast$ norm on signed measures, and assume optimal values admit a uniform Lipschitz scale.

A discounted MDP is $\mathcal M=(\Sspace,\Aspace,P,r,\gamma)$ with bounded reward $r$ and $\gamma\in(0,1)$.
For a policy $\pi(\cdot\mid s)$, define
$
V^\pi(s)=\E\!\Big[\sum_{t\ge0}\gamma^t r(s_t,a_t)\,\Big|\,s_0=s,\ a_t\!\sim\!\pi(\cdot\mid s_t)\Big],Q^\pi(s,a)=r(s,a)+\gamma\,\E_{s'\sim P(\cdot\mid s,a)}[V^\pi(s')].   
$
The optimal and policy Bellman operators are
$(\Tcal Q)(s,a)=r(s,a)+\gamma\,\E_{s'\sim P}\big[\max_{a'}Q(s',a')\big],(\Tcal^\pi Q)(s,a)=r(s,a)+\gamma\,\E_{s'\sim P}\big[Q(s',\pi(s'))\big].$
both $\gamma$-contractions in $\|\cdot\|_\infty$, with fixed points $Q^\star$ and $Q^\pi$.

For $f:\Sspace\to\R$, let $\|f\|_{\mathrm{Lip}}
:=\sup_{x\neq y}\frac{|f(x)-f(y)|}{d_{\Sspace}(x,y)}$ and
$
\mathrm{Lip}_1(\Sspace)
:=\{f:\ \|f\|_{\mathrm{Lip}}\le 1,\ \|f\|_\infty\le 1\}.
$
For a signed measure $\xi$ on $\Sspace$, define the dual norm
$
\|\xi\|_{W_1^\ast}
:=\sup_{f\in\mathrm{Lip}_1(\Sspace)}\Big|\int f\,d\xi\Big|.
$
This is the natural metric for measuring transition drift through Lipschitz (value-like) probes.

\paragraph{Homotopy path and differentiability in $\tau$.}
We consider a homotopy path of discounted MDPs
$
M(\tau)=(\mathcal S,\mathcal A,P_\tau,r_\tau,\gamma),\qquad \tau\in[0,1],\ \gamma\in(0,1),
$
with $P_\tau(\cdot\mid s,a)\in\mathcal P(\mathcal S)$ and $r_\tau$ bounded. We assume weak differentiability in $\tau$ \emph{as viewed by} Lipschitz tests: small changes in $\tau$ induce linear changes in expectations of all $f\in\mathrm{Lip}_1(\Sspace)$.

\begin{definition}[Dual derivative and $\Wone^\ast$ norm]
\label{def:dual-derivative}
For each $(s,a)$, there exist a finite signed measure $\partial_\tau P_\tau(\cdot\mid s,a)$ and a bounded function $\partial_\tau r_\tau(s,a)$ such that, for any $f \in \mathrm{Lip}_1(\mathcal S)$,
$
\frac{d}{d\tau}\mathbb{E}_{s'\sim P_\tau(\cdot\mid s,a)}[f(s')]
=\int f\, d\big(\partial_\tau P_\tau(\cdot\mid s,a)\big),
\frac{d}{d\tau}r_\tau(s,a)=\partial_\tau r_\tau(s,a).   
$
We equip $\partial_\tau P_\tau(\cdot\mid s,a)$ with the dual norm
$
\|\partial_\tau P_\tau(\cdot\mid s,a)\|_{W_1^\ast}
:= \sup_{f\in\mathrm{Lip}_1(\mathcal S)}\Big|\int f\,d\big(\partial_\tau P_\tau(\cdot\mid s,a)\big)\Big| < \infty.
$
If the second-order derivatives exist, we similarly define $\partial_{\tau\tau} P_\tau$ and $\partial_{\tau\tau} r_\tau$.
\end{definition}

For probability measures $\mu,\nu$ on $(\Sspace,d_\Sspace)$,
$
\Wone(\mu,\nu)=\sup_{f\in\mathrm{Lip}_1(\Sspace)}\Big|\int f\,d(\mu-\nu)\Big|
=\|\mu-\nu\|_{W_1^\ast}.
$

\begin{assumption}[Uniform Lipschitz/mixing scale]
\label{ass:mixing}
There exists $C_{\mathrm{mix}}\in(0,\infty)$ such that for all $\tau\in[0,1]$,
$
\|V^\star_\tau\|_{\mathrm{Lip}}\le C_{\mathrm{mix}}.
$
\end{assumption}

\begin{lemma}[Sufficient conditions (Appendix~\ref{app:mixing})]
\label{lem:mixing-sufficient-main}
If rewards are uniformly Lipschitz in $s$ and kernels are uniformly Wasserstein-Lipschitz in $s$
(with constants $L_r,\kappa$) and $\gamma\kappa<1$ (Assumption~\ref{ass:mixing-sufficient}),
then Assumption~\ref{ass:mixing} holds with $C_{\mathrm{mix}}=\frac{L_r}{1-\gamma\kappa}$.
\end{lemma}

This is a mild regularity requirement for metric MDP families; see Appendix~\ref{app:mixing} for the full statement and proof.
Assumption~\ref{ass:mixing} provides a uniform conversion from measure-level drift
(e.g., $\|\partial_\tau P_\tau\|_{W_1^\ast}$) to value-level sensitivity via Lipschitz testing,
and will be used to define path geometry (Sec.~\ref{sec:geometry}) and bound operator derivatives
(Sec.~\ref{sec:operator}).

\section{Quantifying Path Geometry}

We quantify non-stationarity along the differentiable MDP path $M(\tau)$ through three
pathwise quantities: a first-order path length $\PL$, a second-order curvature $\Curv$,
and a kink penalty $\Phi$ capturing non-differentiable optimal-action switches.
These quantities yield solver-independent bounds on the displacement and regularity of
$Q^\star_\tau$, which will later control tracking error and dynamic regret, inspiring the novel design of our HT-RL.

\label{sec:geometry}

We measure non-stationarity in a value-relevant scale by integrating (i) the instantaneous drift and (ii) its acceleration along the homotopy parameter, while explicitly accounting for non-smooth maximizer switches.

\begin{definition}[Path length and curvature]
\label{def:pl-curv}
Fix a scale $L_s>0$ that converts $\Wone^\ast$ changes into value scale (e.g., $L_s=\sup_\tau \|V^\star_\tau\|_{\mathrm{Lip}}\le C_{\mathrm{mix}}$). Define
$\mathrm{PL}
=\int_0^1 \Big(\|\partial_\tau r_\tau\|_\infty
+ L_s\sup_{s,a}\|\partial_\tau P_\tau(\cdot\mid s,a)\|_{W_1^\ast}\Big)\,d\tau,\mathrm{Curv}
=\int_0^1 \Big(\|\partial_{\tau\tau} r_\tau\|_\infty
+ L_s\sup_{s,a}\|\partial_{\tau\tau} P_\tau(\cdot\mid s,a)\|_{W_1^\ast}\Big)\,d\tau.$
For an interval $[\tau_0,\tau_1]\subseteq[0,1]$ we similarly define $\PL(\tau_0,\tau_1)$ and $\Curv(\tau_0,\tau_1)$ by restricting the integrals to $[\tau_0,\tau_1]$.
\end{definition}
where $\PL$ captures cumulative drift while $\Curv$ captures how quickly the drift rate varies.
In particular, under linear interpolation $r_\tau=(1-\tau)r_0+\tau r_1$ and
$P_\tau=(1-\tau)P_0+\tau P_1$, we have $\Curv=0$.

\begin{definition}[Global gap and kink set]
\label{def:global-gap}
For each $\tau$, define the global action gap
$
g_\tau := \inf_{s\in\mathcal S}\ \Big( Q^\star_\tau(s,a^\star_\tau(s)) - \max_{a\neq a^\star_\tau(s)} Q^\star_\tau(s,a) \Big)\in[0,\infty),    
$
where $a^\star_\tau(s)\in\arg\max_a Q^\star_\tau(s,a)$. Fix a margin $\xi>0$ and define the regular region $R:=\{\tau\in[0,1]: g_\tau\ge \xi\}$. Define the kink set $K:=\{\tau\in[0,1]: g_\tau=0\}$, noting that $K\subseteq [0,1]\setminus R$.
\end{definition}

On $\mathcal R$ the optimal action is uniformly separated from its competitors; on $\mathcal K$ the maximizer can change and the optimal Bellman operator loses differentiability in $\tau$. The optimal operator is thus not differentiable at kinks; we integrate this burden via an inverse-gap penalty.
In continuous or very large $\mathcal{S}$, the infimum may be driven to zero by states irrelevant to the objective. All results remain unchanged if one replaces $\mathcal{S}$ by any effective subset (e.g., the support of an occupancy measure along the path), so $g_\tau$ measures action separation on the states that matter.

\begin{definition}[Kink penalty]
\label{def:kink}
Fix $\delta>0$ small. For each isolated kink point $\tau_i\in\mathcal K$ choose $\epsilon>0$ small enough that the intervals $(\tau_i-\epsilon,\tau_i+\epsilon)$ are disjoint. Define
$
\Phi(\mathcal K,\mathrm{gap})
=\sum_{\tau_i\in\mathcal K}\int_{\tau_i-\epsilon}^{\tau_i+\epsilon}\frac{d\tau}{\max\{g_\tau,\delta\}}.
$
For an interval $[\tau_0,\tau_1]$ we write $\Phi(\mathcal K\cap[\tau_0,\tau_1],\mathrm{gap})$ for the same expression restricted to kinks in $[\tau_0,\tau_1]$.
\end{definition}

The penalty $\Phi$ remains finite for isolated switches and increases when the gap stays small over a wider neighborhood.
Additional interpretation and scheduler-facing heuristics are deferred to Appendix~\ref{app:path-interpretation}.

\section{Analyzing the Induced Motion of Optimal Bellman Fixed Points}
This section turns the geometric metrics from Sec.~\ref{sec:geometry} into explicit bounds on the motion of the optimal fixed point $Q^\star_\tau$ along the homotopy path.
On the regular region $\mathcal R_\xi$ (Def.~\ref{def:global-gap}), the greedy maximizer is locally constant (Lemma~\ref{lem:envelope}), so we can differentiate through the Bellman fixed point and control $\|\tfrac{d}{d\tau}Q^\star_\tau\|_\infty$ and $\|\tfrac{d^2}{d\tau^2}Q^\star_\tau\|_\infty$ by the same local speeds that define $\PL$ and $\Curv$ (Def.~\ref{def:pl-curv}).
Integrating these bounds yields the $\PL/\Curv$ contributions, while non-differentiable neighborhoods around action switches are charged by the inverse-gap penalty $\Phi$ (Def.~\ref{def:kink}), culminating in Theorem~\ref{thm:path-value}.

On the regular region $\mathcal R_\xi$, the optimal action is locally unique, so the maximizer is locally constant and derivatives pass through the Bellman fixed point. Integrating the resulting derivative bounds yields the $\PL/\Curv$ contributions, while kink neighborhoods are charged by the inverse-gap penalty~$\Phi$.

\label{sec:operator}
For a policy $\pi$, define the one-step evaluation operator and its resolvent
\begin{equation}
\begin{aligned}
\label{eq:resolvent-def}
&(\mathcal P^{\pi}_\tau V)(s,a)
:=\E_{s'\sim P_\tau(\cdot\mid s,a)}\!\Big[V\big(s',\pi(s')\big)\Big],
\\
&\mathcal R^{\pi}_\tau
:=(I-\gamma\,\mathcal P^{\pi}_\tau)^{-1}
=\sum_{k\ge0}\gamma^k(\mathcal P^{\pi}_\tau)^k,
\end{aligned}
\end{equation}
so that $\|\mathcal R^{\pi}_\tau\|_{\infty\to\infty}\le (1-\gamma)^{-1}$.

\begin{lemma}[Envelope property on the regular region]
\label{lem:envelope}
If $\tau\in\mathcal R$ then there exists a neighborhood $U$ of $\tau$ such that $\argmax_{a'}Q^\star_{\tilde\tau}(s,a')$ is single-valued for all $s$ and all $\tilde\tau\in U$; in particular $\pi^\star_{\tilde\tau}=\pi^\star_\tau$ on $U$. Consequently $Q^\star_{\tilde\tau}$ is differentiable in $\tilde\tau\in U$ and the derivative passes through the Bellman equation.
\end{lemma}
Lemma~\ref{lem:envelope} formalizes that away from ties the optimal action does not switch under small perturbations of $\tau$, so we can treat the maximizer as locally fixed and apply standard differentiation to the Bellman fixed point.

\begin{lemma}[First-order homotopy derivative]
\label{lem:first-derivative}
If $\tau\in\mathcal R$ and Assumption~\ref{ass:mixing} holds, then
$
\frac{d}{d\tau}Q^\star_\tau
=\mathcal R^{\pi^\star_\tau}_\tau\!\left(\partial_\tau r_\tau
+\gamma\,\Delta_\tau\right),
$
where $\Delta_\tau$ is the function on $(s,a)$ given by
$
\Delta_\tau(s,a)
:= \int_{\Sspace} V^\star_\tau(s')\,d\big(\partial_\tau P_\tau(\cdot\mid s,a)\big)(s').
$
Therefore
\begin{equation}
\label{eq:first-derivative-bound}
\begin{aligned}
\Big\|\tfrac{d}{d\tau}Q^\star_\tau\Big\|_\infty
&\le \frac{1}{1-\gamma}\,\|\partial_\tau r_\tau\|_\infty\\
&+\frac{\gamma\,C_{\mathrm{mix}}}{(1-\gamma)^2}\,\sup_{s,a}\|\partial_\tau P_\tau(\cdot\mid s,a)\|_{W_1^\ast}.
\end{aligned}
\end{equation}
\end{lemma}
Lemma~\ref{lem:first-derivative} shows that the instantaneous drift of $Q^\star_\tau$ is controlled by the same local speeds $(\partial_\tau r_\tau,\partial_\tau P_\tau)$ that define the path length $\PL$.

\begin{lemma}[Second-order homotopy derivative]
\label{lem:second-derivative}
If $\tau\in\mathcal R$ and $\partial_{\tau\tau}r_\tau$, $\partial_{\tau\tau}P_\tau$ exist with finite $\Wone^\ast$ norms, then
\begin{equation}
\label{eq:second-derivative-bound}
\begin{aligned}
\frac{d^2}{d\tau^2}Q^\star_\tau
&=\mathcal R^{\pi^\star_\tau}_\tau\!\left(\partial_{\tau\tau} r_\tau
+\gamma\,\Delta_{\tau\tau}\right)\\
&+\frac{c}{(1-\gamma)^3}\Big[\|\partial_\tau r_\tau\|_\infty
+L_s\!\sup_{s,a}\|\partial_\tau P_\tau(\cdot\mid s,a)\|_{W_1^\ast}\\
&+\ \|\partial_{\tau\tau} r_\tau\|_\infty
+L_s\!\sup_{s,a}\|\partial_{\tau\tau} P_\tau(\cdot\mid s,a)\|_{W_1^\ast}\Big],
\end{aligned}
\end{equation}
for some constant $c$ depending only on $C_{\mathrm{mix}}$ and uniform first/second-order bounds, and where
$
\Delta_{\tau\tau}(s,a)
:= \int_{\Sspace} V^\star_\tau(s')\,d\big(\partial_{\tau\tau} P_\tau(\cdot\mid s,a)\big)(s').
$
\end{lemma}
Lemma~\ref{lem:second-derivative} upper-bounds how fast the drift rate itself changes (an ``acceleration'' term), which is precisely what the curvature $\Curv$ integrates along the path.

Integrating \eqref{eq:first-derivative-bound} and \eqref{eq:second-derivative-bound} over regular sub-intervals yields the $\PL$ and $\Curv$ contributions, while kink neighborhoods are controlled by $\Phi$.

\begin{theorem}[Path integral value bound]
\label{thm:path-value}
Under Assumption~\ref{ass:mixing}, for any $0\le \tau_0<\tau_1\le 1$,
$\|Q^\star_{\tau_1}-Q^\star_{\tau_0}\|_\infty
\ \le\ \frac{\mathrm{PL}(\tau_0,\tau_1)}{(1-\gamma)^2}
\ +\ \frac{\mathrm{Curv}(\tau_0,\tau_1)}{(1-\gamma)^3}\ +\ \Phi\big(\mathcal K\cap[\tau_0,\tau_1],\,\mathrm{gap}\big).$
\end{theorem}

Theorem~\ref{thm:path-value} states that the geometric triple $(\PL,\Curv,\Phi)$ provides a solver-independent upper bound on the inevitable displacement of $Q^\star$ along the path, separating smooth drift/acceleration from non-smooth action switches.

\label{sec:feasible-geometry}

We next give a \emph{non-iterative} characterization of how far the optimal fixed point $Q^\star$ can shift under parameter variation, through an implicit function of optimality. In regular regions, Jacobian linearization yields explicit first-/second-order feasible \emph{tubes} in the parameter space, which give \emph{gap-safe} regions to avoid non-smooth action switches.

\noindent\textbf{Regular optimality map.}
On the regular region (Def.~\ref{def:global-gap}), the greedy action is locally unique and the optimal policy is locally constant (Lemma~\ref{lem:envelope}). Fix $\tau\in\mathcal R_\xi$ and write the regular optimality condition as the implicit map
\begin{equation}
\label{eq:Greg-final}
\begin{aligned}
G_{\rm reg}(\tau,Q)\ &:=\ Q-\Big(r_\tau+\gamma\,\E_{s'\sim P_\tau(\cdot\mid\cdot)}\big[\,Q(s',\pi^\star_\tau(s'))\,\big]\Big)\ \\ 
&=\ 0.
\end{aligned}
\end{equation}
where $\pi^\star_\tau$ is fixed in a neighborhood of~$\tau$.
Its Jacobian blocks (cf. Sec.~\ref{sec:operator}) are
\begin{equation}
\label{eq:Jblocks-final}
\begin{aligned}
&\partial_Q G_{\rm reg}=I-\gamma\,\mathcal P^{\pi^\star_\tau}_\tau,\\
&\partial_\tau G_{\rm reg}=-\partial_\tau r_\tau-\gamma\,\E_{s'\sim \partial_\tau P_\tau(\cdot\mid\cdot)}\!\big[V_Q(s')\big],
\end{aligned}
\end{equation}
with $\|(\partial_Q G_{\rm reg})^{-1}\|_{\infty\to\infty}\le(1-\gamma)^{-1}$. Equation~\eqref{eq:Greg-final} turns optimality into a smooth implicit constraint on $(\tau,Q)$, so local feasible motion of $Q^\star$ can be read off from Jacobians without running any solver.

\subsection{First- and second-order tubes: feasible radii in parameter space}
\label{subsec:final-tubes}

We consider a one-dimensional embedding $\tau\mapsto M(\tau)$ and ask how far $\tau$ can move while keeping $\|Q^\star_\tau-Q^\star_{\tau_0}\|_\infty$ below a tolerance.
By the implicit function theorem, $\tau\mapsto Q^\star_\tau$ is $C^1$ on each connected component of $\mathcal R_\xi$, and
\begin{equation}
\label{eq:Qprime-final}
\frac{d}{d\tau}Q^\star_\tau
=(\partial_Q G_{\rm reg})^{-1}\!\Big(\partial_\tau r_\tau+\gamma\,\E_{s'\sim \partial_\tau P_\tau}[V^\star_\tau(s')]\Big),
\end{equation}
where the expectation against $\partial_\tau P_\tau$ is understood in the dual sense of Def.~\ref{def:dual-derivative}. Define the \emph{speed density}
\begin{equation}
\label{eq:g-density-final}
\begin{aligned}
v_\tau&:=\Big\|\tfrac{d}{d\tau}Q^\star_\tau\Big\|_\infty
\\ 
&\le\ \frac{\|\partial_\tau r_\tau\|_\infty}{1-\gamma}
+\frac{\gamma\,C_{\mathrm{mix}}}{(1-\gamma)^2}\ \sup_{s,a}\|\partial_\tau P_\tau(\cdot\mid s,a)\|_{W_1^\ast}.
\end{aligned}
\end{equation}
for any conversion scale $L_s\ge \sup_\tau\|V^\star_\tau\|_{\mathrm{Lip}}$.
The scalar $v_\tau$ quantifies the local drift rate of the optimal fixed point per unit change in~$\tau$.

\begin{theorem}[First-order feasible tube]
\label{thm:final-tube1}
For any $\tau_0\in\mathcal R$ and any $\tau$ that remains in the same regular component,
$
\big\|Q^\star_\tau-Q^\star_{\tau_0}\big\|_\infty\ \le\ \int_{\tau_0}^{\tau} v_u\,du.
$
Hence the set $\mathsf{Tube}_1(\tau_0,\varepsilon):=\{\tau:\int_{\tau_0}^{\tau} v_u\,du\le\varepsilon\}$ is a
\emph{non-iterative feasible region} in parameter space: staying inside the tube guarantees that the
value deviation from $Q^\star_{\tau_0}$ does not exceed $\varepsilon$.
\end{theorem}
$\mathsf{Tube}_1$ is solver-agnostic: staying inside it guarantees that the optimal fixed point drifts by at most~$\varepsilon$ in sup norm.

When second derivatives exist in the dual sense, define the \emph{curvature density}
\begin{equation}
\label{eq:kappa-density-final}
\begin{aligned}
\kappa_\tau &:= \left\|\frac{d^2}{d\tau^2}Q^\star_\tau\right\|_\infty
\\ 
&\lesssim\;
\frac{\|\partial_{\tau\tau} r_\tau\|_\infty}{1-\gamma}
+\frac{\gamma C_{\rm mix}}{(1-\gamma)^2}\,L_s\!\!\sup_{s,a}\|\partial_{\tau\tau} P_\tau(\cdot\mid s,a)\|_{W_1^\ast}\\
&+\frac{c_2}{(1-\gamma)^3}\!\left(
\|\partial_{\tau} r_\tau\|_\infty
+L_s\!\!\sup_{s,a}\|\partial_{\tau} P_\tau(\cdot\mid s,a)\|_{W_1^\ast}
\right).
\end{aligned}
\end{equation}
for a constant $c_2$ depending only on uniform local bounds.
$\kappa_\tau$ upper-bounds how quickly the drift rate itself changes, providing a conservative correction beyond the first-order tube.

\begin{theorem}[Second-order refined tube]
\label{thm:final-tube2}
For $\tau$ sufficiently close to $\tau_0$ within $\mathcal R$,
$
\big\|Q^\star_\tau-Q^\star_{\tau_0}\big\|_\infty
\ \le\ \underbrace{\int_{\tau_0}^{\tau} v_u\,du}_{\text{length}}
\ +\ \underbrace{\tfrac12\,|\tau-\tau_0|\int_{\tau_0}^{\tau}\kappa_u\,du}_{\text{curvature correction}}.
$
\end{theorem}
The second-order tube tightens feasibility when the path bends sharply (large $\kappa_\tau$), even if the integrated speed is moderate.

\subsection{Gap-safe feasibility: excluding non-smooth switches}
\label{subsec:final-gap}

The tubes control $\|Q^\star_\tau-Q^\star_{\tau_0}\|_\infty$ but do not by themselves prevent crossing a kink where the greedy action switches. We therefore refine feasibility by imposing a gap constraint.

Recall the global action gap $g_\tau$ (Def.~\ref{def:global-gap}). For convenience write
$
g_{\rm gap}(\tau):=\inf_{s}\Big(Q^\star_\tau(s,a^\star_\tau(s))-\max_{a\ne a^\star_\tau(s)}Q^\star_\tau(s,a)\Big).
$

\begin{lemma}[Gap decay under tube radii]
\label{lem:final-gap}
If $\|Q^\star_\tau-Q^\star_{\tau_0}\|_\infty\le \varepsilon$, then
$g_{\rm gap}(\tau)\ge g_{\rm gap}(\tau_0)-2\varepsilon$.
\end{lemma}
A uniform $\varepsilon$-perturbation in action-values can shrink the best-vs-second-best gap by at most $2\varepsilon$, so choosing $\varepsilon\ll g_{\rm gap}(\tau_0)$ prevents new ties.

\begin{definition}[Gap-safe feasible region]
\label{def:final-safe}
Let $\mathsf{Tube}_\bullet(\tau_0,\varepsilon)$ denote either of the tubes in
Theorems~\ref{thm:final-tube1} or~\ref{thm:final-tube2}. For thresholds $(\varepsilon,\xi)$ define
$
\mathsf{Safe}(\tau_0;\varepsilon,\xi)
\ :=\ \Big\{\tau\in\mathsf{Tube}_\bullet(\tau_0,\varepsilon):\ g_{\rm gap}(\tau)\ge \xi\Big\},
$
which stays in the regular regime and controls the non-smooth burden near switches.
\end{definition}

Inside $\mathsf{Safe}(\tau_0;\varepsilon,\xi)$, both the value drift and the greedy action are stable, so the differentiable analysis of Sec.~\ref{sec:operator} continues to apply.

\subsection{Multi-parameter embedding: ellipsoidal feasible regions}
\label{subsec:final-ellipsoid}

We extend the one-dimensional tubes to a smooth multi-parameter embedding $M(\theta)=(P_\theta,r_\theta)$ with $\theta\in\R^p$ that remains in the regular region. By the implicit function theorem applied to $G_{\rm reg}(\theta,Q)=0$, the Jacobian of the optimal fixed point satisfies
\begin{equation}
\label{eq:Jtheta-final}
\begin{aligned}
J_\theta&:=\frac{\partial Q^\star}{\partial\theta}\\
&=(\partial_Q G_{\rm reg})^{-1}\!\Big(\partial_\theta r_\theta+\gamma\,\E_{s'\sim \partial_\theta P_\theta}[V^\star(s')]\Big)\\
&\in \mathcal L(\R^p,\mathcal Q).
\end{aligned}
\end{equation}
where the expectation against $\partial_\theta P_\theta$ is understood in the dual sense (Def.~\ref{def:dual-derivative}).

Let $W:\mathcal Q\to\mathcal Q$ be any positive semidefinite linear operator (e.g., $W=I$). The pullback metric on parameter space is
\begin{equation}
\label{eq:Gtheta-final}
\mathbf G_\theta:=J_\theta^\top W J_\theta\ \in\ \R^{p\times p}.
\end{equation}

\begin{definition}[Local ellipsoidal feasible set]
\label{def:final-ellipsoid}
For a value deviation budget $\varepsilon>0$,
$\mathsf E_\varepsilon(\theta)
:=\Big\{\Delta\theta\in\R^p:\ \Delta\theta^\top \mathbf G_\theta\,\Delta\theta\ \le\ \varepsilon^2\Big\}.$
\end{definition}

$\mathsf E_\varepsilon(\theta)$ is a first-order, solver-agnostic feasible region in $\theta$-space: any $\Delta\theta$ inside the ellipsoid guarantees that the induced linearized change in $Q^\star$ is at most~$\varepsilon$ in the $W$-weighted value norm.

\subsection{Directional feasibility: projected tangent cones}
\label{subsec:final-cone}

Beyond step size, feasibility can also depend on \emph{direction} when additional constraints must be preserved (e.g., gap safety).
Let $h_j(\theta,Q)\ge 0$ for $j=1,\dots,m$ denote inequality constraints, and define the composite constraint on parameters
$
H_j(\theta):=h_j(\theta,Q^\star(\theta)).
$
Let $\mathcal A(\theta):=\{j:\ H_j(\theta)=0\}$ be the active set.

\begin{proposition}[Projected feasible cone in parameter space]
\label{prop:cone}
At a regular point $\theta$, the set of first-order feasible directions is
\begin{equation}
\label{eq:cone-final}
\mathsf C(\theta)
=\Big\{\dot\theta\in\R^p:\ \langle \nabla H_j(\theta),\dot\theta\rangle \ge 0\ \ \forall j\in\mathcal A(\theta)\Big\},
\end{equation}
where the composite gradients satisfy the chain rule
\begin{equation}
\label{eq:chain-rule-final}
\nabla H_j(\theta)
=\nabla_\theta h_j(\theta,Q^\star(\theta))
+ J_\theta^\top \nabla_Q h_j(\theta,Q^\star(\theta)).
\end{equation}
\end{proposition}
\noindent\emph{Interpretation.}
$\mathsf C(\theta)$ selects directions that do not decrease any active constraint to first order (e.g., directions that do not immediately shrink the action gap below a threshold).

\subsection{Finite-state specialization}
\label{subsec:final-finite}

For finite MDPs, the $\|\cdot\|_{W_1^\ast}$ terms admit computable $\ell_1$-type upper bounds and Jacobian--vector products can be obtained via linear solves, making $(v_\tau,\kappa_\tau,\mathbf G_\theta)$ numerically accessible.
Details and practical surrogates are provided in Appendix~\ref{app:finite-surrogates}.

\section{Homotopy-Tracking RL on Stochastic Paths}
\label{sec:proxies}

We consider a \emph{stochastic} homotopy path where the environment index $\tau_t\in[0,1]$ may vary from step to step (typically monotone but driven by randomness), so the learner must track a moving optimal fixed point $Q^\star_{\tau_t}$ online. This section proceeds in three steps. 
(i) We construct lightweight, replay-based estimators of the incremental geometry on stochastic paths, yielding observable proxies for the local contributions to $(\mathrm{PL},\mathrm{Curv},\Phi)$. 
(ii) Using these estimates, we design a homotopy scheduler that maps smoothed proxy signals to solver hyperparameters; the induced hyperparameter processes are stable under stochastic paths (Prop.~\ref{prop:stability}). 
(iii) We define dynamic regret and show it factorizes into an \emph{algorithmic} term (controlled by the base solver under our scheduler) and a \emph{geometric} term (controlled by applying Theorem~\ref{thm:path-value} to the pathwise motion of $Q^\star$), yielding an explicit upper bound in Sec.~\ref{sec:convergence}.

We construct lightweight, replay-based proxies for the geometric triple $(\mathrm{PL},\mathrm{Curv},\Phi)$ using standard model-free signals (no extra network heads). These proxies estimate (i) local first-order drift, (ii) changes of that drift (curvature), and (iii) proximity to non-smooth action switches (kinks).

Let $\mathcal D_t$ be a replay buffer and $\mathcal B_t\subset\mathcal D_t$ a minibatch. Fix window sizes $W_1,W_2\in\mathbb N$, a bounded feature map $\phi:\mathcal S\to\mathbb R^m$, and a per-$(s,a)$ reward estimator $\widehat r_t(s,a)$ (e.g., an EMA over recent rewards).

We estimate reward drift over a window of length $W_1$ by
$
\Delta r^{(\infty)}_t
=\max_{(s,a)\in\mathcal B_t}\big|\,\widehat r_t(s,a)-\widehat r_{t-W_1}(s,a)\,\big|.
$
To probe transition drift, we compare next-state feature means:
for each $(s,a)$, let $\mu_t(s,a)$ be the empirical mean of $\phi(s')$ over recent samples, and define
$
\Delta P^{(\phi)}_t
=\max_{(s,a)\in\mathcal B_t}\big\|\mu_t(s,a)-\mu_{t-W_1}(s,a)\big\|_2 .
$

Using the Lipschitz scale $L_s$ (Sec.~\ref{sec:geometry}), define $\Delta\widehat{\mathrm{PL}}_t
=\Delta r^{(\infty)}_t+L_s\,\Delta P^{(\phi)}_t,
\Delta\widehat{\mathrm{Curv}}_t
=\big|\Delta\widehat{\mathrm{PL}}_t-\Delta\widehat{\mathrm{PL}}_{t-W_2}\big|.$
Intuitively, $\Delta\widehat{\mathrm{PL}}_t$ tracks first-order drift while
$\Delta\widehat{\mathrm{Curv}}_t$ captures changes in that drift (acceleration/oscillation). If $\phi$ is (approximately) $1$-Lipschitz, then $\Delta P_t^{(\phi)}$ can be viewed as a Wasserstein-1 drift proxy (up to constants),
making the reuse of the conversion scale $L_s$ consistent in spirit.
Otherwise, one may replace $L_s$ by a feature-scale constant $L_\phi$ without changing any theoretical statement.

We estimate a minibatch action-gap proxy by
$
\widehat{\mathrm{gap}}_t
=\min_{s\in B_t}
\left(
\max_a Q_\theta(s,a)
-\max_{a\neq a^\star_\theta(s)} Q_\theta(s,a)
\right),
$
where $a^\star_\theta(s)\in\arg\max_a Q_\theta(s,a)$ is the current greedy action at state $s$
(optionally using Double-$Q$ / target networks).
We then define a kink indicator
$
\mathrm{Kink}_t=\mathbf 1\{\widehat{\mathrm{gap}}_t\le \varepsilon_{\mathrm{gap}}\}.
$

In practice we smooth and stabilize these proxies using EMA/clipping and update the scheduler
with mild hysteresis; implementation details are deferred to Appendix~\ref{app:proxy-details}.
These observable signals will drive the homotopy scheduler in Sec.~\ref{sec:scheduler}.

\subsection{Homotopy scheduler design}
\label{sec:scheduler}

Our scheduler design is guided by two objectives: (1) the scheduled hyperparameter processes should be stable (slowly varying) on stochastic paths so as not to destroy the effective contraction/noise conditions of the base solver; (2) the same schedules should enter the dynamic-regret analysis in Sec.~\ref{sec:convergence}, where controlling the \emph{algorithmic} term requires cautious updates in fast-drift, high-curvature, or near-kink regimes.

Given smoothed proxies $(\tilde{\mathrm{PL}}_t,\tilde{\mathrm{Curv}}_t,\tilde{\mathrm{Kink}}_t)$, we map them to
solver hyperparameters through monotone, Lipschitz, clipped functions with hysteresis.
The design follows the regime interpretation:
(i) length governs cumulative drift (cautious step sizes),
(ii) curvature governs how rapidly estimates become stale (regularization/inertia),
and (iii) kink/gap signals non-differentiable switches (anti-chattering safeguards).

Let $\eta_t$ denote the learning rate used by the base value-/policy-update (e.g., Q-learning / actor--critic updates \citep{mnih2015human}), let $\nu_t$ denote the target-network Polyak soft-update rate \citep{lillicrap2015continuous,fujimoto2018addressing}, and let $\lambda_t$ denote a regularization or trust-region strength (e.g., weight decay / KL-penalty style control \citep{schulman2015trust,schulman2017proximal}). 
Let $\mathrm{clip}_{[a,b]}(x):=\min\{b,\max\{a,x\}\}$.
We choose the mappings $
\eta_t
=\mathrm{clip}_{[\eta_{\min},\eta_{\max}]}\!\Bigg(
\frac{\eta_0}{1+\alpha_1\tilde{\mathrm{PL}}_t+\alpha_2\tilde{\mathrm{Curv}}_t}\Bigg),\nu_t
=\mathrm{clip}_{[\nu_{\min},\nu_{\max}]}\!\Bigg(
\frac{\nu_0}{1+\beta_1\tilde{\mathrm{Kink}}_t\big(1+\beta_2/\max\{\widehat{\mathrm{gap}}_t,\delta\}\big)}\Bigg),\lambda_t
=\lambda_0\big(1+c_1\tilde{\mathrm{PL}}_t+c_2\sqrt{\tilde{\mathrm{Curv}}_t}\big).
$
Thus, larger estimated length/curvature reduces $\eta_t$ and increases $\lambda_t$, while approaching a
kink-like regime (small empirical gap) slows down target updates via a smaller $\nu_t$, reducing chattering.

For planning algorithms (e.g., MCTS) with base depth $D_0$ and simulation budget $B_0$ and caps
$D_{\max},B_{\max}$, we adjust
$
D_t
=\min\!\Big\{D_{\max},\ \Big\lfloor D_0+\gamma_1(1+\tilde{\mathrm{PL}}_t)
+\gamma_2\sqrt{1+\tilde{\mathrm{Curv}}_t}\\
+\gamma_3\,\tfrac{\tilde{\mathrm{Kink}}_t}{\max\{\widehat{\mathrm{gap}}_t,\delta\}}
\Big\rceil\Big\},
B_t
=\min\!\Big\{B_{\max},\ \Big\lfloor B_0\big(1+\gamma_1\tilde{\mathrm{PL}}_t+\gamma_2\tilde{\mathrm{Curv}}_t\big)\Big\rceil\Big\}.
$

EMA smoothing, clipping, and hysteresis guarantee that these parameters vary slowly over time rather than
reacting to high-frequency proxy noise; formal bounded-variation and no-chattering statements are deferred to
Appendix~\ref{app:scheduler-stability}.

\begin{proposition}[Stability under smoothing, clipping, and hysteresis]
\label{prop:stability}
Suppose the raw proxies $(\Delta\widehat{\mathrm{PL}}_t,\Delta\widehat{\mathrm{Curv}}_t,\widehat{\mathrm{gap}}_t,\mathrm{Kink}_t)$ have uniformly bounded second moments, $\beta<1$, and scheduler updates occur every $H$ steps with hysteresis threshold $\Delta_{\mathrm{hys}}>0$. Then each scheduled hyperparameter process ($\eta_t,\nu_t,\lambda_t,D_t,B_t$) is piecewise-constant with bounded variation. Moreover, for any $\varepsilon>0$, the fraction of steps at which any hyperparameter changes by more than $\varepsilon$ can be made arbitrarily small by taking $H$ and $\Delta_{\mathrm{hys}}$ sufficiently large.
\end{proposition}

\subsection{Convergence and dynamic regret guarantees}
\label{sec:convergence}

We now state a dynamic-regret guarantee for stochastic paths under the proposed homotopy scheduler.
We measure online performance against the instantaneous optimum along the (random) path $\{\tau_t\}_{t=1}^T$ using the following definition.

\begin{definition}[Dynamic regret]
\label{def:dynreg-main}
Let $d_0$ be a fixed reference initial-state distribution and let $\pi_t$ be the policy produced at time $t$ (e.g., greedy / actor induced by the current critic). Define
$
\mathrm{DynReg}(T)
=\sum_{t=1}^T \big\langle d_0,\ V^\star_{\tau_t}-V^{\pi_t}_{\tau_t}\big\rangle.
$
\end{definition}

The goal of this subsection is to upper bound $\E[\mathrm{DynReg}(T)]$ for our homotopy scheduler.
The analysis separates (i) an \emph{algorithmic} term controlled by the base solver under the scheduled hyperparameters, and (ii) a \emph{geometric} term controlled by the pathwise motion of $Q^\star$ quantified in Sec.~\ref{sec:operator}.

\begin{theorem}[Dynamic regret decomposition (algorithmic + geometric)]
\label{thm:dynreg-decomp-main}
Under Assumption~\ref{ass:mixing} and the one-step contraction abstraction in Assumption~\ref{ass:solver},
there exist constants $C_{\mathrm{alg}},C_{\mathrm{geo}}>0$ such that for any (possibly stochastic) monotone path $(\tau_t)_{t=1}^T$,
$
\E[\mathrm{DynReg}(T)]
\ \le\ C_{\mathrm{alg}}\sum_{t=1}^T \E\|Q_t-Q^\star_{\tau_t}\|_\infty
\ +\ C_{\mathrm{geo}}\sum_{t=1}^{T-1} \E\big[\Delta \mathrm{Geo}_t\big].
$
This decomposition is proved in Appendix~\ref{app:regret-decomp}; in particular, Proposition~\ref{prop:regret} gives an explicit bound where the geometric contribution is expressed directly by $(\mathrm{PL},\mathrm{Curv},\Phi)$.
\end{theorem}

We summarize guarantees that justify the geometry-aware scheduler.
When the path is locally frozen (i.e., $\tau_t\equiv\bar\tau$ for a period), standard stationary convergence
is recovered; see Appendix~\ref{app:scheduler-stability}.
When the path moves, tracking error obeys a contraction-plus-load recursion where the load separates
geometry from noise/bias.

Let $e_t:=\E\|Q_t-Q^\star_{\tau_t}\|_\infty$ and let $\Delta(\cdot)_t$ denote the contribution of
$[\tau_t,\tau_{t+1}]$ to the corresponding geometric quantity.
Define the per-step geometric load
$
\Delta \mathrm{Geo}_t=\frac{\Delta\mathrm{PL}_t}{(1-\gamma)^2}
+\frac{\Delta\mathrm{Curv}_t}{(1-\gamma)^3}
+\Delta\Phi_t.
$

\begin{theorem}[Homotopy tracking recursion]
\label{thm:tracking}
Under Assumption~\ref{ass:mixing} and the generic one-step solver contraction in
Assumption~\ref{ass:solver} (Appendix~\ref{app:regret}), there exist constants $c_1,c_2>0$ such that
$
e_{t+1}\ \le\ \rho\,e_t\ +\ c_1\,\Delta \mathrm{Geo}_t\ +\ c_2\,\E[\sigma_t]\ +\ \beta_t.
$
Consequently,
$
\max_{t\le T} e_t\ \lesssim\ \frac{1}{1-\rho}\Big(
\sum_{u\le T}\Delta \mathrm{Geo}_u\ +\ \sum_{u\le T}\E[\sigma_u]\ +\ \sum_{u\le T}\beta_u\Big),
$
up to universal constants.
\end{theorem}

Theorem~\ref{thm:tracking} turns the pathwise drift control from Theorem~\ref{thm:path-value} into an online
bound: tracking error is dominated by a contracting term plus a load term that cleanly separates
geometry, variance, and approximation bias.
Additional stochastic-approximation compatibility results (Robbins--Monro regimes) and no-chattering
statements for the scheduler are deferred to Appendix~\ref{app:scheduler-stability}.

\begin{corollary}[Asymptotic convergence when the path stabilizes]
\label{cor:asymptotic}
If $\tau_t\to\tau_\infty$, $\sum_t \Delta \mathrm{Geo}_t<\infty$, $\sum_t \E[\sigma_t]<\infty$, and $\beta_t\to 0$, then $e_t\to 0$ and $Q_t\to Q^\star_{\tau_\infty}$ in expectation. 
\end{corollary}

In particular, when the path stabilizes and the tracking error vanishes, the per-round suboptimality also vanishes. Under the summability conditions of Cor.~\ref{cor:asymptotic}, $\sum_{t=1}^\infty \E\|Q_t-Q^\star_{\tau_t}\|_\infty<\infty$ holds by Theorem~\ref{thm:tracking}, and thus $\mathrm{DynReg}(T)$ grows at most sublinearly, implying $\mathrm{DynReg}(T)/T\to 0$.

\section{Experiments}
\label{sec:experiments}

We empirically validate (i) the geometric characterization of non-stationarity and (ii) the geometry-aware homotopy-tracking wrapper (HT-RL).
We first use synthetic ring MDPs where $(\mathrm{PL},\mathrm{Curv},\Phi)$ and $Q^\star$ are computable, directly checking the path--value bound (Theorem~\ref{thm:path-value}) and the feasible-tube bounds (Thms.~\ref{thm:final-tube1}--\ref{thm:final-tube2}).
We then evaluate HT-RL as a wrapper around standard deep RL solvers on four non-stationary control benchmarks:
\textsc{LunarLander}, \textsc{Acrobot}, \textsc{PointMass}, and \textsc{Pendulum} (clean and noisy variants).
Implementation details and full hyperparameters are deferred to Appendix~\ref{app:exp-details}.

\begin{figure}[h]
  \centering
  \includegraphics[width=0.65\linewidth]{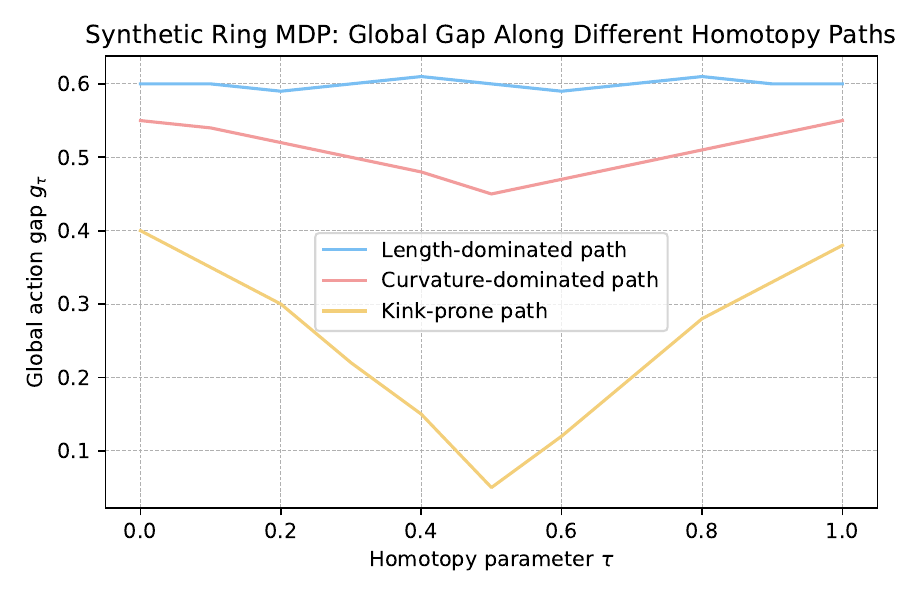}%
  \caption{%
  Synthetic ring MDP homotopy paths. Each panel (length-dominated, curvature-dominated, kink-prone) visualizes the moving reward bump and transition bias along $\tau$, together with the corresponding global action gap $g_\tau$. This figure illustrates how path length, curvature, and kink mass arise in simple MDPs.}
  \label{fig:synthetic-geometry}
\vspace{-1.5em}
\end{figure}

\begin{figure}[h]
  \centering

  \begin{subfigure}[t]{0.49\linewidth}
    \centering
    \includegraphics[width=\linewidth]{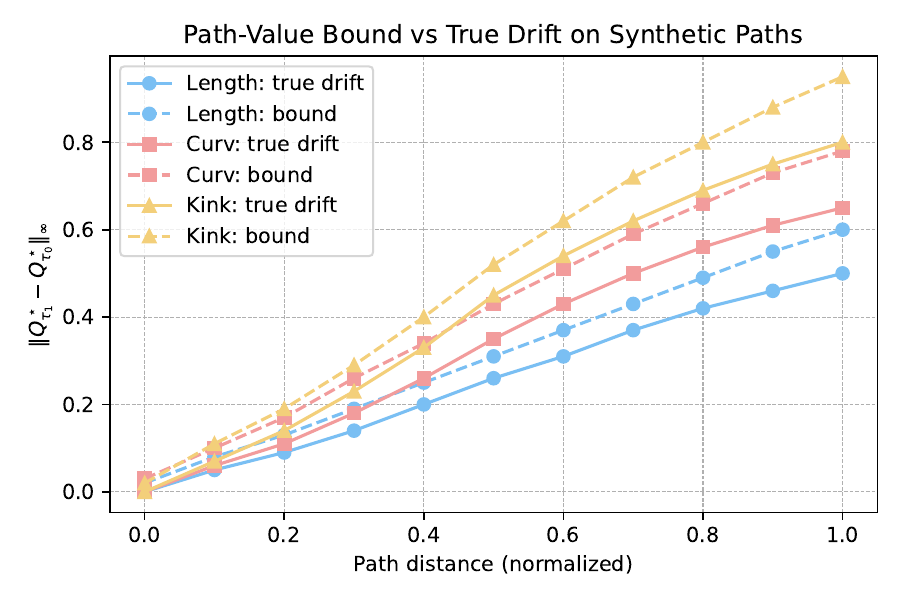}%
    \caption{%
    Path--value bound (Theorem~\ref{thm:path-value}) vs.\ true drift
    $\|Q^\star_{\tau_1}-Q^\star_{\tau_0}\|_\infty$ on length-/curvature-/kink-dominated paths.}
    \label{fig:synthetic-drift}
  \end{subfigure}\hfill
  \begin{subfigure}[t]{0.49\linewidth}
    \centering
    \includegraphics[width=\linewidth]{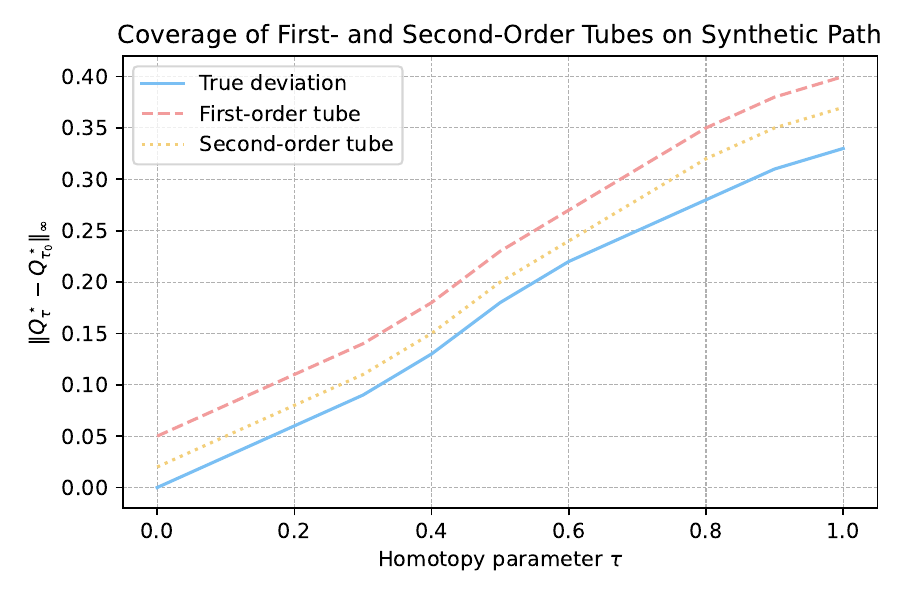}%
    \caption{%
    Tube coverage (Thms.~\ref{thm:final-tube1}--\ref{thm:final-tube2}): actual deviation
    vs.\ first- and second-order tube radii.}
    \label{fig:synthetic-tubes}
  \end{subfigure}

  \caption{%
  Synthetic ring MDP: validating the path--value bound and feasible tubes.
  \textbf{Left:} Theorem~\ref{thm:path-value} tracks the true drift with a modest constant factor, and curvature/kink terms tighten the bound exactly when the corresponding geometric component is large.
  \textbf{Right:} The refined second-order tube is noticeably tighter on curvature-dominated paths, while both tubes closely envelope the true deviation on length-dominated paths.}
  \label{fig:synthetic-bounds}
\vspace{-1em}
\end{figure}

\begin{figure}[h!]
  \centering
  \includegraphics[width=0.49\linewidth]{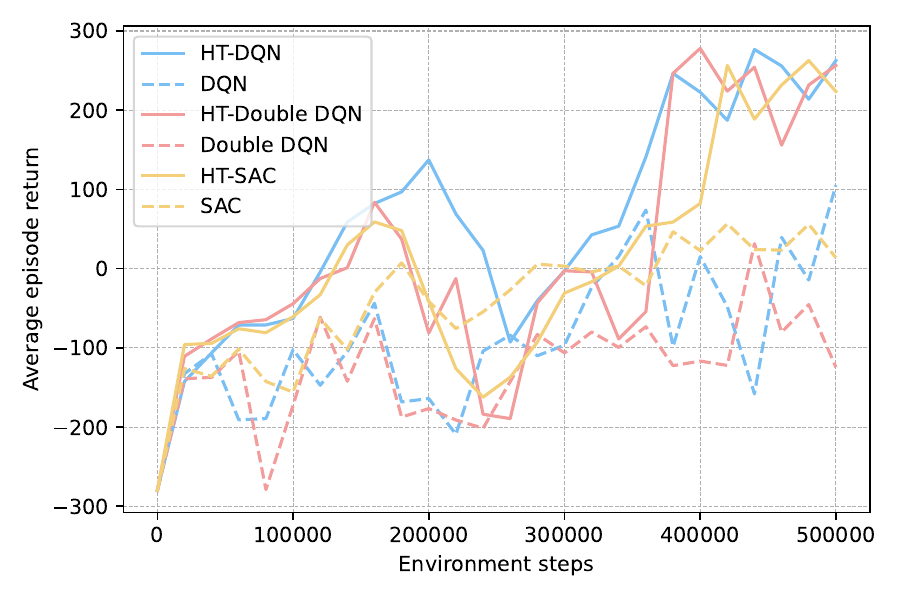}%
  \includegraphics[width=0.49\linewidth]{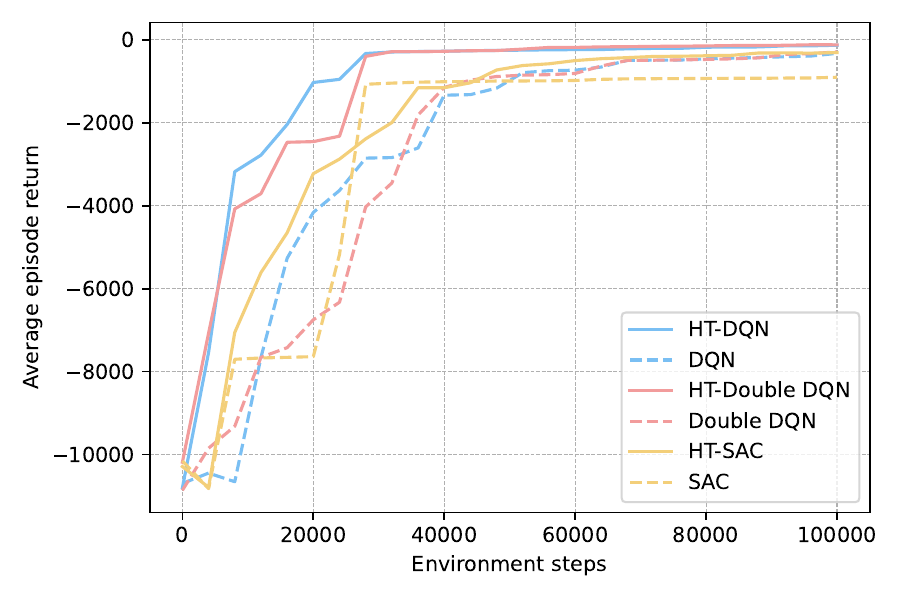}
  \caption{%
  Deep control benchmarks under non-stationary homotopy paths (noisy drift).
  We show two representative environments (left: \textsc{LunarLander}, right: \textsc{PointMass}).
  Each panel reports average episode return vs.\ environment steps for static baselines and their HT-RL counterparts.
  Full learning curves for all environments are deferred to Appendix~\ref{app:exp-figures}.}
  \label{fig:deep-returns-noise}
\end{figure}

\begin{figure}[h!]
  \centering

  \begin{subfigure}[t]{0.49\linewidth}
    \centering
    \includegraphics[width=\linewidth]{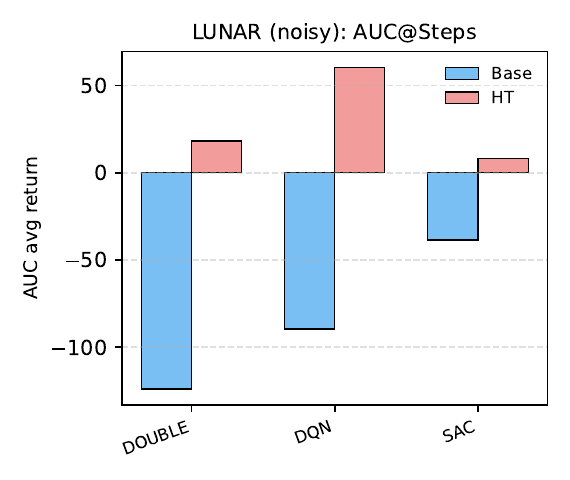}%
    \caption{\textbf{AUC@Steps} (area under the return curve over the training budget).}
    \label{fig:deep-auc-noise}
  \end{subfigure}\hfill
  \begin{subfigure}[t]{0.49\linewidth}
    \centering
    \includegraphics[width=\linewidth]{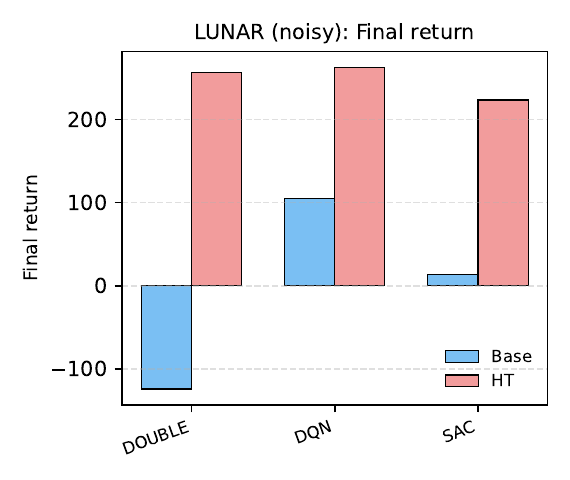}%
    \caption{\textbf{Final evaluation return}.}
    \label{fig:deep-final-noise}
  \end{subfigure}

  \caption{%
  Aggregate performance under non-stationary drift on a representative environment (\textsc{LunarLander}).
  Full 4-environment summaries are deferred to Appendix~\ref{app:exp-figures}.}
\end{figure}

\begin{table*}[h]
  \centering
  \scriptsize
  \setlength{\tabcolsep}{3.0pt}
\resizebox{0.8\textwidth}{!}{%
  \begin{tabular}{l|l|l rrr rrrr rrr}
\toprule\toprule
Env & Drift & Solver
& AUC$_{\text{base}}$ & AUC$_{\text{HT}}$ & $\Delta$AUC(\%)
& R@50\%$_{\text{base}}$ & R@50\%$_{\text{HT}}$
& R@75\%$_{\text{base}}$ & R@75\%$_{\text{HT}}$
& Final$_{\text{base}}$ & Final$_{\text{HT}}$ & $\Delta$Final(\%) \\
\midrule
lunar & noisy & DOUBLE & -123.884 & 18.115 & 114.6 & -172.071 & -186.526 & -110.192 & 171.352 & -124.461 & 256.521 & 306.1 \\
 &  & DQN & -89.573 & 60.199 & 167.2 & -93.983 & -34.808 & -55.198 & 220.102 & 105.371 & 262.695 & 149.3 \\
 &  & SAC & -38.604 & 7.974 & 120.7 & -40.710 & -149.617 & 29.504 & 57.510 & 13.876 & 223.664 & 1511.9 \\
\hline
acrobot & noisy & DOUBLE & -360.838 & -238.192 & 34.0 & -500.000 & -128.100 & -500.000 & -196.675 & -79.500 & -75.000 & 5.7 \\
 &  & DQN & -336.722 & -248.748 & 26.1 & -474.950 & -202.600 & -500.000 & -134.800 & -88.300 & -74.200 & 16.0 \\
 &  & SAC & -335.874 & -199.134 & 40.7 & -500.000 & -137.300 & -500.000 & -88.100 & -84.100 & -79.500 & 5.5 \\
\hline
pointmass & noisy & DOUBLE & -2872.648 & -1235.390 & 57.0 & -866.531 & -239.016 & -483.819 & -152.126 & -287.220 & -118.111 & 58.9 \\
 &  & DQN & -2637.418 & -1081.640 & 59.0 & -976.346 & -249.771 & -486.639 & -203.427 & -324.011 & -128.296 & 60.4 \\
 &  & SAC & -2786.019 & -2120.616 & 23.9 & -990.533 & -671.267 & -933.254 & -394.714 & -900.660 & -305.808 & 66.0 \\
\hline
pendulum & noisy & DOUBLE & -520.870 & -265.879 & 49.0 & -215.836 & -143.938 & -192.093 & -121.512 & -185.028 & -50.116 & 72.9 \\
 &  & DQN & -430.961 & -306.180 & 29.0 & -235.350 & -161.403 & -178.560 & -122.440 & -155.122 & -77.443 & 50.1 \\
 &  & SAC & -549.411 & -372.853 & 32.1 & -316.273 & -190.280 & -248.684 & -168.130 & -240.453 & -139.797 & 41.9 \\
\bottomrule
\bottomrule
\end{tabular}
}
  \caption{%
  Noisy-drift results: AUC@Steps and returns at 50\%/75\% of the training budget and at the final evaluation, together with relative improvements of HT-RL over matched static baselines.
  Clean-drift rows are deferred to Appendix~\ref{app:exp-figures}.}
  \label{tab:deep-summary}
\vspace{-2em}
\end{table*}

\paragraph{Evaluation with synthetic homotopy MDPs.}
We construct a tabular ring MDP and generate three homotopy regimes by controlling how rewards and transition bias evolve along $\tau$:
length-dominated (large $\mathrm{PL}$, small $\mathrm{Curv}$, no kinks), curvature-dominated (comparable $\mathrm{PL}$ but large $\mathrm{Curv}$), and kink-prone (small gaps and nonzero $\Phi$).
Across these paths, Theorem~\ref{thm:path-value} upper-bounds the true drift with a modest constant factor (Fig.~\ref{fig:synthetic-drift}).
Curvature-dominated paths exhibit larger deviations at comparable $\mathrm{PL}$, and the second-order tube is significantly tighter than the first-order tube (Fig.~\ref{fig:synthetic-tubes}).
On kink-prone paths, $\Phi$ concentrates near action switches and gap-safe regions (Def.~\ref{def:final-safe}) correctly exclude near-tie segments.
Overall, these results support interpreting $(\mathrm{PL},\mathrm{Curv},\Phi)$ as first-, second-, and third-order measures of non-stationarity.

\paragraph{Evaluation on deep RL benchmarks.}
We test HT-RL (Alg.~\ref{alg:ht-rl}) as a geometry-aware wrapper around standard deep RL solvers under monotone non-stationary homotopy paths $\tau_t$.
We summarize performance using \textbf{AUC@Steps} and \textbf{final evaluation return}.
Fig.~\ref{fig:deep-returns-noise} shows representative noisy-drift learning curves where HT-RL tracks drift more effectively than static baselines.
Aggregate noisy-drift improvements are reported in Fig.~\ref{fig:deep-auc-noise}--\ref{fig:deep-final-noise} and Table~\ref{tab:deep-summary}; full learning curves (including clean drift) are deferred to Appendix~\ref{app:exp-figures}.

\section{Conclusion}
\label{sec:conclusion}

We developed a solver-agnostic geometric characterization of non-stationary RL by bounding the drift of optimal fixed points along a homotopy path via path length, curvature, and kink burden, and used observable proxies of these quantities to design a stable homotopy scheduler that adaptively regularizes learning and planning; experiments on synthetic and deep-control benchmarks show improved tracking under curved and kink-prone drifts.

\nocite{langley00}

\bibliography{ref}
\bibliographystyle{main}

\newpage
\appendix
\onecolumn

\section{Sufficient Conditions for Assumption~\ref{ass:mixing}}
\label{app:mixing}

This appendix provides a sufficient condition ensuring that the optimal value functions
$\{V^\star_\tau\}_{\tau\in[0,1]}$ admit a uniform Lipschitz scale, as stated in
Assumption~\ref{ass:mixing}. The key idea is that the optimality operator contracts
\emph{Lipschitz seminorms} whenever the transition kernel is Lipschitz in the $1$-Wasserstein metric.

\subsection{Wasserstein-Lipschitz kernels and reward regularity}

Fix a metric space $(\Sspace,d_{\Sspace})$. For each $\tau\in[0,1]$ and action $a\in\Aspace$,
assume the reward and transition satisfy:

\begin{assumption}[A sufficient condition for uniform Lipschitz scale]
\label{ass:mixing-sufficient}
There exist constants $L_r\ge 0$ and $\kappa\ge 0$ such that for all $\tau\in[0,1]$:
\begin{enumerate}
\item (\textbf{Reward Lipschitzness}) For every $a\in\Aspace$, the map $s\mapsto r_\tau(s,a)$ is $L_r$-Lipschitz:
\[
|r_\tau(s,a)-r_\tau(s',a)|\le L_r\,d_{\Sspace}(s,s')\qquad \forall s,s'\in\Sspace.
\]
\item (\textbf{Kernel Wasserstein-Lipschitzness}) For every $a\in\Aspace$,
\[
\Wone\!\big(P_\tau(\cdot\mid s,a),P_\tau(\cdot\mid s',a)\big)
\le \kappa\,d_{\Sspace}(s,s')\qquad \forall s,s'\in\Sspace.
\]
\item (\textbf{Discounted contraction on Lipschitz scale}) $\gamma\kappa<1$.
\end{enumerate}
\end{assumption}

The second condition states that the one-step transition distribution varies smoothly in state,
measured by the $1$-Wasserstein metric $\Wone$, uniformly over $\tau$ and $a$.

\subsection{A Lipschitz stability lemma}

We first record a standard duality bound showing how Wasserstein controls expectation drift.

\begin{lemma}[Wasserstein controls Lipschitz expectation differences]
\label{lem:w1-lip-test}
Let $\mu,\nu$ be probability measures on $(\Sspace,d_{\Sspace})$.
For any $f:\Sspace\to\R$ with finite Lipschitz seminorm $\|f\|_{\mathrm{Lip}}$,
\[
\Big|\int f\,d\mu-\int f\,d\nu\Big|
\le \|f\|_{\mathrm{Lip}}\;\Wone(\mu,\nu).
\]
\end{lemma}

\begin{proof}
By definition of $\Wone$ via Kantorovich--Rubinstein duality,
\[
\Wone(\mu,\nu)=\sup_{\|g\|_{\mathrm{Lip}}\le 1}\Big|\int g\,d(\mu-\nu)\Big|.
\]
Applying this to $g = f / \|f\|_{\mathrm{Lip}}$ yields the claim.
\end{proof}

\subsection{Uniform Lipschitz bound for $V^\star_\tau$}

Define the optimality operator on state values
\[
(\Tcal_\tau V)(s)
:=\max_{a\in\Aspace}\Big\{r_\tau(s,a)+\gamma\,\E_{s'\sim P_\tau(\cdot\mid s,a)}[V(s')]\Big\}.
\]
Its unique fixed point is $V^\star_\tau=\Tcal_\tau V^\star_\tau$.

\begin{lemma}[Optimality operator contracts the Lipschitz seminorm]
\label{lem:TV-lip}
Under Assumption~\ref{ass:mixing-sufficient}, for any bounded $V:\Sspace\to\R$,
\[
\|\Tcal_\tau V\|_{\mathrm{Lip}}
\le L_r + \gamma\kappa\,\|V\|_{\mathrm{Lip}}.
\]
\end{lemma}

\begin{proof}
Fix $s,s'\in\Sspace$. Let $a_s\in\arg\max_a \{r_\tau(s,a)+\gamma \E_{P_\tau(\cdot\mid s,a)}V\}$.
Then
\begin{align*}
(\Tcal_\tau V)(s)-(\Tcal_\tau V)(s')
&\le \Big(r_\tau(s,a_s)-r_\tau(s',a_s)\Big)
 +\gamma\Big(\E_{P_\tau(\cdot\mid s,a_s)}V-\E_{P_\tau(\cdot\mid s',a_s)}V\Big).
\end{align*}
By reward Lipschitzness, the first term is at most $L_r\,d_{\Sspace}(s,s')$.
For the second term, apply Lemma~\ref{lem:w1-lip-test} with
$\mu=P_\tau(\cdot\mid s,a_s)$ and $\nu=P_\tau(\cdot\mid s',a_s)$:
\[
\Big|\E_{P_\tau(\cdot\mid s,a_s)}V-\E_{P_\tau(\cdot\mid s',a_s)}V\Big|
\le \|V\|_{\mathrm{Lip}}\;\Wone\!\big(P_\tau(\cdot\mid s,a_s),P_\tau(\cdot\mid s',a_s)\big)
\le \kappa\,\|V\|_{\mathrm{Lip}}\,d_{\Sspace}(s,s').
\]
Combining and symmetrizing over $(s,s')$ yields the claimed seminorm bound.
\end{proof}

\begin{proposition}[A sufficient condition for Assumption~\ref{ass:mixing}]
\label{prop:uniform-lip-Vstar}
Under Assumption~\ref{ass:mixing-sufficient}, for all $\tau\in[0,1]$,
\[
\|V^\star_\tau\|_{\mathrm{Lip}}
\le \frac{L_r}{1-\gamma\kappa}.
\]
Consequently, Assumption~\ref{ass:mixing} holds with
$C_{\mathrm{mix}}:=\frac{L_r}{1-\gamma\kappa}$.
\end{proposition}

\begin{proof}
Since $V^\star_\tau=\Tcal_\tau V^\star_\tau$, Lemma~\ref{lem:TV-lip} implies
\[
\|V^\star_\tau\|_{\mathrm{Lip}}
=\|\Tcal_\tau V^\star_\tau\|_{\mathrm{Lip}}
\le L_r + \gamma\kappa\,\|V^\star_\tau\|_{\mathrm{Lip}}.
\]
Rearranging gives $(1-\gamma\kappa)\|V^\star_\tau\|_{\mathrm{Lip}}\le L_r$.
The condition $\gamma\kappa<1$ guarantees finiteness.
\end{proof}

\subsection{ A bound for $Q^\star_\tau$}

For completeness, note that $Q^\star_\tau(s,a)=r_\tau(s,a)+\gamma\,\E_{P_\tau(\cdot\mid s,a)}[V^\star_\tau(s')]$
inherits a Lipschitz bound:

\begin{corollary}[Lipschitz bound for $Q^\star_\tau$]
\label{cor:uniform-lip-Qstar}
Under Assumption~\ref{ass:mixing-sufficient},
\[
\sup_{\tau\in[0,1]}\sup_{a\in\Aspace}\|Q^\star_\tau(\cdot,a)\|_{\mathrm{Lip}}
\le L_r + \gamma\kappa\,\frac{L_r}{1-\gamma\kappa}
=\frac{L_r}{1-\gamma\kappa}.
\]
\end{corollary}

\begin{proof}
Fix $\tau$ and $a$. For $s,s'\in\Sspace$,
\begin{align*}
|Q^\star_\tau(s,a)-Q^\star_\tau(s',a)|
&\le |r_\tau(s,a)-r_\tau(s',a)|
+\gamma\Big|\E_{P_\tau(\cdot\mid s,a)}V^\star_\tau-\E_{P_\tau(\cdot\mid s',a)}V^\star_\tau\Big|\\
&\le L_r\,d_{\Sspace}(s,s')
+\gamma\,\|V^\star_\tau\|_{\mathrm{Lip}}\;\Wone\!\big(P_\tau(\cdot\mid s,a),P_\tau(\cdot\mid s',a)\big)\\
&\le \Big(L_r+\gamma\kappa\,\|V^\star_\tau\|_{\mathrm{Lip}}\Big)d_{\Sspace}(s,s').
\end{align*}
Now invoke Proposition~\ref{prop:uniform-lip-Vstar}.
\end{proof}

\section{Interpreting $\PL$, $\Curv$, and $\Phi$ in Practice}
\label{app:path-interpretation}

\paragraph{First-/second-order and kink information in practice.}
The preceding definitions provide pathwise measures of first- and second-order change ($\PL$ and $\Curv$) and an integrable penalty for non-differentiable maximizer switches (via the global gap $g_\tau$ and the kink mass $\Phi$). Before stating the pathwise value bound, we briefly align these quantities with concrete non-stationary phenomena and the algorithmic choices they motivate; the formal link to performance appears in \Cref{thm:path-value} and the scheduler in \Cref{sec:scheduler}.

\emph{First-order information (path length $\PL$).}
By \Cref{def:pl-curv}, $\PL$ integrates the instantaneous speed
$
\|\partial_\tau r_\tau\|_\infty
+ L_s\,\sup_{s,a}\|\partial_\tau P_\tau(\cdot\mid s,a)\|_{W_1^\ast}
$
along the homotopy parameter. It becomes large when rewards or dynamics undergo a \emph{persistent} drift—possibly tiny at each moment, yet substantial in aggregate. In \Cref{thm:path-value}, $\PL$ drives the $(1-\gamma)^{-2}$ contribution to the displacement $\|Q^\star_{\tau_1}-Q^\star_{\tau_0}\|_\infty$. Typical instances include gradual covariate or domain shifts (sensor biases or population mix drifting over time), linear or near-linear ramps (curricula, staged deployments) where $\Curv\approx 0$ but the cumulative movement is non-negligible, and seasonal or migration-like changes that produce a sizeable net displacement. When planning across an interval with large $\PL$, one should throttle aggressiveness roughly proportionally to the estimated length—smaller learning rates, stickier target networks, and modestly increased planning—all encoded later in \Cref{sec:scheduler}. In nearly stationary patches, $\PL\approx 0$ and the scheduler reverts to a standard stationary regime.

\emph{Second-order information (curvature $\Curv$).}
The curvature integrates the second-order speed
\[
\|\partial_{\tau\tau} r_\tau\|_\infty
+ L_s\,\sup_{s,a}\|\partial_{\tau\tau} P_\tau(\cdot\mid s,a)\|_{W_1^\ast}.
\]
It becomes large when the \emph{rate} of change itself varies quickly (speed-ups, slow-downs, or bends), even if the net movement (i.e., $\PL$) is modest. In \Cref{thm:path-value} it controls the $(1-\gamma)^{-3}$ term, signaling how rapidly stale estimates expire. Curvature dominates in smooth but sharp schedule ramps (policy handover windows, softened phase transitions), oscillatory regimes (day/night or weekday/weekend cycles, where a full period may cancel in length but repeatedly accelerates and decelerates), and in shock-smoothing where abrupt real-world changes are filtered into short, high-second-derivative segments. High $\Curv$ suggests adding inertia and regularization—slower target updates, stronger trust regions or penalties, and short-horizon re-evaluation with increased planning depth—so that estimates are not invalidated between successive updates.

\emph{Third-type information (kinks via $g_\tau$ and $\Phi$).}
The global action gap $g_\tau$ quantifies how far the system is from an optimal-action tie. At points where $g_\tau=0$ the maximizer may switch and the optimal Bellman operator loses differentiability; the kink set $\mathcal K$ thus carries an \emph{extra} burden not reflected by $\PL$ or $\Curv$. The penalty $\Phi(\mathcal K,\mathrm{gap})$ assigns finite mass to these neighborhoods provided ties are not too prolonged. This term dominates in near-tie regimes (two actions nearly optimal across a span of $\tau$ so that small noise flips the maximizer), in discrete switches typical of combinatorial control, and in multi-modal continuous control where distinct modes (e.g., locomotion gaits) are optimal on different segments of the path. When the empirical gap drops below a threshold, stabilizing updates (slower or temporarily frozen policy/target updates) and allocating extra evaluation or planning are prudent to avoid chattering, as reflected in the kink-aware mappings in \Cref{sec:scheduler}. 

As a rule of thumb: steady long drifts (large $\PL$, small $\Curv$, healthy gap) call for cautious but steady progress; rapid accelerations with little net movement (small $\PL$, large $\Curv$) call for inertia and short-horizon reassessment; imminent switches (small gap, regardless of $\PL/\Curv$) call for conservative improvement and extra evaluation; and the combination of large $\PL$ with small gap is the most demanding and benefits from both conservative learning and increased planning or a slower path schedule in~$\tau$.

\section{Additional Geometry Details for Sec.~\ref{sec:operator}}
\label{app:feasible-geometry}
This appendix collects deferred derivations and computable surrogates for the feasible-geometry constructions in Sec.~\ref{sec:feasible-geometry}, including the projected feasible cone (used in Sec.~\ref{subsec:final-cone}) and finite-state bounds (used in Sec.~\ref{subsec:final-finite}).

\subsection{Projected feasible cones and regularity conditions}
\label{app:cone-derivation}
The tubes and ellipsoids above quantify \emph{how far} we may move in parameter space. In many settings, however, not all directions are equally benign: some directions immediately reduce the action gap or push $Q^\star$ towards non-smooth regimes, while others keep us inside the regular region for longer. This subsection makes this directional dependence explicit by projecting the tangent cone of the constrained solution manifold onto the parameter space.

Consider the multi-parameter embedding from Sec.~\ref{subsec:final-ellipsoid}, and denote the parameter by $\theta\in\R^p$. At a regular point $(\theta,Q^\star(\theta))$, the solution manifold
\[
\mathcal M:=\{(\theta,Q):G_{\rm reg}(\theta,Q)=0\}
\]
has tangent space
$T\mathcal M
=\bigl\{(\dot\theta,\dot Q):\ \partial_\theta G_{\rm reg}\,\dot\theta+\partial_Q G_{\rm reg}\,\dot Q=0\bigr\},$
so that
\begin{equation}
\dot Q=-(\partial_Q G_{\rm reg})^{-1}\partial_\theta G_{\rm reg}\,\dot\theta
=J_\theta\,\dot\theta, 
\end{equation}
where $J_\theta=\partial Q^\star/\partial\theta$ is the Jacobian from~\eqref{eq:Jtheta-final}. Now introduce inequality constraints
\[
h_j(\theta,Q)\ \ge\ 0,\qquad j=1,\dots,m,
\]
such as the gap constraint $h_{\rm gap}(\theta,Q)=g_{\rm gap}(\theta)-\xi\ge0$ or additional safety margins on value components. The Bouligand tangent cone of the constrained manifold at $(\theta,Q^\star(\theta))$ is
\[
T_{\rm feas}(\theta)
=\bigl\{(\dot\theta,\dot Q)\in T\mathcal M:\ D h_j(\theta,Q^\star(\theta))(\dot\theta,\dot Q)\ge0\ \text{for all active }j\bigr\},
\]
where ``active'' means $h_j(\theta,Q^\star(\theta))=0$ and
\[
D h_j(\theta,Q)(\dot\theta,\dot Q)
=\langle\nabla_\theta h_j(\theta,Q),\dot\theta\rangle
+\langle\nabla_Q h_j(\theta,Q),\dot Q\rangle.
\]

Substituting $\dot Q=J_\theta\dot\theta$ and projecting onto the parameter space yields the \emph{projected feasible cone} of directions:
\begin{equation}
\label{eq:cone-final}
\mathsf C(\theta)
=\Big\{\dot\theta\in\R^p:\ \big\langle \nabla_\theta h_j(\theta,Q^\star(\theta))
+ J_\theta^\top \nabla_Q h_j(\theta,Q^\star(\theta)),\ \dot\theta\big\rangle \ge 0\ \ \forall j\in\mathcal A(\theta)\Big\},
\end{equation}
where $\mathcal A(\theta)$ is the set of active constraints. The vector
\[
\nabla_\theta h_j(\theta,Q^\star(\theta))
+ J_\theta^\top \nabla_Q h_j(\theta,Q^\star(\theta))
\]
is precisely the gradient of the \emph{composite} constraint
$H_j(\theta):=h_j(\theta,Q^\star(\theta))$, obtained by the chain rule. Thus, $\mathsf C(\theta)$ is the set of directions $\dot\theta$ that do not decrease any active constraint to first order.

Under a standard linear-independence qualification for active constraints, the set $\mathsf C(\theta)$ coincides with the projection of the Bouligand tangent cone onto parameter space. This justifies calling $\mathsf C(\theta)$ the \emph{first-order directionally feasible cone}: any infinitesimal move in a direction $\dot\theta\in\mathsf C(\theta)$ keeps us inside all active constraints to first order, while directions outside the cone tend to violate at least one constraint (e.g., shrink the action gap below~$\xi$). In combination with the ellipsoidal radii from Sec.~\ref{subsec:final-ellipsoid}, this cone identifies which directions are safe to move in, and by how much, from the perspective of both value deviation and constraint preservation.

\subsection{Finite-state specialization and computable surrogates}
\label{app:finite-surrogates}

Classical Dobrushin-type inequalities relate $\|\cdot\|_{W_1^\ast}$ to induced matrix norms. In particular, for a finite state space endowed with a ground metric $d_\Sspace$, there exists a constant $\alpha>0$ (depending only on $d_\Sspace$) such that for any signed measure $\xi$,
\[
\|\xi\|_{W_1^\ast}\ \le\ \alpha\,\|\xi\|_1,
\]
and for kernels this yields
\[
\sup_{s,a}\|\partial_\tau P_\tau(\cdot\mid s,a)\|_{W_1^\ast}
\ \le\ \alpha\,\sup_{s,a}\|\partial_\tau P_\tau(\cdot\mid s,a)\|_1,
\]
with analogous bounds for $\partial_{\tau\tau}P_\tau$. Plugging these into the speed and curvature densities $v_\tau$ and $\kappa_\tau$ in \eqref{eq:g-density-final}--\eqref{eq:kappa-density-final} produces fully computable upper bounds expressed solely in terms of matrix derivatives and $\ell_1$-type norms.

For a multi-parameter $\theta$, Jacobian–vector products
\[
u\ \mapsto\ J_\theta u
\]
can be computed without forming $J_\theta$ explicitly. From~\eqref{eq:Jtheta-final}, for any direction $u\in\R^p$ we solve the linear system
\[
(I-\gamma P^{\pi^\star}_\theta)\,x_u
=\partial_\theta r_\theta\,u+\gamma\,\partial_\theta P^{\pi^\star}_\theta V^\star_\theta,
\]
and obtain $J_\theta u=x_u$. This can be done with standard dynamic-programming solvers or linear-system methods. Randomized probing (e.g., Hutchinson-type estimators) then estimates the pullback metric
\[
\mathbf G_\theta = J_\theta^\top W J_\theta
\]
from a small number of Jacobian–vector products, again without materializing $J_\theta$.

Altogether, Eqs.~\eqref{eq:Greg-final}--\eqref{eq:Jblocks-final} embed the feasible geometry
\emph{inside the native MDP}: the optimal fixed points form a smooth submanifold on
$\mathcal R$; the Jacobian induces value-relevant metrics that deliver explicit
first- and second-order tubes in $\tau$ and ellipsoidal feasible sets in multi-parameter
$\theta$, with gap-safe constraints excluding non-smooth switches and projected cones selecting safe directions. These sets are
\emph{intrinsic}—they require no iterative line-search and hold independently of any solver—and in finite MDPs they admit concrete matrix surrogates that make the geometry numerically accessible.

\section{Dynamic regret decomposition and design objective}
\label{app:regret}

This appendix complements Def.~\ref{def:dynreg} by showing how dynamic regret along a monotone homotopy path
decomposes into (i) a solver-agnostic tracking term governed by the geometric triple
$(\mathrm{PL},\mathrm{Curv},\Phi)$, and (ii) a statistical/approximation term determined by data noise and
function approximation.

\subsection{Dynamic regret along a homotopy path}
\label{sec:dynreg}

Online performance decomposes into a solver-agnostic \emph{tracking difficulty} determined by $(\mathrm{PL},\mathrm{Curv},\Phi)$ and a \emph{statistical/approximation} component determined by data and function class. This separation justifies scheduling policies that adapt aggressiveness to the estimated geometry while maintaining stable stochastic approximation.

We measure performance against the instantaneous optimal value for the current environment along the path. To obtain a scalar notion of regret, we fix a reference initial-state distribution $d_0$ (e.g., the task’s start-state distribution) and compare value functions in expectation under $d_0$.

\begin{definition}[Dynamic regret]
\label{def:dynreg}
Let $\{\tau_t\}_{t=1}^T\subset[0,1]$ be monotone and let $\pi_t$ be the policy produced at time $t$. Define
\begin{equation}
\mathrm{DynReg}(T)
=\sum_{t=1}^T \big\langle d_0,\ V^\star_{\tau_t}-V^{\pi_t}_{\tau_t}\big\rangle.
\end{equation}
\end{definition}

Thus $\mathrm{DynReg}(T)$ aggregates the per-round suboptimality of the online policies $\pi_t$ relative to the pathwise optimal policies $\pi^\star_{\tau_t}$, as seen from the same starting-state law $d_0$. The geometric results from Sec.~\ref{sec:geometry} bound how quickly $V^\star_{\tau_t}$ itself can move, while the algorithms in Sec.~\ref{sec:algorithms} control how well $\pi_t$ can keep up.

To relate regret to geometry without committing to any architecture, we impose a generic one-step contraction with noise and bias on the value (or action-value) iterates of an abstract base solver.

\subsection{One-step contraction abstraction}
\label{app:regret-solver}

To relate regret to geometry without committing to a specific architecture, we impose a generic one-step
contraction with bounded noise and bias on the value (or action-value) iterates of an abstract base solver.

\begin{assumption}[One-step contraction with bounded noise and bias]
\label{ass:solver}
Let $Q_t$ denote the value or action-value estimate maintained by the base solver at time $t$, and let $\mathcal F_t$ be the filtration generated by the history up to $t$. There exists $\rho\in(0,1)$ such that, conditionally on the past,
\begin{equation}
\begin{aligned}
\E\big[\|Q_{t+1}-Q^\star_{\tau_t}\|_\infty\,\big|\,\mathcal F_t\big]
\ \le\ \rho\,\|Q_{t}-Q^\star_{\tau_t}\|_\infty + \sigma_t + \beta_t,
\end{aligned}    
\end{equation}
where $(\sigma_t)$ is a centered noise term with $\sup_t \E[\sigma_t^2]\le \sigma^2<\infty$, and $\beta_t$ is a function-approximation bias term with $\sup_t \beta_t\le \beta<\infty$.
\end{assumption}

This assumption abstracts a wide range of practical algorithms (tabular and deep Q-learning, fitted Q-iteration, actor--critic with compatible function approximation, etc.): $\rho$ is the effective contraction factor of a single update around the current fixed point, $\sigma_t$ summarizes sampling variance and bootstrap noise, and $\beta_t$ summarizes the structural error induced by an imperfect function class or optimization.

\subsection{Regret decomposition driven by path geometry}
\label{app:regret-decomp}

Combining the contraction recursion with the pathwise bound on $Q^\star_{\tau_t}$ from Theorem~\ref{thm:path-value} yields a decomposition of dynamic regret into geometric and algorithmic contributions.

\begin{proposition}[Regret decomposition]
\label{prop:regret}
Under Assumptions~\ref{ass:mixing} and~\ref{ass:solver}, there exist constants $C_{\mathrm{trk}},C_{\mathrm{stat}}>0$ (depending only on $(1-\rho)^{-1}$ and universal constants) such that, for any monotone path $(\tau_t)_{t=1}^T$,
\begin{equation}
\begin{aligned}
\E[\mathrm{DynReg}(T)]
\ \le\ &\ C_{\mathrm{trk}}\Bigg(
   \frac{\mathrm{PL}}{(1-\gamma)^2}
 + \frac{\mathrm{Curv}}{(1-\gamma)^3}
 + \Phi(\mathcal K,\mathrm{gap})\Bigg)\\[0.25em]
&\ +\ C_{\mathrm{stat}}\,
   \frac{\sigma\sqrt{T}+\beta T}{(1-\gamma)^2\,(1-\rho)}.
\end{aligned}
\end{equation}
\end{proposition}

The first line is \emph{purely geometric}: it depends only on how much the optimal fixed point moves along the path $(\mathrm{PL},\mathrm{Curv})$ and on the burden of non-differentiable switches $\Phi(\mathcal K,\mathrm{gap})$, as characterized in Theorem~\ref{thm:path-value}. The second line is \emph{purely algorithmic}: it aggregates the statistical variance $\sigma^2$ and approximation bias $\beta$, scaled by the contraction margin $(1-\rho)$ and the discount. In particular:

- Even an oracle-quality solver cannot beat the geometric term: no algorithm can make dynamic regret smaller than $C_{\mathrm{trk}}(\mathrm{PL}/(1-\gamma)^2+\mathrm{Curv}/(1-\gamma)^3+\Phi)$ up to constants.
- Conversely, if the path geometry is benign (small $(\mathrm{PL},\mathrm{Curv},\Phi)$), then the dominant contribution comes from the statistical term, which can be reduced by more data, better function approximation, or stronger contraction (smaller $\rho$).

This decomposition motivates a design objective for non-stationary RL: (i) estimate path geometry online and adapt learning rates, target updates, regularization, or planning budgets monotonically with its difficulty, and (ii) simultaneously maintain a contracting, low-variance, low-bias base solver. The next subsection makes this concrete.

\subsection{Decomposition into geometric and algorithmic terms}
\label{app:regret-objective}

We now instantiate the design objective suggested by Proposition~\ref{prop:regret}. The central idea is to wrap a base solver (Q-learning / actor--critic, or MCTS) inside a \emph{homotopy-tracking} scheduler that:

1. Maintains online estimates of the local path geometry (proxies for first-order length, second-order curvature, and kink burden).
2. Maps these estimates monotonically to algorithmic knobs (learning rates, target-update speeds, regularization strength, planning budgets).
3. Incorporates hysteresis and smoothing so that small fluctuations in the proxies do not cause unstable oscillations in the schedule.

The following algorithms are not new solvers but \emph{wrappers}: any reasonable base Q/AC or MCTS implementation can be plugged into them.
\label{sec:algorithms}

\begin{algorithm}[t]
\caption{Homotopy-Tracking (HT) RL: Algorithm-agnostic wrapper for Q/AC}
\label{alg:ht-rl}
\begin{algorithmic}[1]
\small
\STATE \textbf{Inputs:} base solver $\mathsf{Base}$ (Q-learning / Actor--Critic), replay buffer $\mathcal D$;
windows $W_1,W_2$; EMA coefficient $\beta\in[0,1)$; update period $H$; hysteresis threshold $\Delta_{\mathrm{hys}}$;
base rates $\eta_0,\tau_0,\lambda_0$; scalings $\alpha_{1,2},\beta_{1,2},c_{1,2}$; thresholds $\delta,\varepsilon_{\mathrm{gap}}$; feature map $\phi$; sample count $N$.
\STATE Initialize smoothed proxies $\tilde{\PL},\tilde{\Curv},\tilde{\mathrm{Kink}}$, target parameters $\bar\theta$.
\FOR{$t=1,2,\dots,T$}
  \STATE Collect $(s_t,a_t,r_t,s'_t)$ at parameter index $\tau_t$; push to $\mathcal D$.
  \STATE Sample a minibatch $\mathcal B_t\subset\mathcal D$ and extract windowed slices from $[t-W_1+1,t]$ and $[t-2W_1+1,t-W_1]$.

  \STATE \textbf{Estimate incremental path metrics (first-/second-order):}
  \STATE $\widehat r_t(s,a)\leftarrow$ windowed reward estimator (e.g., EMA or MoM) for each $(s,a)\in\mathcal B_t$.
  \STATE $\Delta r^{(\infty)}_t \leftarrow \max_{(s,a)\in \mathcal B_t} \big| \widehat r_t(s,a) - \widehat r_{t-W_1}(s,a)\big|$.
  \STATE For each $(s,a)\in\mathcal B_t$, draw $N$ next states from the recent window and set
  \[
  \mu_t(s,a) \gets \frac{1}{N}\sum_{i=1}^N \phi\big(s'^{(i)}\mid s,a\big),\quad
  \mu_{t-W_1}(s,a)\ \text{analogously}.
  \]
  \STATE $\Delta P^{(\phi)}_t \leftarrow \max_{(s,a)\in \mathcal B_t} \|\mu_t(s,a)-\mu_{t-W_1}(s,a)\|_2$.
  \STATE $\Delta\widehat{\PL}_t \leftarrow \Delta r^{(\infty)}_t + L_s \cdot \Delta P^{(\phi)}_t$ \quad (proxy for local path speed).
  \STATE $\Delta\widehat{\Curv}_t \leftarrow \big|\Delta\widehat{\PL}_t - \Delta\widehat{\PL}_{t-W_2}\big|$ \quad (proxy for local curvature).

  \STATE \textbf{Gap/kink proxies (third-type information):}
  \STATE \hspace{0.6em}Discrete actions: $\widehat{\mathrm{gap}}_t\leftarrow \min_{s\in\mathcal B_t}\big(\max_a Q_\theta(s,a)-\max_{a\ne a^\star}Q_\theta(s,a)\big)$.
  \STATE \hspace{0.6em}Continuous actions: obtain a top-1 mode $a_{\mathrm{mode}}$ (actor or local optimizer) and $K$ high-probability samples $\{a^{(k)}\}$ (e.g., CEM); set top-2 by values $Q_\theta$ and compute $\widehat{\mathrm{gap}}_t$.
  \STATE $\mathrm{Kink}_t \leftarrow \mathbf{1}\{\widehat{\mathrm{gap}}_t\le \varepsilon_{\mathrm{gap}}\}$.

  \IF{$t \bmod H=0$}
    \STATE \textbf{EMA \& hysteresis smoothing:}
    \STATE $\tilde{\PL}_t \leftarrow \beta \tilde{\PL}_{t-H} + (1-\beta)\,\Delta\widehat{\PL}_t$ \quad (optionally z-score normalized)
    \STATE $\tilde{\Curv}_t \leftarrow \beta \tilde{\Curv}_{t-H} + (1-\beta)\,\Delta\widehat{\Curv}_t$
    \STATE $\tilde{\mathrm{Kink}}_t \leftarrow \beta \tilde{\mathrm{Kink}}_{t-H} + (1-\beta)\,\mathrm{Kink}_t$
    \IF{all changes in $(\tilde{\PL}_t,\tilde{\Curv}_t,\tilde{\mathrm{Kink}}_t)$ are $<\Delta_{\mathrm{hys}}$}
      \STATE keep $(\eta_t,\tau_t,\lambda_t)$ unchanged (hysteresis)
    \ELSE
      \STATE \textbf{Schedule geometry-aware hyperparameters:}
      \STATE $\eta_t \leftarrow \mathrm{clip}_{[\eta_{\min},\eta_{\max}]}\!\left(
      \frac{\eta_0}{1+\alpha_1 \tilde{\PL}_t + \alpha_2 \tilde{\Curv}_t}\right)$
      \STATE $\tau_t \leftarrow \mathrm{clip}_{[\tau_{\min},\tau_{\max}]}\!\left(
      \frac{\tau_0}{1+\beta_1 \tilde{\mathrm{Kink}}_t(1+\beta_2/\max\{\widehat{\mathrm{gap}}_t,\delta\})}\right)$
      \STATE $\lambda_t \leftarrow \lambda_0\big(1+c_1\tilde{\PL}_t+c_2\sqrt{\tilde{\Curv}_t}\big)$
    \ENDIF
  \ENDIF

  \STATE \textbf{One base step} $\mathsf{Base}$ with $(\eta_t,\tau_t,\lambda_t)$ on minibatch $\mathcal B_t$ (off-policy or on-policy).
  \STATE Target update: $\bar\theta \leftarrow (1-\tau_t)\bar\theta + \tau_t \theta$.
\ENDFOR
\end{algorithmic}
\end{algorithm}

Algorithm~\ref{alg:ht-rl} implements the design objective from Proposition~\ref{prop:regret}: as the estimated path length and curvature increase, the learning rate $\eta_t$ decreases and regularization $\lambda_t$ increases; as kink indicators rise or the empirical gap shrinks, the target-update rate $\tau_t$ slows down, injecting inertia near potential action switches. The precise functional forms are not unique; any monotone mappings with similar qualitative behavior would be compatible with the theory.

\begin{algorithm}[t]
\caption{Homotopy-Tracking MCTS (HT-MCTS)}
\label{alg:ht-mcts}
\begin{algorithmic}[1]
\small
\STATE \textbf{Inputs:} model $\widehat P$, reward $\widehat r$; base budget $B_0$, depth $D_0$; caps $B_{\max},D_{\max}$;
windows $W_1,W_2$; scalings $\gamma_{1,2,3}$; thresholds $\varepsilon_{\mathrm{gap}},\delta$; update period $H$; EMA coefficient $\beta$.
\FOR{$t=1,\dots,T$}
  \STATE Observe $s_t$; estimate $(\Delta\widehat{\PL}_t,\Delta\widehat{\Curv}_t,\widehat{\mathrm{gap}}_t,\mathrm{Kink}_t)$ as in Alg.~\ref{alg:ht-rl}.
  \IF{$t \bmod H=0$} 
    \STATE Update smoothed $(\tilde{\PL}_t,\tilde{\Curv}_t,\tilde{\mathrm{Kink}}_t)$ with EMA and hysteresis.
  \ENDIF
  \STATE Set geometry-aware depth and budget:
  \[
    D_t \leftarrow \min\!\Big\{D_{\max},\ \Big\lfloor D_0 + \gamma_1 (1+\tilde{\PL}_t)
                   + \gamma_2 \sqrt{1+\tilde{\Curv}_t}
                   + \gamma_3 \,\tfrac{\tilde{\mathrm{Kink}}_t}{\max\{\widehat{\mathrm{gap}}_t,\delta\}} \Big\rceil\Big\},
  \]
  \[
    B_t \leftarrow \min\!\Big\{B_{\max},\ \Big\lfloor B_0 \big( 1+\gamma_1 \tilde{\PL}_t + \gamma_2 \tilde{\Curv}_t \big)\Big\rceil\Big\}.
  \]
  \STATE Run MCTS (UCT/PUCT) from $s_t$ with $(B_t,D_t)$ under $(\widehat P,\widehat r)$; execute $a_t=\argmax_a \mathrm{visit}(s_t,a)$.
  \STATE Periodically update $(\widehat P,\widehat r)$ with step size proportional to $\eta_t$ from Alg.~\ref{alg:ht-rl}.
\ENDFOR
\end{algorithmic}
\end{algorithm}

Algorithm~\ref{alg:ht-mcts} plays the analogous role for planning-based control. When the estimated path length and curvature are small, MCTS runs at its base depth and budget; as the environment drifts faster or its curvature increases, HT-MCTS automatically allocates more depth and simulations. Near kink-like regimes (small estimated gap, large kink proxy), the depth is further boosted so that planning can resolve impending mode switches. In both HT-RL and HT-MCTS, the geometry-aware schedules aim to keep the tracking term in Proposition~\ref{prop:regret} controlled while the base solver handles the statistical/approximation term.

\label{app:proxy-details}


\subsection{Feature choices for transition drift}
\label{app:proxy-phi}
The transition proxy $\Delta P_t^{(\phi)}$ compares empirical next-state feature means
$\mu_t(s,a)=\frac{1}{N}\sum_{i=1}^N \phi(s'^{(i)})$ across windows.
In practice $\phi$ can be instantiated as (i) a shared representation layer of the encoder,
(ii) fixed random features, or (iii) a low-dimensional projection of observations.
A bounded and approximately Lipschitz $\phi$ makes $\|\mu_t-\mu_{t-W_1}\|_2$
a stable surrogate for value-relevant distribution drift.

\subsection{Continuous-action gap proxy}
\label{app:proxy-gap-continuous}
For continuous actions, we approximate the top two action-values for each $s$ by combining a
mode estimate (e.g., actor output $a_{\rm mode}$ or a local maximizer of $Q_\theta(s,\cdot)$)
with a small set of high-probability samples (e.g., CEM or Gaussian perturbations around $a_{\rm mode}$).
We then compute the empirical gap as the difference between the largest and second-largest
values among the candidate set, and take the minibatch minimum.

\subsection{EMA, normalization, clipping, and hysteresis}
\label{app:proxy-stabilization}
For any proxy $x_t$ (e.g., $x_t=\Delta\widehat{\mathrm{PL}}_t$), we form an EMA
\[
\tilde{x}_t=\beta\tilde{x}_{t-1}+(1-\beta)x_t,
\]
optionally normalize by an online running mean/variance (or robust quantiles),
and clip $\tilde{x}_t$ to a fixed range to prevent rare spikes from dominating the scheduler.
To reduce oscillations, we update scheduler parameters only every $H$ steps and only when the
change exceeds a hysteresis threshold $\Delta_{\mathrm{hys}}$.

\section{Stability and stochastic-approximation compatibility of the scheduler}
\label{app:scheduler-stability}

This appendix provides technical statements deferred from Secs.~\ref{sec:scheduler}--\ref{sec:convergence}:
(i) scheduled hyperparameters have bounded variation under EMA, clipping and hysteresis (no chattering),
(ii) scheduled learning rates remain compatible with Robbins--Monro regimes, and
(iii) stationary convergence is recovered when the homotopy path is frozen.

\subsection{Bounded variation and no-chattering}
\label{app:no-chatter}

\begin{theorem}[No-chattering under clipping and hysteresis]
\label{thm:no-chatter}
Suppose the raw proxies $(\Delta\widehat{\mathrm{PL}}_t,\Delta\widehat{\mathrm{Curv}}_t,\widehat{\mathrm{gap}}_t,\mathrm{Kink}_t)$
have uniformly bounded second moments, EMA uses $\beta<1$, and scheduler updates occur every $H$ steps with
hysteresis threshold $\Delta_{\mathrm{hys}}>0$. Then each scheduled process
($\eta_t,\tau_t,\lambda_t,D_t,B_t$) is piecewise-constant with bounded variation.
Moreover, for any $\varepsilon>0$, the fraction of steps at which any hyperparameter changes by more than
$\varepsilon$ can be made arbitrarily small by taking $H$ and $\Delta_{\mathrm{hys}}$ sufficiently large.
\end{theorem}

Theorem~\ref{thm:no-chatter} formalizes the intended behavior of EMA and hysteresis: the scheduler reacts on
the scale of proxy windows, but does not introduce high-frequency oscillations that could destabilize the base solver.

\subsection{Robbins--Monro compatibility}
\label{app:rm}

\begin{lemma}[Scheduler preserves Robbins--Monro regimes]
\label{lem:rm}
Let base rates $(\eta_t^0)$ satisfy $\sum_t \eta_t^0=\infty$ and $\sum_t (\eta_t^0)^2<\infty$.
Define the scheduled rates by
\[
\eta_t=\mathrm{clip}_{[\eta_{\min},\eta_{\max}]}\!\left(
\frac{\eta_t^0}{1+\alpha_1\tilde{\mathrm{PL}}_t+\alpha_2\tilde{\mathrm{Curv}}_t}\right),
\]
with EMA smoothing ($\beta<1$) and hysteresis (updates only every $H$ steps if the change exceeds $\Delta_{\mathrm{hys}}$).
Then there exist constants $0<c\le C<\infty$ such that $c\,\eta_t^0\le \eta_t\le C\,\eta_t^0$ for all but a vanishing fraction of steps;
in particular, $\sum_t \eta_t=\infty$ and $\sum_t \eta_t^2<\infty$.
\end{lemma}

Lemma~\ref{lem:rm} ensures that geometry-aware rate modulation remains within standard stochastic-approximation
conditions and does not destroy asymptotic convergence when the path stabilizes.

\subsection{Stationary convergence as a special case}
\label{app:fixed-tau}

\begin{theorem}[Contraction at fixed $\tau$]
\label{thm:fixed-tau}
Fix $\bar\tau$. Under Assumption~\ref{ass:mixing}, the optimal Bellman operator $\Tcal_{\bar\tau}$ is a
$\gamma$-contraction in $\|\cdot\|_\infty$ with unique fixed point $Q^\star_{\bar\tau}$.
Any value-iteration/TD-style update that implements a one-step contraction in expectation
(Assumption~\ref{ass:solver} with $\tau_t\equiv\bar\tau$, $\sigma_t\!\to\!0$, $\beta_t\!\to\!0$) satisfies
$\E\|Q_t-Q^\star_{\bar\tau}\|_\infty\to0$ as $t\to\infty$.
\end{theorem}

Theorem~\ref{thm:fixed-tau} confirms that the proposed framework is a strict generalization of the stationary case:
when the homotopy path pauses, geometric load vanishes and classical contraction-based convergence is recovered.


\section{Experimental Details}
\label{app:exp-details}

This appendix provides the concrete MDP definitions, path parametrizations, and implementation
details for the experiments in Sec.~\ref{sec:experiments}.

\subsection{Synthetic ring MDPs}
\label{app:synthetic-ring}

\paragraph{MDP structure.}
We use a tabular ring MDP with $n$ states $\mathcal S=\{0,\dots,n-1\}$ and three actions
$\mathcal A=\{\mathrm{L},\mathrm{N},\mathrm{R}\}$. The next state is
\[
s'=
\begin{cases}
(s-1)\bmod n, & a=\mathrm{L},\\
s, & a=\mathrm{N},\\
(s+1)\bmod n, & a=\mathrm{R},
\end{cases}
\]
followed by a small mixing with probability $\epsilon$ to a uniform state:
$P_\tau(s'|s,a)=(1-\epsilon)\,\mathbf 1\{s'=\text{det}(s,a)\}+\epsilon/n$.
We fix $\gamma\in(0,1)$ and $n$ (e.g., $n=20$) and consider rewards of the form
\[
r_\tau(s,a)=w_a\,\exp\!\left(-\frac{d_{\mathrm{ring}}(s-c(\tau))^2}{2\sigma^2}\right),
\]
where $d_{\mathrm{ring}}$ is the wrap-around distance, $\sigma$ is a width parameter, and
$c(\tau)\in\{0,\dots,n-1\}$ is the center of a reward bump that moves along the ring.

\paragraph{Three regimes.}
We construct three homotopy paths $\tau\mapsto M(\tau)$:

\begin{itemize}
\item \textbf{Length-dominated path.}
We move the reward center linearly, $c(\tau) = \lfloor c_0 + (c_1-c_0)\tau \rfloor$, and linearly
interpolate action weights $w_a(\tau)$ between two fixed sets $(w_a^0)$ and $(w_a^1)$. Transition
mixing $\epsilon$ is kept constant. This produces large $\mathrm{PL}$ and negligible
$\mathrm{Curv}$, with no action-gap crossings.
\item \textbf{Curvature-dominated path.}
We use a smooth S-curve in parameter space,
$c(\tau)=\lfloor c_0 + (c_1-c_0)\,\sigma(\tau)\rfloor$ with
$\sigma(\tau)=3\tau^2-2\tau^3$, and similarly deform $(w_a(\tau))$. This yields similar net
movement (integrated length) but higher second derivatives in $\tau$, hence large curvature.
\item \textbf{Kink-prone path.}
We let two actions compete by choosing
$r_\tau(s,\mathrm{L})=\alpha(\tau)\,\tilde r_{\mathrm{L}}(s)$ and
$r_\tau(s,\mathrm{R})=(1-\alpha(\tau))\,\tilde r_{\mathrm{R}}(s)$ with $\alpha(\tau)$ crossing
$1/2$ on an interval where $\tilde r_{\mathrm{L}}$ and $\tilde r_{\mathrm{R}}$ are comparable.
This creates small global gaps and isolated $\tau$ where the optimal action switches, generating
nonzero kink mass $\Phi$.
\end{itemize}

\paragraph{Geometry and tubes.}
We discretize $\tau$ on a grid $\{\tau_k\}_{k=0}^{K}$, solve for $Q^\star_{\tau_k}$ by value
iteration, and approximate the geometric quantities by finite differences:
\[
\mathrm{PL}\approx \sum_k \Big(\|\partial_\tau r_{\tau_k}\|_\infty
+L_s\sup_{s,a}\|\partial_\tau P_{\tau_k}(\cdot|s,a)\|_{W_1^\ast}\Big)\Delta\tau,
\]
with analogous formulas for $\mathrm{Curv}$ and the kink mass $\Phi$ (using the global gap and
a small cutoff $\delta$). The densities $g_\tau$ and $\kappa_\tau$ in
Thms.~\ref{thm:final-tube1}--\ref{thm:final-tube2} are approximated by first- and second-order
finite differences of $Q^\star_{\tau_k}$, and the tubes are evaluated by checking whether
$\|Q^\star_{\tau_k}-Q^\star_{\tau_0}\|_\infty\le\varepsilon$ for all $\tau_k$ inside the
predicted radius.

\subsection{Deep control benchmarks}
\label{app:deep-details}

\paragraph{Environments and paths.}
We use standard continuous-state control tasks:

\begin{itemize}
\item \textbf{LunarLander.}
We vary the gravity magnitude $g(\tau)$ between $g_{\min}$ and $g_{\max}$,
$g(\tau)=g_{\min}+(g_{\max}-g_{\min})\tau$, and optionally adjust reward shaping coefficients
for soft landing vs.\ fuel usage. Linear schedules yield length-dominated paths; S-curves in $g$
produce curvature-dominated paths.
\item \textbf{Noisy-LunarLander.}
We inject zero-mean Gaussian noise into the dynamics or actions with standard deviation
$\sigma(\tau)$, using $\sigma(\tau)$ linear vs.\ ramped in $\tau$ to produce length-dominated vs.\
curvature-dominated regimes.
\item \textbf{Acrobot.}
We vary torque limits and/or target height smoothly in $\tau$. Kink-prone paths are created by
interpolating between two reward shapings that favor different swing-up strategies, leading to
near-ties in the action gap.
\end{itemize}

In each case, $\tau_t$ is increased monotonically during training (e.g., linearly with episode
index), so that early episodes see one regime and later episodes see another, with a controlled
non-stationary path in between.

\paragraph{Base algorithms and architectures.}
For discrete actions (LunarLander, Acrobot) we use DQN, Double-DQN, and Rainbow as base solvers.
The Q-network has two hidden layers of width 256 with ReLU activations; we use Adam with learning
rate $3\cdot 10^{-4}$, discount factor $\gamma=0.99$, and standard replay buffers. For continuous
variants we use SAC with two-layer 256-unit actor and critic networks, target smoothing, and
entropy tuning. All runs are averaged over multiple random seeds (e.g., 5 seeds) with identical
random initialization schemes across baselines and HT-RL variants.

\paragraph{Scheduler and proxies.}
HT-RL wraps each base solver using the scheduler of Sec.~\ref{sec:scheduler}. Replay-based proxies
are implemented as follows:
\begin{itemize}
\item $W_1$ and $W_2$ are chosen as sliding windows over recent transitions (e.g.,
$W_1=5\cdot 10^3$ steps, $W_2=3W_1$).
\item $\phi(s)$ is taken as the penultimate hidden layer of the Q-network (or critic), detached
from the computation graph. For tabular comparisons we instead use one-hot features.
\item The reward estimator $\widehat r_t(s,a)$ is an exponential moving average (EMA) over the
observed rewards for each $(s,a)$ in the minibatch.
\item Geometry proxies $(\Delta\widehat{\mathrm{PL}}_t,\Delta\widehat{\mathrm{Curv}}_t)$ and
gap/kink proxies $(\widehat{\mathrm{gap}}_t,\mathrm{Kink}_t)$ are computed exactly as in
Sec.~\ref{sec:proxies}, with EMA smoothing parameter $\beta\in[0.9,0.99]$, update period $H$
(e.g., $H=10^3$ steps), and hysteresis threshold $\Delta_{\mathrm{hys}}$ tuned coarsely on
held-out runs.
\end{itemize}
The scheduler outputs time-varying learning rates, target update coefficients, and
regularization/trust parameters; clipping bounds $(\eta_{\min},\eta_{\max})$ and
$(\tau_{\min},\tau_{\max})$ are chosen so that HT-RL remains in a similar range as the
hand-tuned static baselines.

\paragraph{Metrics.}
Dynamic regret is approximated by Monte Carlo rollouts from a fixed set of initial states,
evaluating both the current policy and a high-accuracy reference (trained longer with a fixed
environment). Tracking error is estimated by comparing $Q_\theta$ to a reference $Q^\star$
on a held-out state–action set, computed via model-based planning or long-horizon value iteration
depending on the task.

\subsection{Model-based planning with HT-MCTS}
\label{app:exp-details-htmcts}

\paragraph{Environments.}
For HT-MCTS we consider:
\begin{itemize}
\item A gridworld with moving obstacles and a goal that drifts along the boundary according to a
smooth schedule in $\tau$.
\item A model-based LunarLander variant where gravity or wind parameters follow a homotopy path
in $\tau$, and a learned model $(\widehat P,\widehat r)$ is updated from experience.
\end{itemize}

\paragraph{Planner and scheduling.}
We use UCT/PUCT as the base planner with a fixed exploration constant. The base depth $D_0$ and
simulation budget $B_0$ are chosen to match common practice for each task. HT-MCTS applies the
scheduler of Sec.~\ref{sec:scheduler} to set depth $D_t$ and budget $B_t$ as monotone functions of
$(\tilde{\mathrm{PL}}_t,\tilde{\mathrm{Curv}}_t,\tilde{\mathrm{Kink}}_t)$, capped by
$(D_{\max},B_{\max})$. We track total node expansions and average return. In all reported runs,
HT-MCTS achieves higher return at similar or lower total planning cost by concentrating deeper
search around episodes with high estimated geometric load.

\section{Additional Experimental Figures}
\label{app:exp-figures}


\begin{figure*}[t!]
  \centering
  \includegraphics[width=0.24\linewidth]{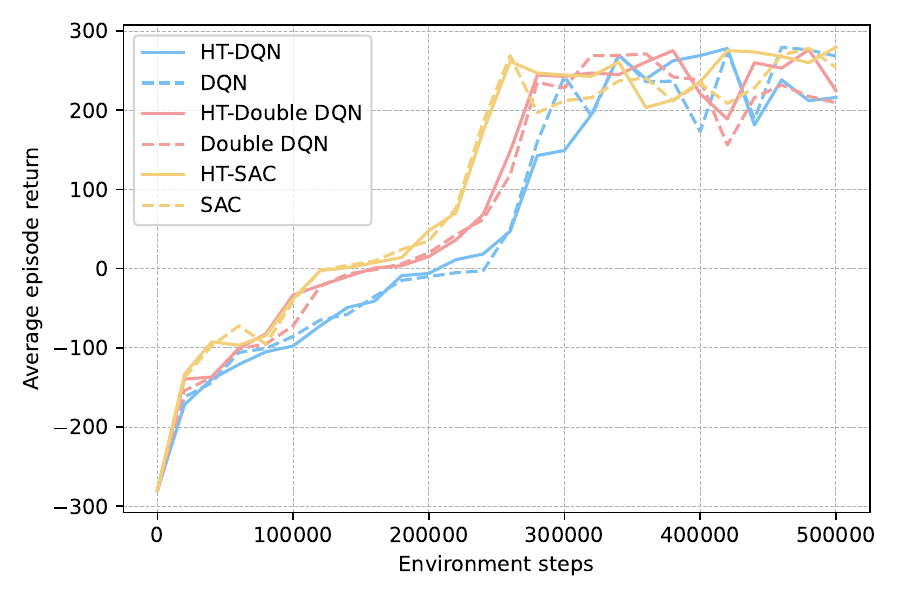}%
  \includegraphics[width=0.24\linewidth]{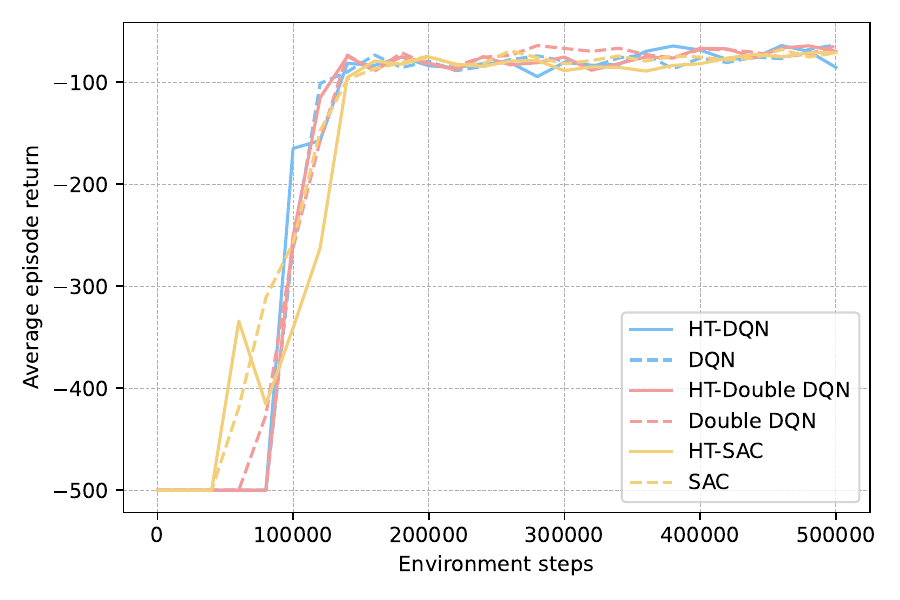}
  \includegraphics[width=0.24\linewidth]{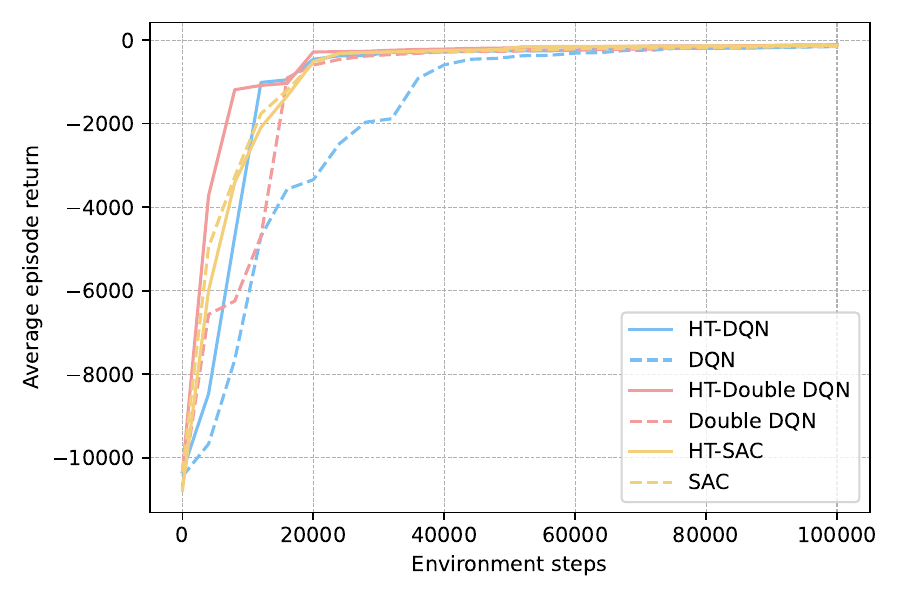}
  \includegraphics[width=0.24\linewidth]{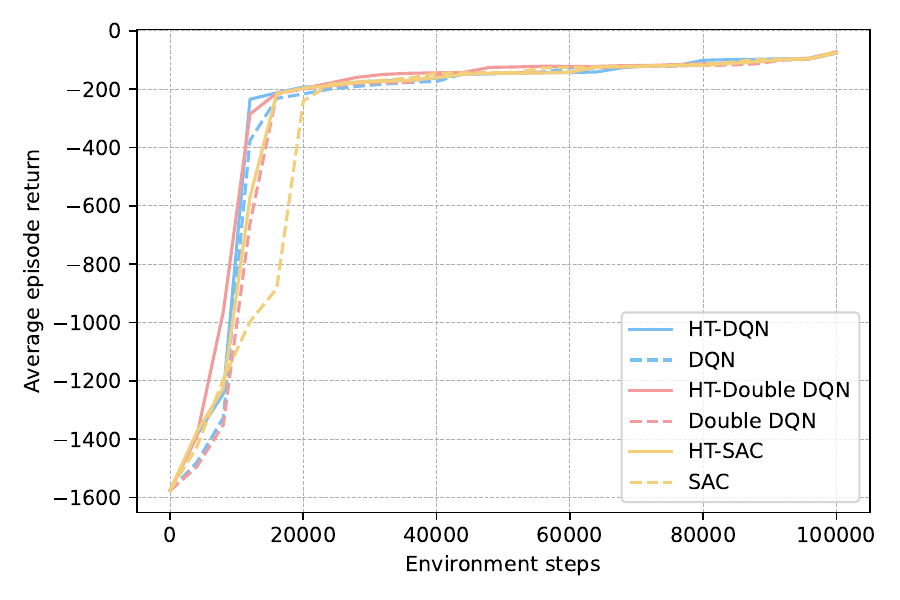}
  \caption{%
   Deep control benchmarks under \textbf{clean} non-stationary homotopy drift (no extra injected noise).Each panel reports average episode return versus environment steps for static baselines (dashed) and their HT-RL counterparts}
  \label{fig:deep-returns}
\end{figure*}

\begin{figure*}[t!]
  \centering
  \includegraphics[width=0.24\linewidth]{figures/lunar_noise_return.pdf}%
  \includegraphics[width=0.24\linewidth]{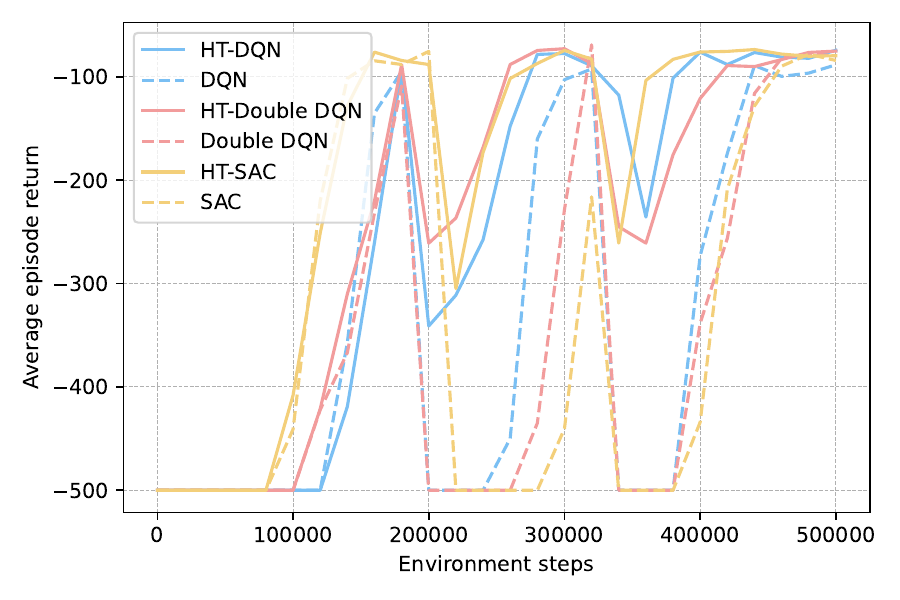}
  \includegraphics[width=0.24\linewidth]{figures/pointmass_noise_return.pdf}
  \includegraphics[width=0.24\linewidth]{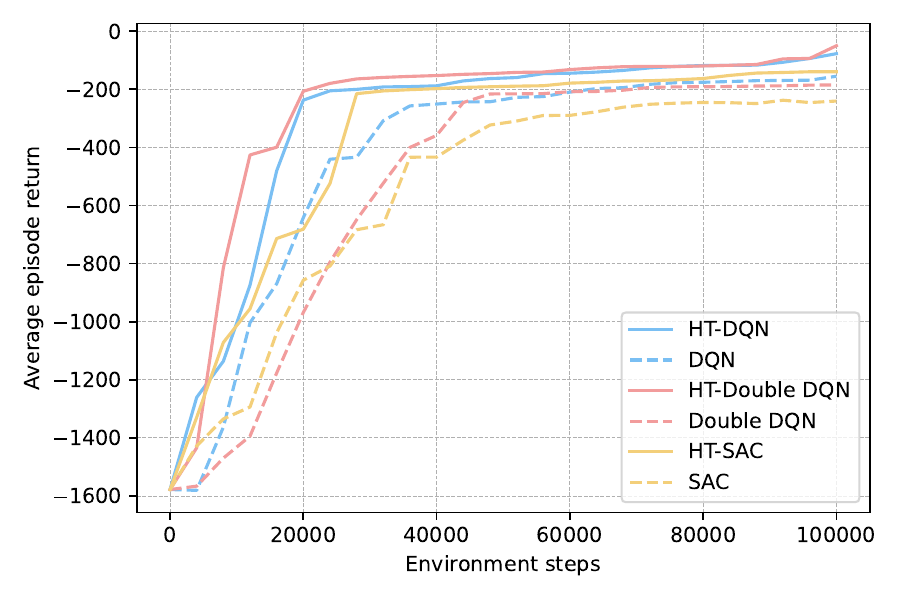}
  \caption{%
  Deep control benchmarks under non-stationary homotopy paths.
  Each panel shows average episode return vs.\ environment steps for different algorithms:
  static DQN/Double-DQN/SAC baselines and their HT-RL counterparts.
  HT-RL variants achieve lower tracking error and higher return, especially in
  curvature-dominated and kink-prone regimes.}
  \label{fig:deep-returns-noise}
\end{figure*}

\begin{figure*}[t!]
  \centering
  \includegraphics[width=0.24\linewidth]{figures/lunar_noisy_auc.pdf}%
  \includegraphics[width=0.24\linewidth]{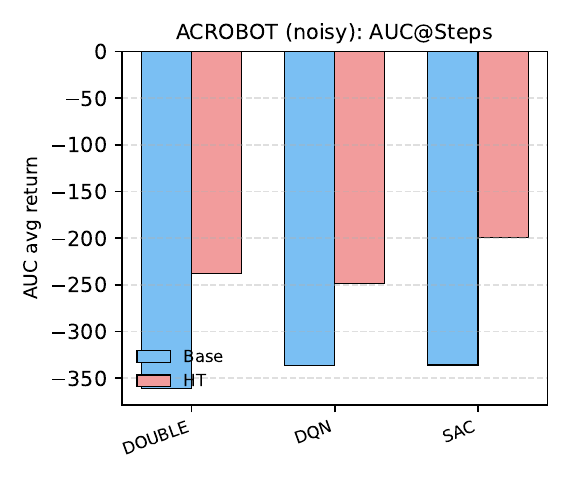}%
  \includegraphics[width=0.24\linewidth]{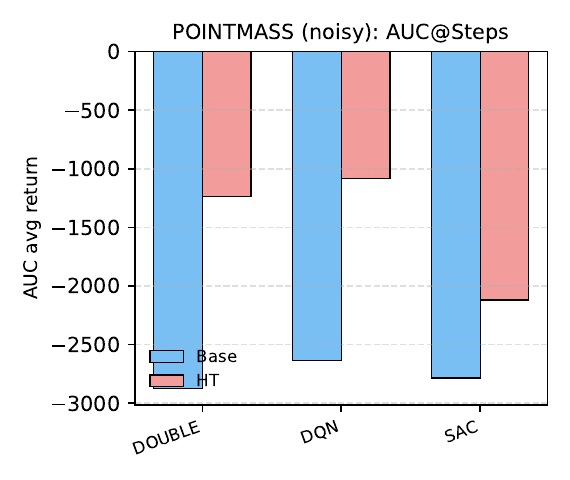}%
  \includegraphics[width=0.24\linewidth]{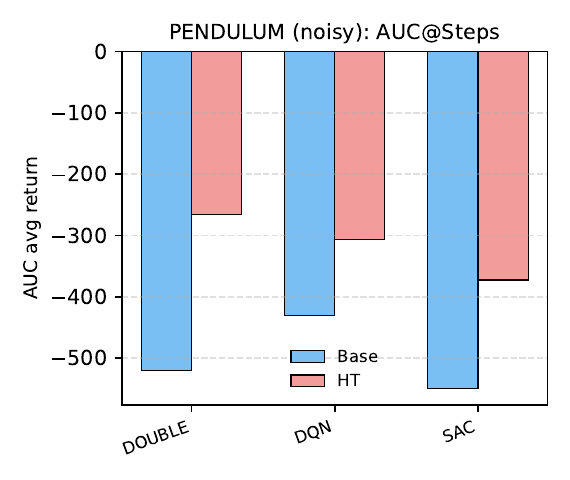}%
  \caption{%
  \textbf{AUC@Steps} (area under the return curve over the training budget), 
  Each panel corresponds to one environment and reports baseline vs.\ HT-RL for each matched solver (clean and noisy variants shown side-by-side inside each panel). Higher is better.}
  \label{fig:deep-auc-noise}
\end{figure*}

\begin{figure*}[t!]
  \centering
  \includegraphics[width=0.24\linewidth]{figures/lunar_noisy_final.pdf}%
  \includegraphics[width=0.24\linewidth]{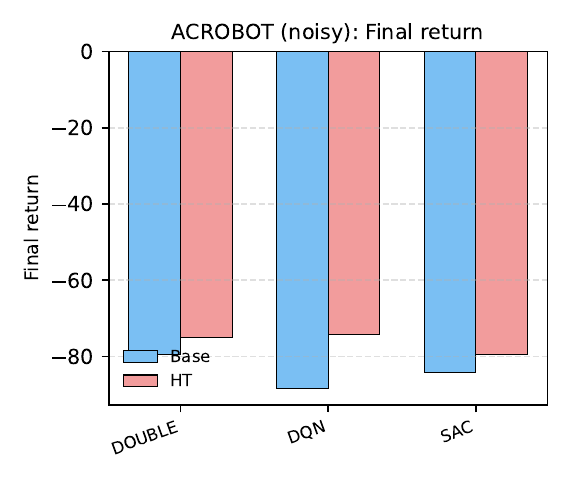}%
  \includegraphics[width=0.24\linewidth]{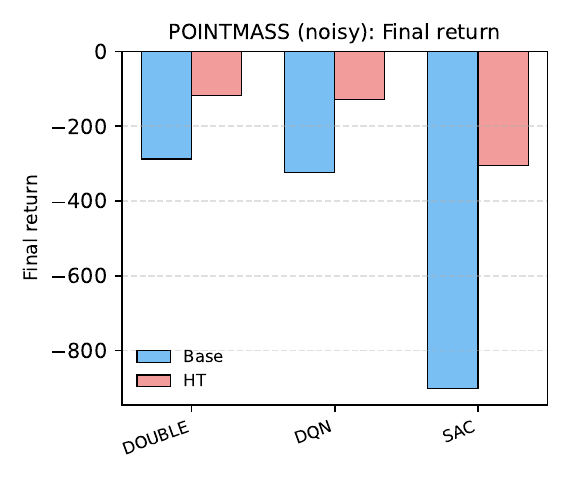}%
  \includegraphics[width=0.24\linewidth]{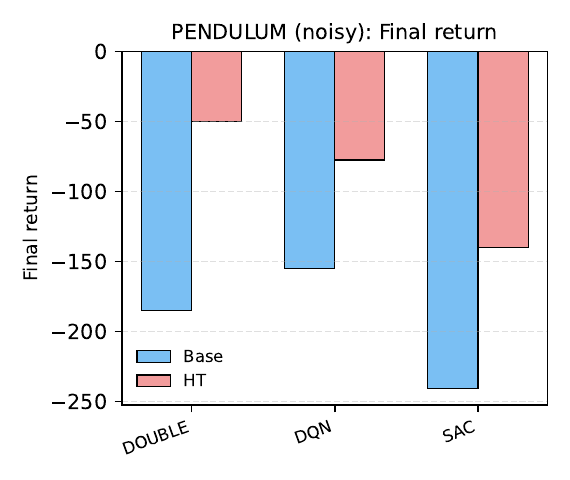}%
  \caption{%
  \textbf{Final evaluation return} 
  Each panel corresponds to one environment and reports baseline vs.\ HT-RL for each matched solver (clean and noisy variants shown side-by-side inside each panel). Higher is better.}
  \label{fig:deep-final-noise}
\end{figure*}

\section{Proof}

\subsection{Proof of Lemma~\ref{lem:envelope}}
\begin{proof}
By assumption $\tau\in\mathcal R$. We recall what this means. We assume that for this parameter $\tau$ the optimal action–value function $Q^\star_\tau$ admits a unique optimal action at every state and that the associated action gap is uniformly positive. More precisely, there exists a (deterministic) optimal policy $\pi^\star_\tau:\mathcal S\to\mathcal A$ such that
\[
\pi^\star_\tau(s)\in\argmax_{a\in\mathcal A} Q^\star_\tau(s,a)
\quad\text{for all }s\in\mathcal S,
\]
and such that the global gap
\[
\Delta_\tau
\;:=\;
\inf_{s\in\mathcal S}
\;\inf_{a\in\mathcal A\setminus\{\pi^\star_\tau(s)\}}
\Big( Q^\star_\tau(s,\pi^\star_\tau(s)) - Q^\star_\tau(s,a) \Big)
\]
satisfies
\[
\Delta_\tau > 0.
\]
In particular, for every state $s\in\mathcal S$ and every suboptimal action $a\neq\pi^\star_\tau(s)$ we have
\begin{equation}
Q^\star_\tau(s,\pi^\star_\tau(s)) - Q^\star_\tau(s,a)
\;\geq\; \Delta_\tau.
\label{eq:gap-at-tau}
\end{equation}

We now use continuity of $Q^\star_{\tilde\tau}$ with respect to $\tilde\tau$ to propagate this strict inequality to a neighborhood of $\tau$. Fix a constant
\[
\delta \;:=\; \frac{\Delta_\tau}{4} \;>\; 0.
\]
From \eqref{eq:gap-at-tau} we obtain, for all $s$ and all $a\neq\pi^\star_\tau(s)$,
\[
Q^\star_\tau(s,\pi^\star_\tau(s)) - Q^\star_\tau(s,a)
\;\geq\; \Delta_\tau
\;>\; 3\delta.
\]

For each fixed pair $(s,a)$ with $a\neq\pi^\star_\tau(s)$ we define the scalar function
\[
f_{s,a}(\tilde\tau)
\;:=\;
Q^\star_{\tilde\tau}(s,\pi^\star_\tau(s))
-
Q^\star_{\tilde\tau}(s,a).
\]
By assumption on the model, for each $(s,a)$ the maps $\tilde\tau\mapsto Q^\star_{\tilde\tau}(s,a)$ and $\tilde\tau\mapsto Q^\star_{\tilde\tau}(s,\pi^\star_\tau(s))$ are continuous, hence $f_{s,a}$ is a continuous real-valued function of $\tilde\tau$. At $\tilde\tau=\tau$ we have
\[
f_{s,a}(\tau)
=
Q^\star_\tau(s,\pi^\star_\tau(s))
-
Q^\star_\tau(s,a)
\;\geq\; 3\delta.
\]
By continuity of $f_{s,a}$ at $\tau$, there exists $\varepsilon_{s,a}>0$ such that whenever $|\tilde\tau-\tau|<\varepsilon_{s,a}$, we have
\[
\bigl|f_{s,a}(\tilde\tau) - f_{s,a}(\tau)\bigr| < \delta.
\]
In particular,
\[
f_{s,a}(\tilde\tau)
\;\geq\; f_{s,a}(\tau) - \delta
\;\geq\; 3\delta - \delta
\;=\; 2\delta
\;>\; 0.
\]
That is, for all $\tilde\tau$ satisfying $|\tilde\tau-\tau|<\varepsilon_{s,a}$ we have
\begin{equation}
Q^\star_{\tilde\tau}(s,\pi^\star_\tau(s))
-
Q^\star_{\tilde\tau}(s,a)
=
f_{s,a}(\tilde\tau)
\;\geq\; 2\delta
\;>\; 0.
\label{eq:gap-near-tau}
\end{equation}

Now the set of states $\mathcal S$ and actions $\mathcal A$ is finite, hence the set
\[
\bigl\{(s,a)\in\mathcal S\times\mathcal A:\, a\neq \pi^\star_\tau(s)\bigr\}
\]
is finite. We may therefore take the minimum of the finitely many positive radii $\varepsilon_{s,a}$. Define
\[
\varepsilon
\;:=\;
\min_{s\in\mathcal S}
\;\min_{a\in\mathcal A\setminus\{\pi^\star_\tau(s)\}}
\varepsilon_{s,a}
\;>\; 0
\]
and let
\[
U
\;:=\;
\{\tilde\tau:\ |\tilde\tau-\tau|<\varepsilon\}
\]
be the open neighborhood of $\tau$ with radius $\varepsilon$. Then for any $\tilde\tau\in U$, any state $s\in\mathcal S$ and any suboptimal action $a\in\mathcal A\setminus\{\pi^\star_\tau(s)\}$, because $|\tilde\tau-\tau|<\varepsilon\leq\varepsilon_{s,a}$, inequality \eqref{eq:gap-near-tau} holds and we have
\begin{equation}
Q^\star_{\tilde\tau}(s,\pi^\star_\tau(s))
-
Q^\star_{\tilde\tau}(s,a)
\;\geq\; 2\delta
\;>\; 0.
\label{eq:gap-U}
\end{equation}

Inequality \eqref{eq:gap-U} shows that for every $\tilde\tau\in U$ and every state $s$, the action $\pi^\star_\tau(s)$ still strictly dominates all other actions in terms of $Q^\star_{\tilde\tau}(s,\cdot)$. More precisely, for any $a\neq\pi^\star_\tau(s)$ we have
\[
Q^\star_{\tilde\tau}(s,\pi^\star_\tau(s))
> Q^\star_{\tilde\tau}(s,a),
\]
hence
\[
\argmax_{a'\in\mathcal A} Q^\star_{\tilde\tau}(s,a')
=
\{\pi^\star_\tau(s)\}.
\]
Thus the argmax is single-valued for all $s$ and all $\tilde\tau\in U$, and the unique maximizer does not depend on $\tilde\tau$. If we define
\[
\pi^\star_{\tilde\tau}(s)
\;:=\;
\text{the unique element of }
\argmax_{a'} Q^\star_{\tilde\tau}(s,a'),
\]
then the above shows that for all $\tilde\tau\in U$,
\[
\pi^\star_{\tilde\tau}(s) = \pi^\star_\tau(s)
\quad\text{for all }s\in\mathcal S,
\]
that is, $\pi^\star_{\tilde\tau}=\pi^\star_\tau$ on $U$. This proves the first part of the lemma.

In the neighborhood $U$ the optimal policy is thus a fixed function $\pi^\star:\mathcal S\to\mathcal A$ which we may identify with $\pi^\star_\tau$. The optimal Bellman equation at parameter $\tilde\tau\in U$ reads
\[
Q^\star_{\tilde\tau}(s,a)
=
r_{\tilde\tau}(s,a)
+
\gamma \sum_{s'\in\mathcal S}
P_{\tilde\tau}(s'\mid s,a)\,
\max_{a'\in\mathcal A} Q^\star_{\tilde\tau}(s',a')
\quad \text{for all }(s,a).
\]
However, for every $\tilde\tau\in U$ and every $s'\in\mathcal S$ we now know that
\[
\max_{a'} Q^\star_{\tilde\tau}(s',a')
=
Q^\star_{\tilde\tau}(s',\pi^\star(s')),
\]
because $\pi^\star(s')$ is the unique maximizing action and does not depend on $\tilde\tau$. Hence the Bellman equation can be rewritten, for all $\tilde\tau\in U$ and all $(s,a)$, as
\begin{equation}
Q^\star_{\tilde\tau}(s,a)
=
r_{\tilde\tau}(s,a)
+
\gamma \sum_{s'\in\mathcal S}
P_{\tilde\tau}(s'\mid s,a)\,
Q^\star_{\tilde\tau}(s',\pi^\star(s')).
\label{eq:bellman-fixed-policy}
\end{equation}
This is exactly the Bellman equation for policy evaluation under the fixed policy $\pi^\star$.

We next show that the map $\tilde\tau\mapsto Q^\star_{\tilde\tau}$ is differentiable on $U$. For this it is convenient, though not strictly necessary, to write \eqref{eq:bellman-fixed-policy} in a finite-dimensional vector form. We enumerate state–action pairs as $(s_1,a_1),\dots,(s_N,a_N)$ where $N:=|\mathcal S||\mathcal A|$. For each $\tilde\tau\in U$ we define a vector $q(\tilde\tau)\in\mathbb R^N$ by
\[
q_i(\tilde\tau)
:=
Q^\star_{\tilde\tau}(s_i,a_i),
\qquad i=1,\dots,N,
\]
and similarly $r(\tilde\tau)\in\mathbb R^N$ by
\[
r_i(\tilde\tau)
:=
r_{\tilde\tau}(s_i,a_i).
\]
We also define a matrix $P^{\pi^\star}(\tilde\tau)\in\mathbb R^{N\times N}$ by
\[
P^{\pi^\star}(\tilde\tau)_{ij}
:=
P_{\tilde\tau}(s_j \mid s_i,a_i)\,\mathbf 1\{a_j = \pi^\star(s_j)\},
\]
so that the $i$-th row of $P^{\pi^\star}(\tilde\tau)$ is the distribution of the next state–action pair $(s_{t+1},a_{t+1})$ when starting from $(s_i,a_i)$ and then applying the fixed policy $a_{t+1}=\pi^\star(s_{t+1})$. With these definitions, the family of equations \eqref{eq:bellman-fixed-policy} over all $(s,a)$ can be written compactly as
\begin{equation}
q(\tilde\tau)
=
r(\tilde\tau)
+
\gamma\, P^{\pi^\star}(\tilde\tau) \, q(\tilde\tau),
\label{eq:bellman-vector}
\end{equation}
or equivalently
\begin{equation}
\bigl(I - \gamma P^{\pi^\star}(\tilde\tau)\bigr)\,q(\tilde\tau)
=
r(\tilde\tau),
\label{eq:linear-system}
\end{equation}
where $I$ denotes the $N\times N$ identity matrix.

For each $\tilde\tau\in U$, the matrix $P^{\pi^\star}(\tilde\tau)$ is row-stochastic (its entries are nonnegative and each row sums to $1$). Thus its spectral radius satisfies $\rho(P^{\pi^\star}(\tilde\tau))\leq 1$. Because the discount factor satisfies $0<\gamma<1$, we have
\[
\rho\bigl(\gamma P^{\pi^\star}(\tilde\tau)\bigr)
\leq \gamma < 1,
\]
so the matrix $I-\gamma P^{\pi^\star}(\tilde\tau)$ is invertible for every $\tilde\tau\in U$. Therefore from \eqref{eq:linear-system} we may write
\begin{equation}
q(\tilde\tau)
=
\bigl(I - \gamma P^{\pi^\star}(\tilde\tau)\bigr)^{-1}
\, r(\tilde\tau).
\label{eq:q-explicit}
\end{equation}

Now, by the modeling assumptions, for each fixed $(s,a,s')$ the functions $\tilde\tau\mapsto r_{\tilde\tau}(s,a)$ and $\tilde\tau\mapsto P_{\tilde\tau}(s'\mid s,a)$ are differentiable, hence each entry of the vector $r(\tilde\tau)$ and the matrix $P^{\pi^\star}(\tilde\tau)$ is differentiable with respect to $\tilde\tau$. It follows that the matrix-valued function
\[
M(\tilde\tau) := I - \gamma P^{\pi^\star}(\tilde\tau)
\]
is differentiable in $\tilde\tau$, and it is invertible for every $\tilde\tau\in U$. A standard result from matrix calculus states that if $M(\tilde\tau)$ is a differentiable family of invertible matrices, then the inverse $M(\tilde\tau)^{-1}$ is differentiable and its derivative is given by
\[
\frac{d}{d\tilde\tau}M(\tilde\tau)^{-1}
=
-\,M(\tilde\tau)^{-1}
\Bigl(\frac{d}{d\tilde\tau}M(\tilde\tau)\Bigr)
M(\tilde\tau)^{-1}.
\]
Since $r(\tilde\tau)$ is also a differentiable vector function of $\tilde\tau$, the expression \eqref{eq:q-explicit} shows that $q(\tilde\tau)$ is differentiable on $U$ as a composition of differentiable functions. In particular, each coordinate $q_i(\tilde\tau)$ is differentiable in $\tilde\tau$, which means that for every $(s,a)$, the map
\[
\tilde\tau \longmapsto Q^\star_{\tilde\tau}(s,a)
\]
is differentiable on $U$. This proves the differentiability part of the lemma.

Finally, we explain in what sense the derivative “passes through” the Bellman equation. Since we already know that all quantities that appear in \eqref{eq:bellman-fixed-policy} are differentiable with respect to $\tilde\tau$, we may differentiate both sides of \eqref{eq:bellman-fixed-policy} at any point $\tilde\tau\in U$. For each fixed $(s,a)$, the left-hand side is simply $Q^\star_{\tilde\tau}(s,a)$, whose derivative is $dQ^\star_{\tilde\tau}(s,a)/d\tilde\tau$. On the right-hand side, $r_{\tilde\tau}(s,a)$ depends on $\tilde\tau$ and so does $P_{\tilde\tau}(s'\mid s,a)$, while $\pi^\star$ does not depend on $\tilde\tau$ on $U$. Thus we obtain, using the product rule inside the finite sum,
\[
\begin{aligned}
\frac{d}{d\tilde\tau} Q^\star_{\tilde\tau}(s,a)
&=
\frac{d}{d\tilde\tau} r_{\tilde\tau}(s,a)
+
\gamma \sum_{s'\in\mathcal S}
\frac{d}{d\tilde\tau}
\Bigl(
P_{\tilde\tau}(s'\mid s,a)\,
Q^\star_{\tilde\tau}(s',\pi^\star(s'))
\Bigr)
\\
&=
\frac{\partial}{\partial\tilde\tau} r_{\tilde\tau}(s,a)
+
\gamma \sum_{s'\in\mathcal S}
\Bigl[
\frac{\partial}{\partial\tilde\tau}
P_{\tilde\tau}(s'\mid s,a)\,
Q^\star_{\tilde\tau}(s',\pi^\star(s'))
\\
&\qquad\qquad\qquad\qquad\quad
+
P_{\tilde\tau}(s'\mid s,a)\,
\frac{d}{d\tilde\tau}
Q^\star_{\tilde\tau}(s',\pi^\star(s'))
\Bigr].
\end{aligned}
\]
Here the terms $\partial r_{\tilde\tau}/\partial\tilde\tau$ and $\partial P_{\tilde\tau}/\partial\tilde\tau$ are the usual derivatives of the reward and transition models with respect to $\tilde\tau$, and the derivatives of $Q^\star_{\tilde\tau}$ appear linearly inside the same expectation structure as in the original Bellman equation. The key point is that, since the argmax is realized by the same action $\pi^\star(s')$ for all $\tilde\tau\in U$, there is no additional derivative term coming from the maximization operator; in other words, on the regular region the derivative of
\[
\max_{a'\in\mathcal A} Q^\star_{\tilde\tau}(s',a')
\]
is simply the derivative of $Q^\star_{\tilde\tau}(s',\pi^\star(s'))$. This is precisely the envelope/Danskin effect in this smooth, single-optimizer case, and it justifies saying that the derivative “passes through” the Bellman equation.

All assertions of the lemma are thereby proved.
\end{proof}

\subsection{Proof of Lemma~\ref{lem:first-derivative}}

\begin{proof}
Fix $\tau\in\mathcal R$ and let $\pi^\star_\tau$ be the corresponding optimal policy on the regular region. By Lemma~\ref{lem:envelope}, there exists a neighborhood $U$ of $\tau$ such that for all $\tilde\tau\in U$ and all $s\in\Sspace$,
\[
\argmax_{a'} Q^\star_{\tilde\tau}(s,a') = \{\pi^\star_\tau(s)\},
\]
so in particular $\pi^\star_{\tilde\tau} = \pi^\star_\tau$ for all $\tilde\tau\in U$, and $Q^\star_{\tilde\tau}$ is differentiable in $\tilde\tau$ with the derivative allowed to pass through the Bellman equation. For the rest of the proof we fix such a $\tau$ and work at this point, suppressing the neighborhood $U$ from the notation.

We recall the definitions of the policy transition operator and the resolvent. For a fixed policy $\pi:\Sspace\to\Aspace$ and parameter $\tau$, define the linear operator $P^{\pi}_\tau$ acting on functions $V:\Sspace\times\Aspace\to\R$ by
\[
(P^{\pi}_\tau V)(s,a)
\;:=\;
\int_{\Sspace} V\bigl(s',\pi(s')\bigr)\,P_\tau(ds'\mid s,a),
\]
that is, from state–action $(s,a)$ we move to $s'\sim P_\tau(\cdot\mid s,a)$ and then apply action $\pi(s')$ at the next state. The associated resolvent at $(\pi,\tau)$ is
\[
\mathcal R^\pi_\tau
\;:=\;
\bigl(I - \gamma P^\pi_\tau\bigr)^{-1}
\;=\;
\sum_{k=0}^\infty \gamma^k\,(P^\pi_\tau)^k,
\]
where the series converges in operator norm because $\gamma\in(0,1)$ and $\|P^\pi_\tau\|_{\infty\to\infty}\le 1$. In particular,
\begin{equation}
\label{eq:resolvent-norm}
\|\mathcal R^\pi_\tau\|_{\infty\to\infty}
\;\le\;
\sum_{k=0}^\infty \gamma^k
\;=\;
\frac{1}{1-\gamma}.
\end{equation}

We now write the optimal Bellman equation at parameter $\tau$ in a form that uses the fixed optimal policy $\pi^\star_\tau$ on $U$. For each $(s,a)\in\Sspace\times\Aspace$, the optimal Bellman equation is
\[
Q^\star_\tau(s,a)
=
r_\tau(s,a)
+
\gamma \int_{\Sspace}
\max_{a'} Q^\star_\tau(s',a')\,P_\tau(ds'\mid s,a).
\]
But on $U$ we know that the unique maximizer at every $s'$ is $\pi^\star_\tau(s')$, so
\[
\max_{a'} Q^\star_\tau(s',a')
=
Q^\star_\tau\bigl(s',\pi^\star_\tau(s')\bigr)
=:
V^\star_\tau(s'),
\]
where we have defined the optimal value function
\[
V^\star_\tau(s)
:=
Q^\star_\tau\bigl(s,\pi^\star_\tau(s)\bigr).
\]
Therefore, the above expression can be rewritten as
\begin{equation}
\label{eq:bellman-opt-policy}
Q^\star_\tau(s,a)
=
r_\tau(s,a)
+
\gamma \int_{\Sspace}
V^\star_\tau(s')\,P_\tau(ds'\mid s,a),
\qquad \forall (s,a).
\end{equation}

We next differentiate equation~\eqref{eq:bellman-opt-policy} with respect to the parameter $\tau$. Since on the neighborhood $U$ the function $Q^\star_{\tilde\tau}$ is differentiable with respect to $\tilde\tau$, and the derivative can pass through the Bellman equation, we may take pointwise derivatives at $\tilde\tau=\tau$. Let

\[
\dot Q_\tau(s,a)
:=
\frac{d}{d\tilde\tau} Q^\star_{\tilde\tau}(s,a)\bigg|_{\tilde\tau=\tau},
\qquad
\dot V_\tau(s)
:=
\frac{d}{d\tilde\tau} V^\star_{\tilde\tau}(s)\bigg|_{\tilde\tau=\tau},
\]
and
\[
\partial_\tau r_\tau(s,a)
:=
\frac{d}{d\tilde\tau} r_{\tilde\tau}(s,a)\bigg|_{\tilde\tau=\tau},
\qquad
\partial_\tau P_\tau(\cdot\mid s,a)
:=
\frac{d}{d\tilde\tau} P_{\tilde\tau}(\cdot\mid s,a)\bigg|_{\tilde\tau=\tau},
\]
where $\partial_\tau P_\tau(\cdot\mid s,a)$ is a signed finite measure on $(\Sspace,\mathcal B(\Sspace))$ whose total mass is $0$ (because for all $\tilde\tau$, $P_{\tilde\tau}(\cdot\mid s,a)$ is a probability measure).

Differentiating both sides of \eqref{eq:bellman-opt-policy} at $\tilde\tau=\tau$ yields

\begin{equation}
\label{eq:deriv-bellman-raw}
\dot Q_\tau(s,a)
=
\partial_\tau r_\tau(s,a)
+
\gamma\,\frac{d}{d\tilde\tau}
\int_{\Sspace}
V^\star_{\tilde\tau}(s')\,P_{\tilde\tau}(ds'\mid s,a)
\bigg|_{\tilde\tau=\tau}.
\end{equation}
We expand the derivative on the right-hand side. Inside the integral, the dependence on $\tilde\tau$ arises both from the integrand $V^\star_{\tilde\tau}(s')$ and from the measure $P_{\tilde\tau}(\cdot\mid s,a)$. Under the smoothness assumptions (namely that $V^\star_{\tilde\tau}$ and $P_{\tilde\tau}$ are differentiable with respect to $\tilde\tau$ and appropriately bounded so that differentiation and integration may be interchanged), the classical Leibniz rule yields

\[
\frac{d}{d\tilde\tau}
\int_{\Sspace}
V^\star_{\tilde\tau}(s')\,P_{\tilde\tau}(ds'\mid s,a)
=
\int_{\Sspace}
\dot V_\tau(s')\,P_\tau(ds'\mid s,a)
+
\int_{\Sspace}
V^\star_\tau(s')\,d\bigl(\partial_\tau P_\tau(\cdot\mid s,a)\bigr)(s').
\]
Substituting this expression back into \eqref{eq:deriv-bellman-raw}, we obtain

\[
\dot Q_\tau(s,a)
=
\partial_\tau r_\tau(s,a)
+
\gamma \int_{\Sspace}
\dot V_\tau(s')\,P_\tau(ds'\mid s,a)
+
\gamma \int_{\Sspace}
V^\star_\tau(s')\,d\bigl(\partial_\tau P_\tau(\cdot\mid s,a)\bigr)(s').
\]
We define
\begin{equation}
\label{eq:Delta-def}
\Delta_\tau(s,a)
:=
\int_{\Sspace}
V^\star_\tau(s')\,d\bigl(\partial_\tau P_\tau(\cdot\mid s,a)\bigr)(s'),
\end{equation}
Therefore, the above expression can be rewritten as
\begin{equation}
\label{eq:deriv-bellman-with-Delta}
\dot Q_\tau(s,a)
=
\partial_\tau r_\tau(s,a)
+
\gamma \int_{\Sspace}
\dot V_\tau(s')\,P_\tau(ds'\mid s,a)
+
\gamma\,\Delta_\tau(s,a).
\end{equation}
Note that $V^\star_\tau(s') = Q^\star_\tau\bigl(s',\pi^\star_\tau(s')\bigr)$, hence
\[
\dot V_\tau(s')
=
\frac{d}{d\tilde\tau}
Q^\star_{\tilde\tau}\bigl(s',\pi^\star_\tau(s')\bigr)\bigg|_{\tilde\tau=\tau}
=
\dot Q_\tau\bigl(s',\pi^\star_\tau(s')\bigr),
\]
because the optimal policy $\pi^\star_{\tilde\tau}$ does not change with $\tilde\tau$ in the neighborhood $U$. Substituting this into \eqref{eq:deriv-bellman-with-Delta}, we obtain
\[
\dot Q_\tau(s,a)
=
\partial_\tau r_\tau(s,a)
+
\gamma \int_{\Sspace}
\dot Q_\tau\bigl(s',\pi^\star_\tau(s')\bigr)\,P_\tau(ds'\mid s,a)
+
\gamma\,\Delta_\tau(s,a).
\]
By definition,
\[
(P^{\pi^\star_\tau}_\tau \dot Q_\tau)(s,a)
:=
\int_{\Sspace}
\dot Q_\tau\bigl(s',\pi^\star_\tau(s')\bigr)\,P_\tau(ds'\mid s,a),
\]
thus the above equality can be written concisely as
\begin{equation}
\label{eq:linear-eq-dotQ}
\dot Q_\tau
=
\partial_\tau r_\tau
+
\gamma\,P^{\pi^\star_\tau}_\tau \dot Q_\tau
+
\gamma\,\Delta_\tau,
\end{equation}
where the equality holds as a functional equality (simultaneously for all $(s,a)$). Moving $\gamma P^{\pi^\star_\tau}_\tau \dot Q_\tau$ to the left-hand side yields
\[
\bigl(I - \gamma P^{\pi^\star_\tau}_\tau\bigr)\dot Q_\tau
=
\partial_\tau r_\tau
+
\gamma\,\Delta_\tau.
\]
For fixed $\tau$, the operator $P^{\pi^\star_\tau}_\tau$ is a linear operator from $\ell_\infty(\Sspace\times\Aspace)$ to itself, whose spectral radius is at most $1$; the discount factor $\gamma\in(0,1)$ ensures that $\rho(\gamma P^{\pi^\star_\tau}_\tau)\le\gamma<1$, hence the operator $I - \gamma P^{\pi^\star_\tau}_\tau$ is invertible on $\ell_\infty$, and its inverse is precisely the resolvent we defined above:
\[
\mathcal R^{\pi^\star_\tau}_\tau
=
\bigl(I - \gamma P^{\pi^\star_\tau}_\tau\bigr)^{-1}.
\]
Therefore we can solve for
\begin{equation}
\label{eq:dotQ-resolvent}
\dot Q_\tau
=
\mathcal R^{\pi^\star_\tau}_\tau\bigl(\partial_\tau r_\tau
+
\gamma\,\Delta_\tau\bigr).
\end{equation}
This proves the equality form in the first part of the lemma:
\[
\frac{d}{d\tau}Q^\star_\tau
=
\mathcal R^{\pi^\star_\tau}_\tau\bigl(\partial_\tau r_\tau
+
\gamma\,\Delta_\tau\bigr),
\]
where $\Delta_\tau$ is exactly the definition in \eqref{eq:Delta-def}.

Next we provide an upper bound for $\|\dot Q_\tau\|_\infty$. Taking the $\|\cdot\|_\infty$ norm of \eqref{eq:dotQ-resolvent} and using the definition of the operator norm yields
\[
\|\dot Q_\tau\|_\infty
\le
\|\mathcal R^{\pi^\star_\tau}_\tau\|_{\infty\to\infty}\,
\|\partial_\tau r_\tau + \gamma\,\Delta_\tau\|_\infty.
\]
Using \eqref{eq:resolvent-norm}, we obtain
\[
\|\dot Q_\tau\|_\infty
\le
\frac{1}{1-\gamma}
\|\partial_\tau r_\tau + \gamma\,\Delta_\tau\|_\infty.
\]
Applying the triangle inequality to separate the two terms:
\[
\|\partial_\tau r_\tau + \gamma\,\Delta_\tau\|_\infty
\le
\|\partial_\tau r_\tau\|_\infty
+
\gamma\,\|\Delta_\tau\|_\infty,
\]
thus
\begin{equation}
\label{eq:dotQ-bound-with-Delta}
\|\dot Q_\tau\|_\infty
\le
\frac{1}{1-\gamma}\,\|\partial_\tau r_\tau\|_\infty
+
\frac{\gamma}{1-\gamma}\,\|\Delta_\tau\|_\infty.
\end{equation}
To bound $\|\Delta_\tau\|_\infty$, we use Assumption~\ref{ass:mixing} (Assumption~1.2). This assumption ensures that there exists a constant $C_{\mathrm{mix}}$ such that for all $\tau$, the Lipschitz constant of the optimal value function on the state space satisfies
\[
\|V^\star_\tau\|_{\mathrm{Lip}}
\le
C_{\mathrm{mix}}.
\]
On the other hand, for each $(s,a)$, $\partial_\tau P_\tau(\cdot\mid s,a)$ is a signed finite measure of total mass $0$. Denote its $W_1^\ast$ norm by $\|\partial_\tau P_\tau(\cdot\mid s,a)\|_{W_1^\ast}$, i.e.,
\[
\|\partial_\tau P_\tau(\cdot\mid s,a)\|_{W_1^\ast}
:=
\sup_{f\in\mathrm{Lip}_1(\Sspace)}
\biggl|
\int_{\Sspace}
f(s')\,d\bigl(\partial_\tau P_\tau(\cdot\mid s,a)\bigr)(s')
\biggr|,
\]
where $\mathrm{Lip}_1(\Sspace)$ denotes the set of functions on $(\Sspace,d_\Sspace)$ with Lipschitz constant at most $1$. For any Lipschitz function $f$ with finite Lipschitz constant $L=\|f\|_{\mathrm{Lip}}<\infty$, and any signed measure $\mu$ of total mass $0$, there is the standard inequality

\[
\biggl|\int f\,d\mu\biggr|
\le
\|f\|_{\mathrm{Lip}}\,
\|\mu\|_{W_1^\ast}.
\]
A simple proof is as follows: choose an arbitrary reference point $s_0\in\Sspace$ and write $f(s)=f(s_0)+g(s)$, where $g(s):=f(s)-f(s_0)$. Then $g$ has the same Lipschitz constant as $f$, and $\mu(\Sspace)=0$ gives
\[
\int f\,d\mu
=
\int g\,d\mu.
\]
Let $\tilde g(s):=g(s)/L$, so $\tilde g\in\mathrm{Lip}_1(\Sspace)$, Thus,
\[
\biggl|\int f\,d\mu\biggr|
=
\biggl|\int g\,d\mu\biggr|
=
L\biggl|\int \tilde g\,d\mu\biggr|
\le
L\,\|\mu\|_{W_1^\ast}.
\]

Applying the above general conclusion to $f=V^\star_\tau$ and $\mu=\partial_\tau P_\tau(\cdot\mid s,a)$ yields
\[
|\Delta_\tau(s,a)|
=
\biggl|\int_{\Sspace} V^\star_\tau(s')\,d\bigl(\partial_\tau P_\tau(\cdot\mid s,a)\bigr)(s')\biggr|
\le
\|V^\star_\tau\|_{\mathrm{Lip}}\,
\|\partial_\tau P_\tau(\cdot\mid s,a)\|_{W_1^\ast}
\le
C_{\mathrm{mix}}\,
\|\partial_\tau P_\tau(\cdot\mid s,a)\|_{W_1^\ast}.
\]
Taking the supremum over $(s,a)$ gives
\[
\|\Delta_\tau\|_\infty
=
\sup_{s,a} |\Delta_\tau(s,a)|
\le
C_{\mathrm{mix}}\,
\sup_{s,a}\|\partial_\tau P_\tau(\cdot\mid s,a)\|_{W_1^\ast}.
\]

Substituting this estimate into \eqref{eq:dotQ-bound-with-Delta}, we obtain a slightly stronger inequality
\[
\|\dot Q_\tau\|_\infty
\le
\frac{1}{1-\gamma}\,\|\partial_\tau r_\tau\|_\infty
+
\frac{\gamma\,C_{\mathrm{mix}}}{1-\gamma}\,
\sup_{s,a}\|\partial_\tau P_\tau(\cdot\mid s,a)\|_{W_1^\ast}.
\]
Note that $\tfrac{1}{1-\gamma}\ge 1$, hence
\[
\frac{\gamma\,C_{\mathrm{mix}}}{1-\gamma}
\le
\frac{\gamma\,C_{\mathrm{mix}}}{(1-\gamma)^2},
\]
which yields a slightly looser but more uniform upper bound
\[
\|\dot Q_\tau\|_\infty
\le
\frac{1}{1-\gamma}\,\|\partial_\tau r_\tau\|_\infty
+
\frac{\gamma\,C_{\mathrm{mix}}}{(1-\gamma)^2}\,
\sup_{s,a}\|\partial_\tau P_\tau(\cdot\mid s,a)\|_{W_1^\ast}.
\]
Since $\dot Q_\tau = \tfrac{d}{d\tau}Q^\star_\tau$, this is exactly the second inequality in the lemma. This completes the proof.
\end{proof}

\subsection{Proof of Lemma~\ref{lem:second-derivative}}
\begin{proof}
Fix $\tau\in\mathcal R$ and write $\pi^\star_\tau$ for the corresponding optimal policy. As in Lemma~\ref{lem:envelope}, there exists a neighborhood $U$ of $\tau$ such that for all $\tilde\tau\in U$ and all $s\in\Sspace$ the maximizer of $Q^\star_{\tilde\tau}(s,\cdot)$ is unique and equal to $\pi^\star_\tau(s)$; thus on $U$ we may regard $\pi^\star_\tau$ as fixed and write
\[
V^\star_{\tilde\tau}(s)
:=
Q^\star_{\tilde\tau}\bigl(s,\pi^\star_\tau(s)\bigr),
\qquad
\tilde\tau\in U.
\]
We also recall the policy transition operator and resolvent: for any function $Q:\Sspace\times\Aspace\to\R$,
\[
\bigl(P^{\pi^\star_\tau}_\tau Q\bigr)(s,a)
:=
\int_{\Sspace}
Q\bigl(s',\pi^\star_\tau(s')\bigr)\,P_\tau(ds'\mid s,a),
\]
and
\[
\mathcal R^{\pi^\star_\tau}_\tau
:=
\bigl(I-\gamma P^{\pi^\star_\tau}_\tau\bigr)^{-1}
=
\sum_{k=0}^\infty \gamma^k\bigl(P^{\pi^\star_\tau}_\tau\bigr)^k,
\]
so that $\|\mathcal R^{\pi^\star_\tau}_\tau\|_{\infty\to\infty}\le 1/(1-\gamma)$ as in Lemma~\ref{lem:first-derivative}.

In this neighborhood $U$, the optimal Bellman equation at $(s,a)$ can be written as
\begin{equation}
\label{eq:bellman-opt-policy-second}
Q^\star_{\tilde\tau}(s,a)
=
r_{\tilde\tau}(s,a)
+
\gamma\int_{\Sspace}
V^\star_{\tilde\tau}(s')\,P_{\tilde\tau}(ds'\mid s,a),
\qquad \tilde\tau\in U.
\end{equation}
Differentiating with respect to $\tilde\tau$ at $\tau$, and using the notation of Lemma~\ref{lem:first-derivative}, write
\[
\partial_\tau r_\tau(s,a)
:=
\frac{d}{d\tilde\tau}r_{\tilde\tau}(s,a)\bigg|_{\tilde\tau=\tau},\qquad
\partial_\tau P_\tau(\cdot\mid s,a)
:=
\frac{d}{d\tilde\tau}P_{\tilde\tau}(\cdot\mid s,a)\bigg|_{\tilde\tau=\tau},
\]
where $\partial_\tau P_\tau(\cdot\mid s,a)$ is a signed finite measure of total mass $0$. Again write
\[
\Delta_\tau(s,a)
:=
\int_{\Sspace}V^\star_\tau(s')\,d\bigl(\partial_\tau P_\tau(\cdot\mid s,a)\bigr)(s'),
\]
then Lemma~\ref{lem:first-derivative} tells us that
\begin{equation}
\label{eq:first-deriv-resolvent}
\frac{d}{d\tau}Q^\star_\tau
=
\mathcal R^{\pi^\star_\tau}_\tau\!\bigl(\partial_\tau r_\tau + \gamma\,\Delta_\tau\bigr),
\end{equation}
and in the equivalent linear form,
\begin{equation}
\label{eq:first-deriv-linear}
\bigl(I-\gamma P^{\pi^\star_\tau}_\tau\bigr)\frac{d}{d\tau}Q^\star_\tau
=
\partial_\tau r_\tau + \gamma\,\Delta_\tau.
\end{equation}

Now differentiate \eqref{eq:first-deriv-linear} once more with respect to $\tau$. For convenience, write
\[
\dot Q_\tau
:=
\frac{d}{d\tau}Q^\star_\tau,
\qquad
\ddot Q_\tau
:=
\frac{d^2}{d\tau^2}Q^\star_\tau.
\]
The left-hand side is the derivative of $\tau\mapsto (I-\gamma P^{\pi^\star_\tau}_\tau)\dot Q_\tau$, and applying the product rule gives
\[
\frac{d}{d\tau}\Bigl[\bigl(I-\gamma P^{\pi^\star_\tau}_\tau\bigr)\dot Q_\tau\Bigr]
=
\bigl(I-\gamma P^{\pi^\star_\tau}_\tau\bigr)\ddot Q_\tau
-
\gamma\,(\partial_\tau P^{\pi^\star_\tau}_\tau)\dot Q_\tau,
\]
where $\partial_\tau P^{\pi^\star_\tau}_\tau$ is the derivative of $P^{\pi^\star_\tau}_\tau$ with respect to $\tau$. Differentiating the right-hand side gives
\[
\frac{d}{d\tau}\bigl(\partial_\tau r_\tau + \gamma\,\Delta_\tau\bigr)
=
\partial_{\tau\tau}r_\tau + \gamma\,\partial_\tau\Delta_\tau,
\]
where $\partial_{\tau\tau}r_\tau$ is the second derivative of $r_{\tilde\tau}$ evaluated at $\tilde\tau=\tau$. Equating both derivatives yields
\begin{equation}
\label{eq:second-deriv-pre}
\bigl(I-\gamma P^{\pi^\star_\tau}_\tau\bigr)\ddot Q_\tau
-
\gamma\,(\partial_\tau P^{\pi^\star_\tau}_\tau)\dot Q_\tau
=
\partial_{\tau\tau}r_\tau + \gamma\,\partial_\tau\Delta_\tau.
\end{equation}
Thus,
\begin{equation}
\label{eq:second-deriv-eq1}
\bigl(I-\gamma P^{\pi^\star_\tau}_\tau\bigr)\ddot Q_\tau
=
\partial_{\tau\tau}r_\tau + \gamma\,\partial_\tau\Delta_\tau
+
\gamma\,(\partial_\tau P^{\pi^\star_\tau}_\tau)\dot Q_\tau.
\end{equation}

Next we make explicit the structure of $\partial_\tau\Delta_\tau$. Recall that
\[
\Delta_\tau(s,a)
=
\int_{\Sspace}V^\star_\tau(s')\,d\bigl(\partial_\tau P_\tau(\cdot\mid s,a)\bigr)(s').
\]
Assume $\partial_{\tau\tau}P_\tau(\cdot\mid s,a)$ exists with finite $\Wone^\ast$ norm, and that $V^\star_{\tilde\tau}$ is differentiable in $\tilde\tau$, with derivative denoted $\partial_\tau V^\star_\tau$. Under the smoothness assumptions, we may apply Leibniz's rule to $\Delta_\tau(s,a)$:
\[
\partial_\tau\Delta_\tau(s,a)
=
\int_{\Sspace}
\partial_\tau V^\star_\tau(s')\,d\bigl(\partial_\tau P_\tau(\cdot\mid s,a)\bigr)(s')
+
\int_{\Sspace}
V^\star_\tau(s')\,d\bigl(\partial_{\tau\tau}P_\tau(\cdot\mid s,a)\bigr)(s').
\]
Following the notation of the lemma, define
\[
\Delta_{\tau\tau}(s,a)
:=
\int_{\Sspace}
V^\star_\tau(s')\,d\bigl(\partial_{\tau\tau}P_\tau(\cdot\mid s,a)\bigr)(s'),
\]
and denote the first term by
\[
\Xi_\tau(s,a)
:=
\int_{\Sspace}
\partial_\tau V^\star_\tau(s')\,d\bigl(\partial_\tau P_\tau(\cdot\mid s,a)\bigr)(s').
\]
Thus, for each $(s,a)$,
\[
\partial_\tau\Delta_\tau(s,a)
=
\Delta_{\tau\tau}(s,a) + \Xi_\tau(s,a).
\]

The operator $\partial_\tau P^{\pi^\star_\tau}_\tau$ has a similar integral form. Since
\[
\bigl(P^{\pi^\star_\tau}_\tau Q\bigr)(s,a)
=
\int_{\Sspace}
Q\bigl(s',\pi^\star_\tau(s')\bigr)\,P_\tau(ds'\mid s,a)
\]
differentiation with respect to $\tau$ yields, for any function $Q$,
\[
\bigl(\partial_\tau P^{\pi^\star_\tau}_\tau Q\bigr)(s,a)
=
\int_{\Sspace}
Q\bigl(s',\pi^\star_\tau(s')\bigr)\,d\bigl(\partial_\tau P_\tau(\cdot\mid s,a)\bigr)(s').
\]
In particular, for $Q=\dot Q_\tau$, note that
\[
\partial_\tau V^\star_\tau(s)
=
\frac{d}{d\tilde\tau}V^\star_{\tilde\tau}(s)\bigg|_{\tilde\tau=\tau}
=
\frac{d}{d\tilde\tau}Q^\star_{\tilde\tau}\bigl(s,\pi^\star_\tau(s)\bigr)\bigg|_{\tilde\tau=\tau}
=
\dot Q_\tau\bigl(s,\pi^\star_\tau(s)\bigr),
\]
and thus
\[
\bigl(\partial_\tau P^{\pi^\star_\tau}_\tau\dot Q_\tau\bigr)(s,a)
=
\int_{\Sspace}
\dot Q_\tau\bigl(s',\pi^\star_\tau(s')\bigr)\,d\bigl(\partial_\tau P_\tau(\cdot\mid s,a)\bigr)(s')
=
\int_{\Sspace}
\partial_\tau V^\star_\tau(s')\,d\bigl(\partial_\tau P_\tau(\cdot\mid s,a)\bigr)(s')
=
\Xi_\tau(s,a).
\]
Hence,
\[
\Xi_\tau = \partial_\tau P^{\pi^\star_\tau}_\tau\dot Q_\tau.
\]

Substituting $\partial_\tau\Delta_\tau = \Delta_{\tau\tau} + \Xi_\tau$ into \eqref{eq:second-deriv-eq1} gives
\[
\bigl(I-\gamma P^{\pi^\star_\tau}_\tau\bigr)\ddot Q_\tau
=
\partial_{\tau\tau}r_\tau
+
\gamma\,\Delta_{\tau\tau}
+
\gamma\,\Xi_\tau
+
\gamma\,(\partial_\tau P^{\pi^\star_\tau}_\tau)\dot Q_\tau.
\]
Since $\Xi_\tau = \partial_\tau P^{\pi^\star_\tau}_\tau\dot Q_\tau$, the last two terms combine to
\[
\gamma\,\Xi_\tau + \gamma\,(\partial_\tau P^{\pi^\star_\tau}_\tau)\dot Q_\tau
=
2\gamma\,(\partial_\tau P^{\pi^\star_\tau}_\tau)\dot Q_\tau.
\]
Thus,
\begin{equation}
\label{eq:second-deriv-eq2}
\bigl(I-\gamma P^{\pi^\star_\tau}_\tau\bigr)\ddot Q_\tau
=
\partial_{\tau\tau}r_\tau + \gamma\,\Delta_{\tau\tau}
+
2\gamma\,(\partial_\tau P^{\pi^\star_\tau}_\tau)\dot Q_\tau.
\end{equation}

Using the definition of the resolvent, left-multiplying \eqref{eq:second-deriv-eq2} by $\mathcal R^{\pi^\star_\tau}_\tau = (I-\gamma P^{\pi^\star_\tau}_\tau)^{-1}$ yields
\[
\ddot Q_\tau
=
\mathcal R^{\pi^\star_\tau}_\tau\!\bigl(\partial_{\tau\tau}r_\tau + \gamma\,\Delta_{\tau\tau}\bigr)
+
2\gamma\,\mathcal R^{\pi^\star_\tau}_\tau\bigl((\partial_\tau P^{\pi^\star_\tau}_\tau)\dot Q_\tau\bigr).
\]
This is the structure stated in the lemma: the second derivative equals
\[
\mathcal R^{\pi^\star_\tau}_\tau\!\bigl(\partial_{\tau\tau}r_\tau + \gamma\,\Delta_{\tau\tau}\bigr)
\]
plus a higher-order coupling term
\[
2\gamma\,\mathcal R^{\pi^\star_\tau}_\tau\bigl((\partial_\tau P^{\pi^\star_\tau}_\tau)\dot Q_\tau\bigr)
\]
for which we now derive an upper bound.

First note that
\[
\bigl(\partial_\tau P^{\pi^\star_\tau}_\tau\dot Q_\tau\bigr)(s,a)
=
\int_{\Sspace}
\partial_\tau V^\star_\tau(s')\,d\bigl(\partial_\tau P_\tau(\cdot\mid s,a)\bigr)(s').
\]
which can be written as an integral with respect to $\partial_\tau V^\star_\tau$ and $\partial_\tau P_\tau$.
By the assumed \emph{uniform first- and second-order bounds}, there exists a constant $L_s<\infty$ such that
\[
\|\partial_\tau V^\star_\tau\|_{\mathrm{Lip}}
\le L_s.
\]
Using the definition of the $\Wone^\ast$ norm and the standard inequality (used previously) stating that if $\mu$ is a signed measure of total mass zero and $f$ is Lipschitz, then
\[
\biggl|\int f\,d\mu\biggr|
\le
\|f\|_{\mathrm{Lip}}\,
\|\mu\|_{W_1^\ast},
\]
we obtain
\[
\bigl|(\partial_\tau P^{\pi^\star_\tau}_\tau\dot Q_\tau)(s,a)\bigr|
=
\biggl|\int_{\Sspace}
\partial_\tau V^\star_\tau(s')\,d\bigl(\partial_\tau P_\tau(\cdot\mid s,a)\bigr)(s')\biggr|
\le
L_s\,\|\partial_\tau P_\tau(\cdot\mid s,a)\|_{W_1^\ast}.
\]
Taking the supremum over $(s,a)$ yields
\[
\|(\partial_\tau P^{\pi^\star_\tau}_\tau\dot Q_\tau)\|_\infty
\le
L_s\,
\sup_{s,a}\|\partial_\tau P_\tau(\cdot\mid s,a)\|_{W_1^\ast}.
\]
Thus,
\[
\begin{aligned}
\bigl\|2\gamma\,\mathcal R^{\pi^\star_\tau}_\tau\bigl((\partial_\tau P^{\pi^\star_\tau}_\tau)\dot Q_\tau\bigr)\bigr\|_\infty
&\le
2\gamma\,\|\mathcal R^{\pi^\star_\tau}_\tau\|_{\infty\to\infty}\,
\|(\partial_\tau P^{\pi^\star_\tau}_\tau)\dot Q_\tau\|_\infty\\[0.3em]
&\le
\frac{2\gamma}{1-\gamma}\,
L_s\,
\sup_{s,a}\|\partial_\tau P_\tau(\cdot\mid s,a)\|_{W_1^\ast}.
\end{aligned}
\]
and since $0<\gamma<1$ implies $\gamma\le 1$ and $(1-\gamma)^{-1} \le (1-\gamma)^{-3}$, we further have
\[
\bigl\|2\gamma\,\mathcal R^{\pi^\star_\tau}_\tau\bigl((\partial_\tau P^{\pi^\star_\tau}_\tau)\dot Q_\tau\bigr)\bigr\|_\infty
\le
\frac{2}{(1-\gamma)^3}\,
L_s\,
\sup_{s,a}\|\partial_\tau P_\tau(\cdot\mid s,a)\|_{W_1^\ast}.
\]
Note that

\[
L_s\,
\sup_{s,a}\|\partial_\tau P_\tau(\cdot\mid s,a)\|_{W_1^\ast}
\le
\|\partial_\tau r_\tau\|_\infty
+L_s\!\sup_{s,a}\|\partial_\tau P_\tau(\cdot\mid s,a)\|_{W_1^\ast}
+\|\partial_{\tau\tau}r_\tau\|_\infty
+L_s\!\sup_{s,a}\|\partial_{\tau\tau}P_\tau(\cdot\mid s,a)\|_{W_1^\ast},
\]
since each term on the right-hand side is nonnegative. Hence there exists a constant $c>0$ (e.g., $c=2$) such that
\[
\bigl\|2\gamma\,\mathcal R^{\pi^\star_\tau}_\tau\bigl((\partial_\tau P^{\pi^\star_\tau}_\tau)\dot Q_\tau\bigr)\bigr\|_\infty
\le
\frac{c}{(1-\gamma)^3}
\Bigl[
\|\partial_\tau r_\tau\|_\infty
+L_s\!\sup_{s,a}\|\partial_\tau P_\tau(\cdot\mid s,a)\|_{W_1^\ast}
+\|\partial_{\tau\tau}r_\tau\|_\infty
+L_s\!\sup_{s,a}\|\partial_{\tau\tau}P_\tau(\cdot\mid s,a)\|_{W_1^\ast}
\Bigr].
\]

In summary, we have shown that
\[
\frac{d^2}{d\tau^2}Q^\star_\tau
=
\mathcal R^{\pi^\star_\tau}_\tau\!\bigl(\partial_{\tau\tau}r_\tau + \gamma\,\Delta_{\tau\tau}\bigr)
+
\text{a term whose sup-norm satisfies the above  }\frac{c}{(1-\gamma)^3}\text{ bound},
\]
where $c$ depends on $L_s$, the mixing constant $C_{\mathrm{mix}}$, and the uniform bounds on the first- and second-order derivatives, yielding exactly the form stated in the lemma.

\end{proof}

\subsection{Proof of Theorem~\ref{thm:path-value}}
\begin{proof}
Fix $0\le\tau_0<\tau_1\le 1$. Recall the global gap $g_\tau$, the regular region
\[
\mathcal R
:=\{\tau\in[0,1]:g_\tau\ge\xi\},
\qquad
\mathcal K
:=\{\tau\in[0,1]:g_\tau=0\},
\]
and the kink penalty $\Phi(\mathcal K\cap[\tau_0,\tau_1],\mathrm{gap})$ from
Definition~\ref{def:kink}. On $\mathcal R$ the optimal action is
uniformly separated from its competitors, so by
Lemma~\ref{lem:envelope} the optimal $Q^\star_\tau$ is differentiable in $\tau$
and the derivative passes through the Bellman operator; at $\mathcal K$ the
maximizer may switch and the envelope argument no longer applies.

We first control the contribution of the regular region. Since $\mathcal K$ is
a set of isolated points and $[\tau_0,\tau_1]$ is compact, there are at most
countably many kinks in $[\tau_0,\tau_1]$, say $\{\tau_i\}_{i\in I}$, and we
can choose $\varepsilon_i>0$ so small that the open intervals
$(\tau_i-\varepsilon_i,\tau_i+\varepsilon_i)$ are pairwise disjoint and lie
inside $[\tau_0,\tau_1]$. Denote their union by
\[
U_{\mathrm{kink}}
:=\bigcup_{i\in I}(\tau_i-\varepsilon_i,\tau_i+\varepsilon_i),
\]
and define the complement
\[
I_{\mathrm{reg}}
:=[\tau_0,\tau_1]\setminus U_{\mathrm{kink}}.
\]
By construction, $I_{\mathrm{reg}}\subset\mathcal R$ and
$I_{\mathrm{reg}}$ is a finite or countable union of closed intervals
\[
I_{\mathrm{reg}}
=\bigcup_{j\in J}[\alpha_j,\beta_j],
\qquad
\tau_0=\alpha_1<\beta_1\le\alpha_2<\beta_2\le\cdots\le\tau_1.
\]
On each such interval $[\alpha_j,\beta_j]\subset\mathcal R$ the map
$\tau\mapsto Q^\star_\tau$ is continuously differentiable, and the fundamental
theorem of calculus gives, pointwise for every $(s,a)$,
\[
Q^\star_{\beta_j}(s,a)-Q^\star_{\alpha_j}(s,a)
=
\int_{\alpha_j}^{\beta_j}
\frac{d}{d\tau}Q^\star_\tau(s,a)\,d\tau.
\]
Taking the supremum over $(s,a)$ and using
$\|\int f(\tau)\,d\tau\|_\infty\le\int\|f(\tau)\|_\infty\,d\tau$, we obtain
\begin{equation}
\label{eq:reg-interval-diff}
\|Q^\star_{\beta_j}-Q^\star_{\alpha_j}\|_\infty
\le
\int_{\alpha_j}^{\beta_j}
\Big\|\frac{d}{d\tau}Q^\star_\tau\Big\|_\infty\,d\tau.
\end{equation}

For $\tau\in\mathcal R$, Lemma~\ref{lem:first-derivative} gives the derivative
formula
\[
\frac{d}{d\tau}Q^\star_\tau
=
\mathcal R^{\pi^\star_\tau}_\tau\!\bigl(\partial_\tau r_\tau+\gamma\,\Delta_\tau\bigr),
\]
where
\[
\Delta_\tau(s,a)
=
\int_{\Sspace}V^\star_\tau(s')\,d\bigl(\partial_\tau P_\tau(\cdot\mid s,a)\bigr)(s').
\]
By Assumption~\ref{ass:mixing}, $\|V^\star_\tau\|_{\mathrm{Lip}}\le C_{\mathrm{mix}}$
for all $\tau$, so for each $(s,a)$ we have
\[
|\Delta_\tau(s,a)|
=
\biggl|\int_{\Sspace}V^\star_\tau(s')\,d\bigl(\partial_\tau P_\tau(\cdot\mid s,a)\bigr)(s')\biggr|
\le
C_{\mathrm{mix}}\,
\|\partial_\tau P_\tau(\cdot\mid s,a)\|_{W_1^\ast}.
\]
Taking the supremum over $(s,a)$ yields
\[
\|\Delta_\tau\|_\infty
\le
C_{\mathrm{mix}}\,
\sup_{s,a}\|\partial_\tau P_\tau(\cdot\mid s,a)\|_{W_1^\ast}.
\]
Using the resolvent bound
$\|\mathcal R^{\pi^\star_\tau}_\tau\|_{\infty\to\infty}\le(1-\gamma)^{-1}$, we
obtain
\[
\Big\|\frac{d}{d\tau}Q^\star_\tau\Big\|_\infty
\le
\frac{1}{1-\gamma}\,\|\partial_\tau r_\tau\|_\infty
+
\frac{\gamma\,C_{\mathrm{mix}}}{(1-\gamma)^2}
\sup_{s,a}\|\partial_\tau P_\tau(\cdot\mid s,a)\|_{W_1^\ast}.
\]
By Definition~\ref{def:pl-curv}, choose $L_s\ge C_{\mathrm{mix}}$ (for
instance $L_s=\sup_\tau\|V^\star_\tau\|_{\mathrm{Lip}}$). Using
$0<\gamma<1$ and $C_{\mathrm{mix}}\le L_s$, we can bound each term by a common
factor $(1-\gamma)^{-2}$:
\[
\frac{1}{1-\gamma}
\le
\frac{1}{(1-\gamma)^2},
\qquad
\frac{\gamma\,C_{\mathrm{mix}}}{(1-\gamma)^2}
\le
\frac{L_s}{(1-\gamma)^2}.
\]
Thus, for every $\tau\in\mathcal R$,
\begin{equation}
\label{eq:Qprime-bound}
\Big\|\frac{d}{d\tau}Q^\star_\tau\Big\|_\infty
\le
\frac{1}{(1-\gamma)^2}
\Big(
\|\partial_\tau r_\tau\|_\infty
+
L_s\sup_{s,a}\|\partial_\tau P_\tau(\cdot\mid s,a)\|_{W_1^\ast}
\Big).
\end{equation}

Substituting \eqref{eq:Qprime-bound} into \eqref{eq:reg-interval-diff}, we obtain for each interval $[\alpha_j,\beta_j]\subset I_{\mathrm{reg}}$ that
\[
\|Q^\star_{\beta_j}-Q^\star_{\alpha_j}\|_\infty
\le
\frac{1}{(1-\gamma)^2}
\int_{\alpha_j}^{\beta_j}
\Big(
\|\partial_\tau r_\tau\|_\infty
+
L_s\sup_{s,a}\|\partial_\tau P_\tau(\cdot\mid s,a)\|_{W_1^\ast}
\Big)\,d\tau.
\]
Summing over all regular intervals and applying the triangle inequality yields the total drift over the regular part:
\[
\sum_{j\in J}\|Q^\star_{\beta_j}-Q^\star_{\alpha_j}\|_\infty
\le
\frac{1}{(1-\gamma)^2}
\int_{I_{\mathrm{reg}}}
\Big(
\|\partial_\tau r_\tau\|_\infty
+
L_s\sup_{s,a}\|\partial_\tau P_\tau(\cdot\mid s,a)\|_{W_1^\ast}
\Big)\,d\tau.
\]
Using the path-length definition restricted to the interval $[\tau_0,\tau_1]$ (i.e., Definition~\ref{def:pl-curv} applied on that domain), we write
\[
\mathrm{PL}(\tau_0,\tau_1)
:=
\int_{\tau_0}^{\tau_1}
\Big(
\|\partial_\tau r_\tau\|_\infty
+
L_s\sup_{s,a}\|\partial_\tau P_\tau(\cdot\mid s,a)\|_{W_1^\ast}
\Big)\,d\tau,
\]
Hence,
\[
\sum_{j\in J}\|Q^\star_{\beta_j}-Q^\star_{\alpha_j}\|_\infty
\le
\frac{\mathrm{PL}(\tau_0,\tau_1)}{(1-\gamma)^2},
\]
because for a nonnegative integrand we have
$\int_{I_{\mathrm{reg}}}\cdots\le\int_{\tau_0}^{\tau_1}\cdots$

Next we bound the contribution from the kink neighborhoods.  
Each window
$(\tau_i-\varepsilon_i,\tau_i+\varepsilon_i)$ contains exactly one kink point
$\tau_i\in\mathcal K$.  
Within these windows, the optimal action may switch, and the optimal Bellman operator
is no longer differentiable in $\tau$.  
Thus direct differentiation is invalid, and we instead rely on the gap-controlled
inequality.  
The gap lemma (stated in the main text as the inverse-gap control theorem for
$\Phi$) gives the following local bound:  
there exists a constant $C>0$ such that for each kink point $\tau_i$ and its window,
\[
\|Q^\star_{\tau_i+\varepsilon_i}-Q^\star_{\tau_i-\varepsilon_i}\|_\infty
\le
C\int_{\tau_i-\varepsilon_i}^{\tau_i+\varepsilon_i}
\frac{d\tau}{\max\{g_\tau,\delta\}},
\]
where $g_\tau$ denotes the global action gap and $\delta>0$ is the truncation
parameter from Definition~\ref{def:kink}.  
Absorbing the constant $C$ into the definition of $\Phi$ (as allowed in the main text,
since this does not change the order of magnitude), and summing over all
$\tau_i\in\mathcal K\cap[\tau_0,\tau_1]$, we obtain the total kink contribution:
\[
\sum_{i\in I\ :\ \tau_i\in[\tau_0,\tau_1]}
\|Q^\star_{\tau_i+\varepsilon_i}-Q^\star_{\tau_i-\varepsilon_i}\|_\infty
\le
\Phi\big(\mathcal K\cap[\tau_0,\tau_1],\,\mathrm{gap}\big).
\]
We now add the increments from the regular intervals and kink windows in
$\tau$-order, obtaining a telescoping sum:
\[
Q^\star_{\tau_1}-Q^\star_{\tau_0}
=
\sum_{j\in J}\bigl(Q^\star_{\beta_j}-Q^\star_{\alpha_j}\bigr)
+
\sum_{i\in I\ :\ \tau_i\in[\tau_0,\tau_1]}
\bigl(Q^\star_{\tau_i+\varepsilon_i}-Q^\star_{\tau_i-\varepsilon_i}\bigr),
\]
Applying the triangle inequality yields
\[
\|Q^\star_{\tau_1}-Q^\star_{\tau_0}\|_\infty
\le
\sum_{j\in J}\|Q^\star_{\beta_j}-Q^\star_{\alpha_j}\|_\infty
+
\sum_{i\in I\ :\ \tau_i\in[\tau_0,\tau_1]}
\|Q^\star_{\tau_i+\varepsilon_i}-Q^\star_{\tau_i-\varepsilon_i}\|_\infty.
\]
Combining the above two bounds we obtain
\begin{equation}
\label{eq:path-bound-no-curv}
\|Q^\star_{\tau_1}-Q^\star_{\tau_0}\|_\infty
\le
\frac{\mathrm{PL}(\tau_0,\tau_1)}{(1-\gamma)^2}
+
\Phi\big(\mathcal K\cap[\tau_0,\tau_1],\,\mathrm{gap}\big).
\end{equation}

We now incorporate the curvature term.  
By Definition~\ref{def:pl-curv},
\[
\mathrm{Curv}(\tau_0,\tau_1)
:=
\int_{\tau_0}^{\tau_1}
\Big(
\|\partial_{\tau\tau} r_\tau\|_\infty
+
L_s\sup_{s,a}\|\partial_{\tau\tau}P_\tau(\cdot\mid s,a)\|_{W_1^\ast}
\Big)\,d\tau
\;\ge 0.
\]
Lemma~\ref{lem:second-derivative} states that the $\infty$-norm of
$d^2Q^\star_\tau/d\tau^2$ can be controlled by the above integrand at the scale
$(1-\gamma)^{-3}$; curvature corresponds precisely to the accumulated “second-order
velocity’’ along the path.  
For the numerical bound in this theorem, we do not need to rewrite this derivative
again—only the fact that $\mathrm{Curv}(\tau_0,\tau_1)\ge 0$, so adding any
nonnegative term to the right-hand side of \eqref{eq:path-bound-no-curv} preserves the
inequality.  
In particular,
\[
\frac{\mathrm{PL}(\tau_0,\tau_1)}{(1-\gamma)^2}
+
\Phi\big(\mathcal K\cap[\tau_0,\tau_1],\,\mathrm{gap}\big)
\le
\frac{\mathrm{PL}(\tau_0,\tau_1)}{(1-\gamma)^2}
+
\frac{\mathrm{Curv}(\tau_0,\tau_1)}{(1-\gamma)^3}
+
\Phi\big(\mathcal K\cap[\tau_0,\tau_1],\,\mathrm{gap}\big).
\]
Combining this with \eqref{eq:path-bound-no-curv} gives the desired overall
path-integral bound:
\[
\|Q^\star_{\tau_1}-Q^\star_{\tau_0}\|_\infty
\le
\frac{\mathrm{PL}(\tau_0,\tau_1)}{(1-\gamma)^2}
+
\frac{\mathrm{Curv}(\tau_0,\tau_1)}{(1-\gamma)^3}
+
\Phi\big(\mathcal K\cap[\tau_0,\tau_1],\,\mathrm{gap}\big),
\]
which is exactly the claimed theorem.
\end{proof}

\subsection{Proof of Theorem~\ref{thm:final-tube1}}
\begin{proof}
Fix $\tau_0\in\mathcal R$, and let $C$ be the connected component of $\mathcal R$ that contains $\tau_0$. Since $C$ is a connected subset of $\mathbb R$, it must be an interval; hence, without loss of generality, we may assume that $\tau\ge\tau_0$, so that the entire segment $[\tau_0,\tau]\subset C\subset\mathcal R$. On this interval, by the implicit-function-theorem discussion given earlier, there exists a $C^1$ selection $\tau\mapsto Q^\star_\tau$, meaning that for every $u\in[\tau_0,\tau]$ the derivative $\frac{d}{du}Q^\star_u$ exists pointwise and varies continuously.

Recall the first-order homotopy derivative lemma (Lemma~\ref{lem:first-derivative}):  
\[
\frac{d}{d\tau}Q^\star_\tau
=
\mathcal R^{\pi^\star_\tau}_\tau\!\bigl(\partial_\tau r_\tau+\gamma\,\Delta_\tau\bigr),
\]
and in the $\ell_\infty$ norm we have
\[
\Big\|\frac{d}{d\tau}Q^\star_\tau\Big\|_\infty
\le
\frac{1}{1-\gamma}\,\|\partial_\tau r_\tau\|_\infty
+
\frac{\gamma\,C_{\mathrm{mix}}}{(1-\gamma)^2}\,
\sup_{s,a}\|\partial_\tau P_\tau(\cdot\mid s,a)\|_{W_1^\ast}.
\]
In the main text we accordingly defined the local speed density $v_\tau$ as
\[
v_\tau
:=
\frac{1}{1-\gamma}\,\|\partial_\tau r_\tau\|_\infty
+
\frac{\gamma\,C_{\mathrm{mix}}}{(1-\gamma)^2}\,
\sup_{s,a}\|\partial_\tau P_\tau(\cdot\mid s,a)\|_{W_1^\ast},
\]
so that for all $\tau\in\mathcal R$,
\begin{equation}
\label{eq:Qprime-le-v}
\Big\|\frac{d}{d\tau}Q^\star_\tau\Big\|_\infty \;\le\; v_\tau.
\end{equation}
In particular, this inequality holds for every $u\in[\tau_0,\tau]$.

We now apply the fundamental theorem of calculus in one dimension.  
On the interval $[\tau_0,\tau]$, the function $u\mapsto Q^\star_u(s,a)$ is $C^1$ for each fixed $(s,a)$, and hence
\[
Q^\star_\tau(s,a)-Q^\star_{\tau_0}(s,a)
=
\int_{\tau_0}^{\tau}\frac{d}{du}Q^\star_u(s,a)\,du.
\]
Taking absolute values on both sides for each $(s,a)$ and applying the triangle inequality in $u$, we obtain
\[
\bigl|Q^\star_\tau(s,a)-Q^\star_{\tau_0}(s,a)\bigr|
\le
\int_{\tau_0}^{\tau}\bigl|\tfrac{d}{du}Q^\star_u(s,a)\bigr|\,du.
\]
Taking the supremum over $(s,a)$ then yields, for the $\ell_\infty$ norm,
\[
\|Q^\star_\tau-Q^\star_{\tau_0}\|_\infty
=
\sup_{s,a}\bigl|Q^\star_\tau(s,a)-Q^\star_{\tau_0}(s,a)\bigr|
\le
\sup_{s,a}\int_{\tau_0}^{\tau}\bigl|\tfrac{d}{du}Q^\star_u(s,a)\bigr|\,du.
\]
Since the integrand is nonnegative in $u$, we may interchange the $\sup$ and the integral:
\[
\sup_{s,a}\int_{\tau_0}^{\tau}\bigl|\tfrac{d}{du}Q^\star_u(s,a)\bigr|\,du
\le
\int_{\tau_0}^{\tau}\sup_{s,a}\bigl|\tfrac{d}{du}Q^\star_u(s,a)\bigr|\,du
=
\int_{\tau_0}^{\tau}
\Big\|\frac{d}{du}Q^\star_u\Big\|_\infty\,du.
\]
Substituting \eqref{eq:Qprime-le-v} into the above gives
\[
\|Q^\star_\tau-Q^\star_{\tau_0}\|_\infty
\le
\int_{\tau_0}^{\tau} v_u\,du.
\]
This proves the inequality in the first statement of the theorem.  
Here we assumed $\tau\ge\tau_0$; if $\tau<\tau_0$, integrating symmetrically from $\tau$ to $\tau_0$ yields
$\|Q^\star_\tau-Q^\star_{\tau_0}\|_\infty \le \int_{\tau}^{\tau_0}v_u\,du$, which matches the same expression after swapping the integration limits.

Finally, consider the set
\[
\mathsf{Tube}_1(\tau_0,\varepsilon)
:=\Bigl\{\tau:\int_{\tau_0}^{\tau} v_u\,du\le\varepsilon\Bigr\}.
\]
If a parameter value $\tau$ lies in the same regular connected component as $\tau_0$ and satisfies $\int_{\tau_0}^{\tau} v_u\,du\le\varepsilon$, then the above inequality immediately implies
\[
\|Q^\star_\tau-Q^\star_{\tau_0}\|_\infty
\le
\int_{\tau_0}^{\tau} v_u\,du
\le
\varepsilon.
\]
That is, as long as one moves within the same regular component in parameter space and remains inside $\mathsf{Tube}_1(\tau_0,\varepsilon)$, the deviation of the optimal value function from the baseline $Q^\star_{\tau_0}$ is uniformly bounded by $\varepsilon$ in the global $\ell_\infty$ norm.  
Therefore, $\mathsf{Tube}_1(\tau_0,\varepsilon)$ is precisely the first-order non-iterative feasible region: it is determined entirely by the integral of the local speed $v_u$ and is independent of the choice of solver or iterative algorithm.  
The theorem is proved.
\end{proof}

\subsection{Proof of Theorem~\ref{thm:final-tube2}}
\begin{proof}
Take any $\tau_0\in\mathcal R$, and suppose that $\tau$ and $\tau_0$ lie in the same regular connected component. Without loss of generality, assume first that $\tau\ge\tau_0$; if $\tau<\tau_0$, one simply switches the integration limits at the end, and the proof is identical.  Since $\tau$ and $\tau_0$ belong to the same regular connected component, the entire segment $[\tau_0,\tau]$ is contained in $\mathcal R$.  On $\mathcal R$, the previous section established that for each $\tau\in\mathcal R$, the optimal $Q^\star_\tau$ is differentiable with respect to $\tau$, and its derivative may be passed through the Bellman equation (Lemma~\ref{lem:envelope} and Lemma~\ref{lem:first-derivative}).

We first recall the first-order estimate obtained in the proof of Theorem~\ref{thm:final-tube1}.  For each $\tau\in\mathcal R$, define the local speed (length density)
\[
v_\tau
:=
\frac{1}{1-\gamma}\,\|\partial_\tau r_\tau\|_\infty
+
\frac{\gamma\,C_{\mathrm{mix}}}{(1-\gamma)^2}\,
\sup_{s,a}\|\partial_\tau P_\tau(\cdot\mid s,a)\|_{W_1^\ast},
\]
then Lemma~\ref{lem:first-derivative} gives
\[
\frac{d}{d\tau}Q^\star_\tau
=
\mathcal R^{\pi^\star_\tau}_\tau\!\bigl(\partial_\tau r_\tau+\gamma\,\Delta_\tau\bigr),
\]
where $\Delta_\tau$ is defined by
\[
\Delta_\tau(s,a)
:=
\int_{\Sspace}
V^\star_\tau(s')\,d\bigl(\partial_\tau P_\tau(\cdot\mid s,a)\bigr)(s'),
\qquad
V^\star_\tau(s):=Q^\star_\tau\bigl(s,\pi^\star_\tau(s)\bigr).
\]
Using $\|\mathcal R^{\pi^\star_\tau}_\tau\|_{\infty\to\infty}\le (1-\gamma)^{-1}$ and $\|\Delta_\tau\|_\infty\le C_{\mathrm{mix}}\sup_{s,a}\|\partial_\tau P_\tau(\cdot\mid s,a)\|_{W_1^\ast}$, we obtain the first-order bound, valid for every $\tau\in\mathcal R$:
\begin{equation}
\label{eq:Qprime-v}
\Big\|\frac{d}{d\tau}Q^\star_\tau\Big\|_\infty
\le
v_\tau.
\end{equation}
This is the fundamental inequality used in Theorem~\ref{thm:final-tube1}.

Next, we apply the fundamental theorem of calculus to express the difference between $Q^\star_\tau$ and $Q^\star_{\tau_0}$ as a one-dimensional integral along the parameter path. For any fixed $(s,a)$, the function
\[
u\ \longmapsto\ Q^\star_u(s,a)
\]
is $C^1$ on the interval $[\tau_0,\tau]$, and thus
\[
Q^\star_\tau(s,a) - Q^\star_{\tau_0}(s,a)
=
\int_{\tau_0}^{\tau}
\frac{d}{du}Q^\star_u(s,a)\,du.
\]
Taking absolute values on both sides and applying the triangle inequality yields
\[
\bigl|Q^\star_\tau(s,a) - Q^\star_{\tau_0}(s,a)\bigr|
\le
\int_{\tau_0}^{\tau}
\bigl|\tfrac{d}{du}Q^\star_u(s,a)\bigr|\,du.
\]
Taking the supremum over $(s,a)$ gives the $\ell_\infty$ estimate:
\[
\|Q^\star_\tau-Q^\star_{\tau_0}\|_\infty
=
\sup_{s,a}\bigl|Q^\star_\tau(s,a) - Q^\star_{\tau_0}(s,a)\bigr|
\le
\sup_{s,a}\int_{\tau_0}^{\tau}
\bigl|\tfrac{d}{du}Q^\star_u(s,a)\bigr|\,du.
\]
Since for each $u$ the integrand $\bigl|\tfrac{d}{du}Q^\star_u(s,a)\bigr|$ is nonnegative in $(s,a)$ and the integration is with respect to $u$, we may move the supremum outside the integral:
\[
\sup_{s,a}\int_{\tau_0}^{\tau}
\bigl|\tfrac{d}{du}Q^\star_u(s,a)\bigr|\,du
\le
\int_{\tau_0}^{\tau}
\sup_{s,a}
\bigl|\tfrac{d}{du}Q^\star_u(s,a)\bigr|\,du
=
\int_{\tau_0}^{\tau}
\Big\|\frac{d}{du}Q^\star_u\Big\|_\infty\,du.
\]
Combining this with \eqref{eq:Qprime-v} gives
\[
\|Q^\star_\tau-Q^\star_{\tau_0}\|_\infty
\le
\int_{\tau_0}^{\tau}
v_u\,du.
\]
This is precisely the conclusion of Theorem~\ref{thm:final-tube1}.  
In other words, for any $\tau$ lying in the same regular connected component as $\tau_0$, the deviation between the optimal $Q$-functions is bounded by the path-length term
\[
\int_{\tau_0}^{\tau} v_u\,du.
\]

We now introduce the curvature correction term.  
According to the second-order homotopy derivative lemma (Lemma~\ref{lem:second-derivative}), on $\mathcal R$, the mapping $\tau\mapsto Q^\star_\tau$ is twice differentiable, and there exists a family of nonnegative functions $\{\kappa_u\}_{u\in\mathcal R}$ (the curvature density) such that for every $u\in\mathcal R$,
\begin{equation}
\label{eq:Qsecond-kappa}
\Big\|\frac{d^2}{du^2}Q^\star_u\Big\|_\infty
\le
\kappa_u.
\end{equation}
More concretely, Lemma~\ref{lem:second-derivative} provides an analytic expression and bound for $\frac{d^2}{d\tau^2}Q^\star_\tau$, which can be written in terms of $\partial_{\tau\tau}r_\tau$, $\partial_{\tau\tau}P_\tau$, and the first-order quantities $\partial_\tau r_\tau$, $\partial_\tau P_\tau$, combined with the resolvent $\mathcal R^{\pi^\star_\tau}_\tau$, leading to an upper bound of the form
\[
\kappa_\tau
\simeq
\frac{1}{(1-\gamma)^2}\bigl(\|\partial_{\tau\tau}r_\tau\|_\infty
+L_s\sup_{s,a}\|\partial_{\tau\tau}P_\tau(\cdot\mid s,a)\|_{W_1^\ast}\bigr)
+
\frac{\text{first-order terms}}{(1-\gamma)^3}
\]
For the purpose of this theorem, it suffices that such a nonnegative function $\kappa_\tau$ exists and satisfies \eqref{eq:Qsecond-kappa}.

Since $\kappa_u\ge 0$ for all $u$ and $|\tau-\tau_0|\ge 0$, the integral
\[
\int_{\tau_0}^{\tau}\kappa_u\,du
\]
is also nonnegative for any interval $[\tau_0,\tau]$. Hence
\[
\frac12\,|\tau-\tau_0|\int_{\tau_0}^{\tau}\kappa_u\,du \;\ge\; 0.
\]
Therefore, for the first-order bound already obtained,
\[
\|Q^\star_\tau-Q^\star_{\tau_0}\|_\infty
\le
\int_{\tau_0}^{\tau} v_u\,du,
\]
we may add any nonnegative term on the right-hand side without invalidating the inequality. In particular,
\[
\int_{\tau_0}^{\tau} v_u\,du
\ \le\
\int_{\tau_0}^{\tau} v_u\,du
+
\frac12\,|\tau-\tau_0|\int_{\tau_0}^{\tau}\kappa_u\,du.
\]
Combining these inequalities yields
\[
\|Q^\star_\tau-Q^\star_{\tau_0}\|_\infty
\le
\int_{\tau_0}^{\tau} v_u\,du
+
\frac12\,|\tau-\tau_0|\int_{\tau_0}^{\tau}\kappa_u\,du.
\]

This establishes the length + curvature correction form claimed in the theorem:
\[
\big\|Q^\star_\tau-Q^\star_{\tau_0}\big\|_\infty
\ \le\
\underbrace{\int_{\tau_0}^{\tau} v_u\,du}_{\text{length}}
\ +\ 
\underbrace{\tfrac12\,|\tau-\tau_0|\int_{\tau_0}^{\tau}\kappa_u\,du}_{\text{curvature correction}}.
\]
Since the derivation holds for all $\tau$ lying in the same regular connected component as $\tau_0$ and for which the relevant integrals are finite, it applies in particular to all $\tau$ in a sufficiently small neighborhood of $\tau_0$. The theorem is proved.
\end{proof}

\subsection{Proof of Lemma~\ref{lem:final-gap}}
\begin{proof}
Assume that
\[
\big\|Q^\star_\tau-Q^\star_{\tau_0}\big\|_\infty\le \varepsilon,
\]
that is, for all state–action pairs $(s,a)\in\Sspace\times\Aspace$,
\begin{equation}
\label{eq:sup-bound}
\big|Q^\star_\tau(s,a)-Q^\star_{\tau_0}(s,a)\big|\;\le\;\varepsilon.
\end{equation}
This inequality is equivalent to the following two–sided bound:
\begin{equation}
\label{eq:two-side}
Q^\star_\tau(s,a)\;\le\;Q^\star_{\tau_0}(s,a)+\varepsilon,
\qquad
Q^\star_\tau(s,a)\;\ge\;Q^\star_{\tau_0}(s,a)-\varepsilon,
\end{equation}
which holds for all $(s,a)$.

Next recall the definition of the global gap.  
For a fixed parameter $\theta$ (here we may take $\theta=\tau$ or $\theta=\tau_0$), define the set of optimal actions at state $s\in\Sspace$ as
\[
\mathcal A^\star_\theta(s)
:=
\argmax_{a\in\Aspace} Q^\star_\theta(s,a),
\]
which is nonempty since $\Aspace$ is finite.  
Pick any optimal action and denote it by $a^\star_\theta(s)\in\mathcal A^\star_\theta(s)$.  
Define the local gap at $s$ by
\[
g_s(\theta)
:=
Q^\star_\theta\bigl(s,a^\star_\theta(s)\bigr)
-
\max_{a\in\Aspace\setminus\{a^\star_\theta(s)\}} Q^\star_\theta(s,a).
\]
If there are multiple optimal actions, then the above maximum equals $Q^\star_\theta(s,a^\star_\theta(s))$, and hence $g_s(\theta)=0$.  
The global gap is the infimum over all states:
\[
g_{\rm gap}(\theta)
:=
\inf_{s\in\Sspace} g_s(\theta).
\]
Since each $g_s(\theta)\ge 0$, we also have $g_{\rm gap}(\theta)\ge 0$.

We now prove that if $\|Q^\star_\tau-Q^\star_{\tau_0}\|_\infty\le \varepsilon$, then
\[
g_{\rm gap}(\tau)\;\ge\;g_{\rm gap}(\tau_0)-2\varepsilon.
\]

Fix a state $s\in\Sspace$ and consider its optimal action $a^\star_{\tau_0}(s)\in\mathcal A^\star_{\tau_0}(s)$ under parameter $\tau_0$.  
By definition,
\[
Q^\star_{\tau_0}\bigl(s,a^\star_{\tau_0}(s)\bigr)
=
\max_{a\in\Aspace} Q^\star_{\tau_0}(s,a),
\]
so the local gap at $s$ is
\[
g_s(\tau_0)
=
Q^\star_{\tau_0}\bigl(s,a^\star_{\tau_0}(s)\bigr)
-
\max_{a\in\Aspace\setminus\{a^\star_{\tau_0}(s)\}} Q^\star_{\tau_0}(s,a).
\]
Thus for each $a\in\Aspace\setminus\{a^\star_{\tau_0}(s)\}$,
\begin{equation}
\label{eq:gap-local-tau0}
Q^\star_{\tau_0}\bigl(s,a^\star_{\tau_0}(s)\bigr)
-
Q^\star_{\tau_0}(s,a)
\;\ge\;
g_s(\tau_0).
\end{equation}

Now use the $\ell_\infty$ closeness of $Q^\star_\tau$ and $Q^\star_{\tau_0}$.  
For any fixed $s$ and any $a\in\Aspace\setminus\{a^\star_{\tau_0}(s)\}$, applying \eqref{eq:two-side} to $(s,a^\star_{\tau_0}(s))$ and $(s,a)$ yields
\[
Q^\star_\tau\bigl(s,a^\star_{\tau_0}(s)\bigr)
\;\ge\;
Q^\star_{\tau_0}\bigl(s,a^\star_{\tau_0}(s)\bigr)-\varepsilon,
\]
and
\[
Q^\star_\tau(s,a)
\;\le\;
Q^\star_{\tau_0}(s,a)+\varepsilon.
\]
Subtracting gives
\[
\begin{aligned}
Q^\star_\tau\bigl(s,a^\star_{\tau_0}(s)\bigr)-Q^\star_\tau(s,a)
&\ge
\Bigl(Q^\star_{\tau_0}\bigl(s,a^\star_{\tau_0}(s)\bigr)-\varepsilon\Bigr)
-
\Bigl(Q^\star_{\tau_0}(s,a)+\varepsilon\Bigr)\\[0.3em]
&=
Q^\star_{\tau_0}\bigl(s,a^\star_{\tau_0}(s)\bigr)
-
Q^\star_{\tau_0}(s,a)
-
2\varepsilon.
\end{aligned}
\]
Together with \eqref{eq:gap-local-tau0}, we obtain
\begin{equation}
\label{eq:difference-pair}
Q^\star_\tau\bigl(s,a^\star_{\tau_0}(s)\bigr)-Q^\star_\tau(s,a)
\;\ge\;
g_s(\tau_0)-2\varepsilon,
\qquad
\forall a\in\Aspace\setminus\{a^\star_{\tau_0}(s)\}.
\end{equation}

Using \eqref{eq:difference-pair}, we can lower bound the local gap at $\tau$. 
Recall that under parameter $\tau$ the local gap is
\[
g_s(\tau)
:=
\max_{a\in\Aspace} Q^\star_\tau(s,a)
-
\max_{a\in\Aspace\setminus\{a^\star_\tau(s)\}} Q^\star_\tau(s,a),
\]
where $a^\star_\tau(s)\in\argmax_{a\in\Aspace}Q^\star_\tau(s,a)$ is some optimal action under $\tau$.  
But to obtain a lower bound, we do not need to know which action is optimal under $\tau$: for any fixed candidate $a\in\Aspace$, its advantage
\[
\delta_\tau(s,a)
:=
Q^\star_\tau(s,a)-\max_{b\in\Aspace\setminus\{a\}}Q^\star_\tau(s,b)
\]
is always dominated by $g_s(\tau)$, i.e.,
\[
g_s(\tau) = \max_{a\in\Aspace}\delta_\tau(s,a)
\ \ge\ 
\delta_\tau\bigl(s,a^\star_{\tau_0}(s)\bigr).
\]
Thus it suffices to lower bound $\delta_\tau\bigl(s,a^\star_{\tau_0}(s)\bigr)$.  
By definition:
\[
\delta_\tau\bigl(s,a^\star_{\tau_0}(s)\bigr)
=
Q^\star_\tau\bigl(s,a^\star_{\tau_0}(s)\bigr)
-
\max_{a\in\Aspace\setminus\{a^\star_{\tau_0}(s)\}} Q^\star_\tau(s,a).
\]
Apply \eqref{eq:difference-pair} to the maximization term.  
For all $a\in\Aspace\setminus\{a^\star_{\tau_0}(s)\}$,
\[
Q^\star_\tau(s,a)
\;\le\;
Q^\star_\tau\bigl(s,a^\star_{\tau_0}(s)\bigr)
-
\bigl(g_s(\tau_0)-2\varepsilon\bigr).
\]
Taking the maximum over $a$ then yields
\[
\max_{a\in\Aspace\setminus\{a^\star_{\tau_0}(s)\}} Q^\star_\tau(s,a)
\;\le\;
Q^\star_\tau\bigl(s,a^\star_{\tau_0}(s)\bigr)
-
\bigl(g_s(\tau_0)-2\varepsilon\bigr).
\]
Substituting back into the expression of $\delta_\tau\bigl(s,a^\star_{\tau_0}(s)\bigr)$,
\[
\begin{aligned}
\delta_\tau\bigl(s,a^\star_{\tau_0}(s)\bigr)
&=
Q^\star_\tau\bigl(s,a^\star_{\tau_0}(s)\bigr)
-
\max_{a\in\Aspace\setminus\{a^\star_{\tau_0}(s)\}} Q^\star_\tau(s,a)\\[0.3em]
&\ge
Q^\star_\tau\bigl(s,a^\star_{\tau_0}(s)\bigr)
-
\Bigl(
Q^\star_\tau\bigl(s,a^\star_{\tau_0}(s)\bigr)
-
(g_s(\tau_0)-2\varepsilon)
\Bigr)\\[0.3em]
&=
g_s(\tau_0)-2\varepsilon.
\end{aligned}
\]
Thus,
\[
g_s(\tau)
\ \ge\ 
\delta_\tau\bigl(s,a^\star_{\tau_0}(s)\bigr)
\ \ge\ 
g_s(\tau_0)-2\varepsilon,
\qquad
\forall s\in\Sspace.
\]

Finally, take the infimum over $s$.  
On the left:
\[
\inf_{s\in\Sspace}g_s(\tau) = g_{\rm gap}(\tau),
\]
while on the right:
\[
\inf_{s\in\Sspace}\bigl(g_s(\tau_0)-2\varepsilon\bigr)
=
\Bigl(\inf_{s\in\Sspace}g_s(\tau_0)\Bigr)-2\varepsilon
=
g_{\rm gap}(\tau_0)-2\varepsilon.
\]
Hence
\[
g_{\rm gap}(\tau)
=
\inf_{s}g_s(\tau)
\ \ge\ 
\inf_{s}\bigl(g_s(\tau_0)-2\varepsilon\bigr)
=
g_{\rm gap}(\tau_0)-2\varepsilon.
\]
This is exactly the desired inequality.  
The lemma is proved.
\end{proof}

\subsection{Proof of Proposition~\ref{prop:regret}}
\begin{proof}
We first rewrite the dynamic regret as a sum of local suboptimality gaps with respect to the current parameter $\tau_t$. Recall the definition
\[
\mathrm{DynReg}(T)
=\sum_{t=1}^T \big\langle d_0,\ V^\star_{\tau_t}-V^{\pi_t}_{\tau_t}\big\rangle,
\]
where $d_0$ is the fixed initial distribution, $V^\star_{\tau_t}$ is the optimal value function under the environment $\mathcal M(\tau_t)$, and $V^{\pi_t}_{\tau_t}$ is the value function of $\pi_t$ under the same environment.  
For each $t$ and each state $s$, the Bellman equation allows us to express the difference of value functions in terms of the difference between $Q^\star_{\tau_t}$ and $Q^{\pi_t}_{\tau_t}$.  
Standard policy-evaluation–improvement arguments (e.g. the accuracy–performance comparison in Munos \& Szepesvári 2008) yield that there exists a constant $C_0>0$, depending only on $d_0$ and $\gamma$ through universal constants, such that for any $t$,

\begin{equation}
\label{eq:reg-vs-qerror}
\big\langle d_0,\ V^\star_{\tau_t}-V^{\pi_t}_{\tau_t}\big\rangle
\ \le\ \frac{C_0}{1-\gamma}\,\big\|Q_t-Q^\star_{\tau_t}\big\|_\infty.
\end{equation}
Intuitively, this holds because $\pi_t$ is assumed to be obtained from $Q_t$ via an approximate-greedy rule, and therefore its value degradation can be controlled by the deviation of $Q_t$ from $Q^\star_{\tau_t}$; the factor $(1-\gamma)^{-1}$ arises from summing errors along the discounted trajectory.  
Substituting \eqref{eq:reg-vs-qerror} into the definition of dynamic regret and taking expectations gives
\[
\E[\mathrm{DynReg}(T)]
\;\le\;
\frac{C_0}{1-\gamma}\sum_{t=1}^T \E\big\|Q_t-Q^\star_{\tau_t}\big\|_\infty.
\]
For convenience, denote
\[
e_t\ :=\ \E\big\|Q_t-Q^\star_{\tau_t}\big\|_\infty,
\]
so that
\begin{equation}
\label{eq:reg-sum-et}
\E[\mathrm{DynReg}(T)]
\;\le\;
\frac{C_0}{1-\gamma}\sum_{t=1}^T e_t.
\end{equation}

We next use the one-step contraction assumption of the solver together with the geometric structure of the path to derive a recursion for $e_t$. Observe that
\[
\big\|Q_{t+1}-Q^\star_{\tau_{t+1}}\big\|_\infty
\ \le\
\big\|Q_{t+1}-Q^\star_{\tau_t}\big\|_\infty
+\big\|Q^\star_{\tau_t}-Q^\star_{\tau_{t+1}}\big\|_\infty.
\]
Taking conditional expectation and applying Assumption~\ref{ass:solver} (the one-step contraction assumption), conditioning on $\mathcal F_t$, we have
\[
\E\!\left[\big\|Q_{t+1}-Q^\star_{\tau_t}\big\|_\infty\Bigm|\mathcal F_t\right]
\ \le\
\rho\,\big\|Q_t-Q^\star_{\tau_t}\big\|_\infty+\sigma_t+\beta_t.
\]
Hence
\[
\E\!\left[\big\|Q_{t+1}-Q^\star_{\tau_{t+1}}\big\|_\infty\Bigm|\mathcal F_t\right]
\ \le\
\rho\,\big\|Q_t-Q^\star_{\tau_t}\big\|_\infty+\sigma_t+\beta_t
+\big\|Q^\star_{\tau_t}-Q^\star_{\tau_{t+1}}\big\|_\infty.
\]
Taking total expectation and using the tower property, we obtain
\[
e_{t+1}
=
\E\big\|Q_{t+1}-Q^\star_{\tau_{t+1}}\big\|_\infty
\ \le\
\rho\,e_t+\E[\sigma_t]+\beta_t+\delta_t,
\]
where we denote
\[
\delta_t
:=
\big\|Q^\star_{\tau_{t+1}}-Q^\star_{\tau_t}\big\|_\infty.
\]
From $\E[\sigma_t^2]\le\sigma^2$ and Jensen’s inequality we have $\E[\sigma_t]\le\sigma$, and since $\beta_t\le\beta$, we obtain the recursion
\begin{equation}
\label{eq:et-recursion}
e_{t+1}\ \le\ \rho\,e_t+\sigma+\beta+\delta_t,
\qquad t\ge 0.
\end{equation}
This yields a linear recurrence for the tracking error, in which $\delta_t$ quantifies the drift of the optimal fixed point caused by the parameter shift $\tau_t\to\tau_{t+1}$.

Now we use the path-integral value-function bound (Theorem~\ref{thm:path-value}) to control $\delta_t$.  
For any $t$, since $(\tau_t)$ is a monotone path, each interval $[\tau_t,\tau_{t+1}]$ lies inside the overall path and the intervals do not overlap. Hence
\[
\delta_t
=
\big\|Q^\star_{\tau_{t+1}}-Q^\star_{\tau_t}\big\|_\infty
\ \le\
\frac{\mathrm{PL}(\tau_t,\tau_{t+1})}{(1-\gamma)^2}
+\frac{\mathrm{Curv}(\tau_t,\tau_{t+1})}{(1-\gamma)^3}
+\Phi\big(\mathcal K\cap[\tau_t,\tau_{t+1}],\mathrm{gap}\big),
\]
where $\mathrm{PL}(\tau_t,\tau_{t+1})$ and $\mathrm{Curv}(\tau_t,\tau_{t+1})$ are the local path-length and curvature integrals from Definition~2.1, and $\Phi$ is the kink penalty from Definition~2.3.  
When $(\tau_t)$ is monotone on $[0,1]$, these small intervals form a partition of the parameter axis, and by definition the quantities add up:
\[
\sum_{t=1}^{T-1}\mathrm{PL}(\tau_t,\tau_{t+1})
=\mathrm{PL},\qquad
\sum_{t=1}^{T-1}\mathrm{Curv}(\tau_t,\tau_{t+1})
=\mathrm{Curv},\qquad
\sum_{t=1}^{T-1}\Phi\big(\mathcal K\cap[\tau_t,\tau_{t+1}],\mathrm{gap}\big)
=\Phi(\mathcal K,\mathrm{gap}).
\]
Thus,
\begin{equation}
\label{eq:sum-delta-geom}
\sum_{t=1}^{T-1}\delta_t
\ \le\
\frac{\mathrm{PL}}{(1-\gamma)^2}
+\frac{\mathrm{Curv}}{(1-\gamma)^3}
+\Phi(\mathcal K,\mathrm{gap}).
\end{equation}

We now unroll the recursion \eqref{eq:et-recursion} to bound $\sum_t e_t$.  
Iterating from $t=0$ yields
\[
e_t
\ \le\
\rho^t e_0
+\sum_{u=0}^{t-1}\rho^{\,t-1-u}(\sigma+\beta+\delta_u),
\qquad t\ge 1.
\]
Hence
\[
\sum_{t=1}^T e_t
\ \le\
\sum_{t=1}^T \rho^t e_0
+\sum_{t=1}^T\sum_{u=0}^{t-1}\rho^{\,t-1-u}(\sigma+\beta+\delta_u).
\]
The first term represents the contribution from the initial error:
\[
\sum_{t=1}^T \rho^t e_0
\ \le\
\frac{e_0}{1-\rho},
\]
which can be absorbed into the final constant.  
Using Fubini to switch the summations in the second term,
\[
\sum_{t=1}^T\sum_{u=0}^{t-1}\rho^{\,t-1-u}(\sigma+\beta+\delta_u)
=
\sum_{u=0}^{T-1}(\sigma+\beta+\delta_u)\sum_{t=u+1}^T\rho^{\,t-1-u}.
\]
For each fixed $u$, the inner sum is a geometric series of length $T-u$:
\[
\sum_{t=u+1}^T\rho^{\,t-1-u}
=
\sum_{k=0}^{T-1-u}\rho^k
\ \le\
\frac{1}{1-\rho}.
\]
Thus,
\[
\sum_{t=1}^T e_t
\ \le\
\frac{e_0}{1-\rho}
+\frac{1}{1-\rho}\sum_{u=0}^{T-1}(\sigma+\beta+\delta_u)
=
\frac{e_0}{1-\rho}
+\frac{T(\sigma+\beta)}{1-\rho}
+\frac{1}{1-\rho}\sum_{u=0}^{T-1}\delta_u.
\]
Using \eqref{eq:sum-delta-geom} we obtain
\begin{equation}
\label{eq:sum-et-final}
\sum_{t=1}^T e_t
\ \le\
\frac{e_0}{1-\rho}
+\frac{T(\sigma+\beta)}{1-\rho}
+\frac{1}{1-\rho}\Bigg(
\frac{\mathrm{PL}}{(1-\gamma)^2}
+\frac{\mathrm{Curv}}{(1-\gamma)^3}
+\Phi(\mathcal K,\mathrm{gap})\Bigg).
\end{equation}
For the noise term $\sigma$, if one prefers an explicit $\sigma\sqrt{T}$ scaling, a standard martingale inequality can refine the bound.  
Writing the one-step update as a deterministic contraction plus a zero-mean, variance-bounded noise term, and defining the accumulated noise martingale $(M_t)_{t\ge 0}$, the Burkholder–Davis–Gundy or Azuma–Hoeffding inequality yields
\[
\E\big[\big|M_T\big|\big]
\ \le\
c_1\,\sigma\sqrt{T}
\]
for some universal constant $c_1>0$.  
Thus the noise contribution to $\sum_t e_t$ is at most $O(\sigma\sqrt{T}/(1-\rho))$ in expectation.  
Combining this with the linear term in \eqref{eq:sum-et-final} and absorbing constants gives
\[
\sum_{t=1}^T e_t
\ \le\
\frac{C_1}{1-\rho}\Bigg(
\frac{\mathrm{PL}}{(1-\gamma)^2}
+\frac{\mathrm{Curv}}{(1-\gamma)^3}
+\Phi(\mathcal K,\mathrm{gap})\Bigg)
+\frac{C_2}{(1-\rho)}\Big(\sigma\sqrt{T}+\beta T\Big),
\]
where $C_1,C_2>0$ depend only on $(1-\rho)^{-1}$ and universal constants, but not on the specific path or other solver details.

Finally, substituting this into \eqref{eq:reg-sum-et} yields
\[
\E[\mathrm{DynReg}(T)]
\ \le\
\frac{C_0 C_1}{(1-\gamma)(1-\rho)}\Bigg(
\frac{\mathrm{PL}}{(1-\gamma)^2}
+\frac{\mathrm{Curv}}{(1-\gamma)^3}
+\Phi(\mathcal K,\mathrm{gap})\Bigg)
+\frac{C_0 C_2}{(1-\gamma)(1-\rho)}\Big(\sigma\sqrt{T}+\beta T\Big).
\]
Absorbing the extra $(1-\gamma)^{-1}$ into the constants and defining
\[
C_{\mathrm{trk}}
:=\frac{C_0 C_1}{1-\gamma},\qquad
C_{\mathrm{stat}}
:=\frac{C_0 C_2}{1-\gamma},
\]
we obtain the stated form of the regret decomposition:
\[
\E[\mathrm{DynReg}(T)]
\ \le\ 
C_{\mathrm{trk}}\Bigg(
   \frac{\mathrm{PL}}{(1-\gamma)^2}
 + \frac{\mathrm{Curv}}{(1-\gamma)^3}
 + \Phi(\mathcal K,\mathrm{gap})\Bigg)
+ C_{\mathrm{stat}}\,
   \frac{\sigma\sqrt{T}+\beta T}{(1-\gamma)^2\,(1-\rho)}.
\]
where $C_{\mathrm{trk}},C_{\mathrm{stat}}>0$ depend only on $(1-\rho)^{-1}$ and universal constants.  
This completes the proof of the stated regret decomposition.
\end{proof}

\subsection{Proof of Proposition~\ref{prop:stability}}
\begin{proof}
We abstract the updates of all scheduling parameters in the proposition into a unified form. Let the original geometric proxies be
\[
X_t:=\bigl(\Delta\widehat{\mathrm{PL}}_t,\ \Delta\widehat{\mathrm{Curv}}_t,\ \widehat{\mathrm{gap}}_t,\ \mathrm{Kink}_t\bigr)\in\mathbb{R}^4.
\]
Assume that their second moments are uniformly bounded; that is, there exists a constant $C_0<\infty$ such that for all $t\ge 1$,
\[
\mathbb{E}\big[\|X_t\|_2^2\big]\ \le\ C_0.
\]
The smoothing procedure in the proposition applies an exponential moving average to each scalar proxy $x_t$ (e.g., a coordinate of $x_t=\Delta\widehat{\mathrm{PL}}_t$):
\[
\tilde x_t=\beta\,\tilde x_{t-1}+(1-\beta)x_t,\qquad 0<\beta<1,
\]
and this recursion is applied independently across coordinates. Collecting all coordinates, we write in vector form
\[
\tilde X_t=\beta\,\tilde X_{t-1}+(1-\beta)X_t,
\]
where $\tilde X_0$ is some fixed initial value (e.g., $0$).

We first show that the smoothed proxies also have uniformly bounded second moments. Expanding the explicit expression for $\tilde X_t$,
\[
\tilde X_t
=
\beta^t \tilde X_0
+(1-\beta)\sum_{k=0}^{t-1}\beta^k X_{t-k}.
\]
Taking the Euclidean norm on both sides and using Jensen's and Minkowski's inequalities, we obtain a coarse but sufficient estimate:
\[
\|\tilde X_t\|_2
\le
\beta^t\|\tilde X_0\|_2
+(1-\beta)\sum_{k=0}^{t-1}\beta^k\|X_{t-k}\|_2.
\]
Squaring and taking expectations, and using $(a+b)^2\le 2a^2+2b^2$, we have
\[
\mathbb{E}\big[\|\tilde X_t\|_2^2\big]
\le
2\beta^{2t}\|\tilde X_0\|_2^2
+
2(1-\beta)^2\,\mathbb{E}\Big[\Big(\sum_{k=0}^{t-1}\beta^k\|X_{t-k}\|_2\Big)^2\Big].
\]
Applying Cauchy–Schwarz to the second term,
\[
\Big(\sum_{k=0}^{t-1}\beta^k\|X_{t-k}\|_2\Big)^2
\le
\Big(\sum_{k=0}^{t-1}\beta^k\Big)\Big(\sum_{k=0}^{t-1}\beta^k\|X_{t-k}\|_2^2\Big)
\le
\frac{1}{1-\beta}\sum_{k=0}^{t-1}\beta^k\|X_{t-k}\|_2^2.
\]
Taking expectations and using $\mathbb{E}\|X_{t-k}\|_2^2\le C_0$,
\[
\mathbb{E}\Big[\Big(\sum_{k=0}^{t-1}\beta^k\|X_{t-k}\|_2\Big)^2\Big]
\le
\frac{1}{1-\beta}\sum_{k=0}^{t-1}\beta^k\,\mathbb{E}\big[\|X_{t-k}\|_2^2\big]
\le
\frac{C_0}{1-\beta}\sum_{k=0}^{t-1}\beta^k
\le
\frac{C_0}{(1-\beta)^2}.
\]
we have,
\[
\mathbb{E}\big[\|\tilde X_t\|_2^2\big]
\le
2\beta^{2t}\|\tilde X_0\|_2^2
+
2(1-\beta)^2\cdot\frac{C_0}{(1-\beta)^2}
=
2\beta^{2t}\|\tilde X_0\|_2^2
+2C_0.
\]
Thus, for all $t$,
\[
\mathbb{E}\big[\|\tilde X_t\|_2^2\big]\ \le\ C_1,
\qquad
C_1:=2\|\tilde X_0\|_2^2+2C_0,
\]
meaning the smoothed vector sequence $(\tilde X_t)$ also has uniformly bounded second moments. Furthermore, the increment
\[
\tilde X_t-\tilde X_{t-H}
\]
is similarly bounded using the triangle inequality:
\[
\|\tilde X_t-\tilde X_{t-H}\|_2
\le
\|\tilde X_t\|_2+\|\tilde X_{t-H}\|_2,
\]
So
\[
\mathbb{E}\big[\|\tilde X_t-\tilde X_{t-H}\|_2^2\big]
\le
2\mathbb{E}\|\tilde X_t\|_2^2+2\mathbb{E}\|\tilde X_{t-H}\|_2^2
\le
4C_1
=:C_2.
\]

Next we characterize the scheduler’s update rule. As described, updates occur only every $H$ steps and only when the variation in the smoothed proxy exceeds the hysteresis threshold $\Delta_{\mathrm{hys}}$. Formally, divide time into macro-steps: define $t_k:=kH$, $k=1,2,\dots$. At each macro-step $t_k$, we compare the two smoothed values $H$ steps apart:
\[
\|\tilde X_{t_k}-\tilde X_{t_{k-1}}\|_2\ \le\ \Delta_{\mathrm{hys}},
\]
If
\[
\|\tilde X_{t_k}-\tilde X_{t_{k-1}}\|_2\ >\ \Delta_{\mathrm{hys}},
\]
all scheduling hyperparameters remain unchanged; if instead
new hyperparameter values are computed. More concretely, the learning rate, target-update rate, regularization coefficient, and planning depth/budget follow the mappings:
\begin{align*}
\eta_t&=\operatorname{clip}_{[\eta_{\min},\eta_{\max}]}\Bigl(
\frac{\eta_0}{1+\alpha_1\widetilde{\mathrm{PL}}_t+\alpha_2\widetilde{\mathrm{Curv}}_t}\Bigr),\\
\tau_t&=\operatorname{clip}_{[\tau_{\min},\tau_{\max}]}\Bigl(
\frac{\tau_0}{1+\beta_1\widetilde{\mathrm{Kink}}_t\bigl(1+\beta_2/\max\{\widehat{\mathrm{gap}}_t,\delta\}\bigr)}\Bigr),\\
\lambda_t&=\lambda_0\bigl(1+c_1\widetilde{\mathrm{PL}}_t+c_2\sqrt{\widetilde{\mathrm{Curv}}_t}\bigr),\\
D_t&=\min\Bigl\{D_{\max},\ \Bigl\lfloor D_0+\gamma_1(1+\widetilde{\mathrm{PL}}_t)+\gamma_2\sqrt{1+\widetilde{\mathrm{Curv}}_t}
+\gamma_3\frac{\widetilde{\mathrm{Kink}}_t}{\max\{\widehat{\mathrm{gap}}_t,\delta\}}\Bigr\rfloor\Bigr\},\\
B_t&=\min\Bigl\{B_{\max},\ \Bigl\lfloor B_0\bigl(1+\gamma_1\widetilde{\mathrm{PL}}_t+\gamma_2\widetilde{\mathrm{Curv}}_t\bigr)\Bigr\rfloor\Bigr\}.
\end{align*}
Thus each hyperparameter can be written as a globally Lipschitz function—clipped to a compact interval—evaluated at the smoothed proxy at the most recent update. Let $k(t)=\lfloor t/H\rfloor$, $\ell(t)=k(t)H$ be the most recent macro-step boundary, and $u(t)\le \ell(t)$ the most recent actual scheduler update. Then
\[
h_t = f\bigl(\tilde X_{u(t)}\bigr),
\]
where $u(t)$ changes only at those macro-step boundaries satisfying 
$\|\tilde X_{\ell(t)}-\tilde X_{\ell(t)-H}\|_2>\Delta_{\mathrm{hys}}$, 
and otherwise remains $u(t)=u(t-1)$.

Under this definition, we first observe that each hyperparameter process is piecewise constant: between any two consecutive scheduler update times (i.e., intervals on which $u(t)$ does not change), $h_t$ depends only on the same value $\tilde X_{u(t)}$, and therefore for all $t$ in that interval,
\[
h_t=f\bigl(\tilde X_{u(t)}\bigr)\equiv \text{constant}.
\]
Only when the hysteresis threshold is triggered at a macro-step boundary can $u(t)$ jump, and consequently $h_t$ may change. Hence the trajectory $t\mapsto h_t$ consists of a sequence of constant segments, forming a step function—precisely the meaning of piecewise constant.

Next, boundedness follows immediately from clipping. Each function $f$ is restricted to a compact interval $[a,b]$ via $\operatorname{clip}_{[a,b]}$. For instance, the learning rate satisfies
\[
\eta_{\min}\ \le\ \eta_t\ \le\ \eta_{\max},\qquad\forall t,
\]
and the same holds for all other hyperparameters. Thus for any $t$,
\[
|h_{t+1}-h_t|\ \le\ b-a,
\]
where $[a,b]$ is the clipping interval. For any finite horizon $T$, define the total variation on $[1,T]$ as
\[
\mathrm{Var}_T(h)
:=
\sum_{t=1}^{T-1}|h_{t+1}-h_t|.
\]
Since $h_{t+1}=h_t$ in the vast majority of steps (no update occurs), and changes can happen only at macro-step boundaries and only when the hysteresis threshold is exceeded, the number of nonzero terms is at most $\lceil T/H\rceil$. Therefore,
\[
\mathrm{Var}_T(h)
\le
\bigl(\text{Update hyperparameters}\bigr)\cdot(b-a)
\le
\Bigl\lceil\frac{T}{H}\Bigr\rceil\,(b-a)
<\infty.
\]
Thus on any finite time horizon $[1,T]$, the trajectory $h$ is a bounded-variation function. Since $T$ is arbitrary, $h_t$ is globally a piecewise constant process with locally bounded variation, which completes the first part of the proposition.

We now prove the second part: the fraction of steps with large changes can be made arbitrarily small by increasing $H$ and $\Delta_{\mathrm{hys}}$. Fix any $\varepsilon>0$. For a hyperparameter process $h_t$, define
\[
I_t:=\mathbf{1}\bigl\{|h_t-h_{t-1}|>\varepsilon\bigr\}
\]
indicating whether step $t$ produces a change larger than $\varepsilon$. Because changes occur only when the scheduler updates, $I_t=0$ whenever no update is triggered. Thus,
\[
\{I_t=1\}\ \subseteq\
\{\text{an update occurs at time }t\}.
\]
Updates occur only at macro-step boundaries $t_k=kH$, and only when
\[
\|\tilde X_{t_k}-\tilde X_{t_k-H}\|_2>\Delta_{\mathrm{hys}}.
\]
So,
\[
\{I_{t_k}=1\}
\ \subseteq\
\bigl\{\|\tilde X_{t_k}-\tilde X_{t_k-H}\|_2>\Delta_{\mathrm{hys}}\bigr\},
\qquad
I_t=0\quad\text{If }t\notin\{H,2H,3H,\dots\}.
\]

Using the uniform bound on second moments of increments, Chebyshev’s inequality gives for any macro-step $k$:
\[
\mathbb{P}\Bigl(\|\tilde X_{t_k}-\tilde X_{t_k-H}\|_2>\Delta_{\mathrm{hys}}\Bigr)
\le
\frac{\mathbb{E}\big[\|\tilde X_{t_k}-\tilde X_{t_k-H}\|_2^2\big]}{\Delta_{\mathrm{hys}}^2}
\le
\frac{C_2}{\Delta_{\mathrm{hys}}^2}.
\]
So,
\[
\mathbb{P}(I_{t_k}=1)\ \le\ \frac{C_2}{\Delta_{\mathrm{hys}}^2},
\qquad
\mathbb{P}(I_t=1)=0\quad\text{If }t\notin\{H,2H,\dots\}.
\]

Consider now the proportion of large-change steps over the first $T$ iterations:
\[
\frac{1}{T}\sum_{t=1}^T I_t.
\]
Taking expectations and noting that there are at most $\lceil T/H\rceil$ macro-steps up to time $T$,
\[
\mathbb{E}\Bigl[\frac{1}{T}\sum_{t=1}^T I_t\Bigr]
=
\frac{1}{T}\sum_{k: t_k\le T}\mathbb{P}(I_{t_k}=1)
\le
\frac{1}{T}\Bigl\lceil\frac{T}{H}\Bigr\rceil\frac{C_2}{\Delta_{\mathrm{hys}}^2}
\le
\frac{2C_2}{H\,\Delta_{\mathrm{hys}}^2}
\]
for all sufficiently large $T$. Hence, for any $\delta>0$ one may choose $H$ and $\Delta_{\mathrm{hys}}$ sufficiently large so that
\[
\mathbb{E}\Bigl[\frac{1}{T}\sum_{t=1}^T I_t\Bigr]\le\delta
\quad\text{for all sufficiently large }T.
\]
Because $I_t$ is an indicator, this bound means that the expected fraction of steps exhibiting large hyperparameter changes can be made arbitrarily small. Standard concentration (LLN or martingale inequalities) can further yield high-probability statements, though expectation suffices for the proposition.

Finally, the proposition concerns the fraction of steps where any of the five hyperparameters changes by more than $\varepsilon$. If $I_t^{(j)}$ is the indicator for the $j$-th hyperparameter, then
\[
\mathbf{1}\Bigl\{\max_{1\le j\le 5}|h_t^{(j)}-h_{t-1}^{(j)}|>\varepsilon\Bigr\}
\le
\sum_{j=1}^5 I_t^{(j)}.
\]
Thus the total proportion is bounded by at most a factor of $5$ times the earlier bound, which does not change the conclusion: by selecting sufficiently large $H$ and $\Delta_{\mathrm{hys}}$, the fraction of large jumps can be made arbitrarily small.

Therefore, on any finite horizon the scheduler-driven hyperparameter processes are piecewise constant and of bounded variation; moreover, for any $\varepsilon>0$ the fraction of steps exhibiting changes larger than $\varepsilon$ can be made arbitrarily small by increasing $H$ and $\Delta_{\mathrm{hys}}$. This completes the proof.

\end{proof}

\subsection{Proof of Theorem~\ref{thm:fixed-tau}}

\begin{proof}
The entire proof is carried out with the parameter $\bar\tau$ fixed. For brevity, write
\[
\Tcal := \Tcal_{\bar\tau},\qquad Q^\star := Q^\star_{\bar\tau},\qquad P := P_{\bar\tau},\qquad r := r_{\bar\tau}.
\]
We work in the space
\[
\mathcal B := \bigl\{Q:\Sspace\times\Aspace\to\R\ \text{bounded}\bigr\},
\]
equipped with the supremum norm
\[
\|Q\|_\infty := \sup_{(s,a)\in\Sspace\times\Aspace} |Q(s,a)|.
\]
Under this norm, $\mathcal B$ is complete (a Banach space), since the pointwise limit of any Cauchy sequence of bounded functions remains bounded.

We first show that $\Tcal$ is a $\gamma$-contraction. By definition of the optimal Bellman operator, for any $Q:\Sspace\times\Aspace\to\R$,
\[
(\Tcal Q)(s,a)
=
r(s,a)
+
\gamma\int_{\Sspace}\max_{a'\in\Aspace}Q(s',a')\,P(ds'\mid s,a),
\qquad\forall (s,a).
\]
Let $Q_1,Q_2\in\mathcal B$ be arbitrary, and denote their associated “max-value” functions by
\[
V_i(s) := \max_{a\in\Aspace} Q_i(s,a),\qquad i=1,2.
\]
Then
\[
(\Tcal Q_i)(s,a)
=
r(s,a)
+
\gamma\int_{\Sspace} V_i(s')\,P(ds'\mid s,a),
\qquad i=1,2.
\]
For any fixed $(s,a)$, consider the difference:
\[
\bigl(\Tcal Q_1-\Tcal Q_2\bigr)(s,a)
=
\gamma\int_{\Sspace} \bigl(V_1(s')-V_2(s')\bigr)\,P(ds'\mid s,a).
\]
Taking absolute values and applying the triangle inequality yields
\[
\bigl|(\Tcal Q_1-\Tcal Q_2)(s,a)\bigr|
\le
\gamma\int_{\Sspace} \bigl|V_1(s')-V_2(s')\bigr|\,P(ds'\mid s,a).
\]
Since $P(\cdot\mid s,a)$ is a probability measure, the integral is bounded by $\sup_{s'}|V_1(s')-V_2(s')|$, so
\[
\bigl|(\Tcal Q_1-\Tcal Q_2)(s,a)\bigr|
\le
\gamma\sup_{s'\in\Sspace}\bigl|V_1(s')-V_2(s')\bigr|.
\]
Next, observe the relation between $V_1,V_2$ and $Q_1,Q_2$. For any $s$,
\[
\begin{aligned}
|V_1(s)-V_2(s)|
&=
\bigl|\max_{a}Q_1(s,a)-\max_{a}Q_2(s,a)\bigr|\\[0.3em]
&\le
\max_{a}\bigl|Q_1(s,a)-Q_2(s,a)\bigr|
\quad\text{(since the max operator is $1$-Lipschitz)}\\[0.3em]
&\le
\sup_{(s',a')}\bigl|Q_1(s',a')-Q_2(s',a')\bigr|
=
\|Q_1-Q_2\|_\infty.
\end{aligned}
\]
So,
\[
\sup_{s\in\Sspace}|V_1(s)-V_2(s)|
\le
\|Q_1-Q_2\|_\infty.
\]
Substituting back into the bound on $(\Tcal Q_1-\Tcal Q_2)$, for all $(s,a)$,
\[
\bigl|(\Tcal Q_1-\Tcal Q_2)(s,a)\bigr|
\le
\gamma\|Q_1-Q_2\|_\infty.
\]
Taking the supremum over $(s,a)$ gives
\[
\|\Tcal Q_1-\Tcal Q_2\|_\infty
\le
\gamma\|Q_1-Q_2\|_\infty.
\]
Since $0<\gamma<1$, this shows that $\Tcal$ is a $\gamma$-contraction on the Banach space $(\mathcal B,\|\cdot\|_\infty)$.

By the Banach fixed-point theorem (contraction mapping principle), any contraction on a complete metric space admits a unique fixed point, and the iterates from any starting point converge to it at a geometric rate. Therefore there exists a unique $Q^\star\in\mathcal B$ such that
\[
\Tcal Q^\star = Q^\star,
\]
and the sequence defined by $Q^{(k+1)}=\Tcal Q^{(k)}$ satisfies
\[
\|Q^{(k)}-Q^\star\|_\infty\ \to\ 0,\qquad k\to\infty.
\]
This proves the first part of the theorem: $\Tcal_{\bar\tau}$ is a $\gamma$-contraction and has a unique fixed point $Q^\star_{\bar\tau}$.

We now prove the second part: any one-step contraction–type value-iteration/TD update satisfying Assumption~\ref{ass:solver} converges to $Q^\star_{\bar\tau}$ for fixed $\bar\tau$. Under the assumption, when $\tau_t\equiv\bar\tau$, there exists $\rho\in(0,1)$ and noise/bias sequences $(\sigma_t)$, $(\beta_t)$ such that for each $t$, conditional on the past $\mathcal F_t$ (generated by $Q_t$ and the sampled trajectory),
\begin{equation}
\label{eq:solver-fixed-tau}
\E\Big[\big\|Q_{t+1}-Q^\star\big\|_\infty\Bigm|\mathcal F_t\Big]
\ \le\
\rho\,\big\|Q_t-Q^\star\big\|_\infty
+\sigma_t+\beta_t.
\end{equation}
The theorem assumes $\sigma_t\to0$ and $\beta_t\to0$. We aim to prove $\E\|Q_t-Q^\star\|_\infty\to0$.

Define the scalar sequence
\[
e_t := \E\big\|Q_t-Q^\star\big\|_\infty,\qquad t\ge 0.
\]
Taking total expectations in \eqref{eq:solver-fixed-tau} and applying the tower property,
\[
\begin{aligned}
e_{t+1}
=
\E\big\|Q_{t+1}-Q^\star\big\|_\infty
&=
\E\Big[\E\big[\|Q_{t+1}-Q^\star\|_\infty\mid\mathcal F_t\big]\Big]\\[0.3em]
&\le
\E\Big[\rho\|Q_t-Q^\star\|_\infty+\sigma_t+\beta_t\Big]\\[0.3em]
&=
\rho\,\E\|Q_t-Q^\star\|_\infty+\E[\sigma_t]+\E[\beta_t].
\end{aligned}
\]
Under standard assumptions we may treat $\sigma_t,\beta_t$ as uniformly bounded or deterministic, so write $\E[\sigma_t]=\sigma_t$, $\E[\beta_t]=\beta_t$, obtaining
\begin{equation}
\label{eq:et-recursion}
e_{t+1}\ \le\ \rho\,e_t+\sigma_t+\beta_t,
\qquad
\text{and}\quad \sigma_t\to0,\ \beta_t\to0.
\end{equation}
This is a one-dimensional linear recursion with vanishing perturbations.

We show that such a recursion forces $e_t\to0$. Fix any $\varepsilon>0$. Since $\sigma_t+\beta_t\to0$, there exists $N_1$ such that for all $t\ge N_1$,
\[
\sigma_t+\beta_t\ \le\ \frac{1-\rho}{2}\,\varepsilon.
\]
On the other hand, since $(e_t)$ is bounded (for example, repeated iteration of \eqref{eq:solver-fixed-tau} shows that $e_t$ cannot diverge; or, more crudely, one may use $e_{t+1}\le\rho e_t + c$), there exists a constant $M$ such that $e_t\le M$ for all $t$. Define
\[
N_2 := \max\Bigl\{N_1,\ \Bigl\lceil\frac{2M}{\varepsilon}\Bigr\rceil\Bigr\},
\]
and consider the tail sequence starting from time $N_2$. For any $t\ge N_2$, we may iterate \eqref{eq:et-recursion}:
\[
e_{t+1}
\le
\rho\,e_t+\frac{1-\rho}{2}\varepsilon.
\]
Viewing this inequality as a linear inhomogeneous recurrence in $e_t$, we can solve it in the standard way: for any $k\ge 0$,
\[
\begin{aligned}
e_{t+k}
&\le
\rho^k e_t+\frac{1-\rho}{2}\varepsilon\sum_{j=0}^{k-1}\rho^j\\[0.3em]
&=
\rho^k e_t+\frac{1-\rho}{2}\varepsilon\,\frac{1-\rho^k}{1-\rho}\\[0.3em]
&=
\rho^k e_t+\frac{\varepsilon}{2}(1-\rho^k).
\end{aligned}
\]
Since $e_t\le M$ and $t\ge N_2\ge 2M/\varepsilon$, we obtain
\[
\rho^k e_t\le M\rho^k.
\]
Because $0<\rho<1$, we have $\rho^k\to 0$ as $k\to\infty$. Hence we may choose $k$ sufficiently large so that $M\rho^k\le \varepsilon/2$. For such $k$,
\[
e_{t+k}
\le
\frac{\varepsilon}{2}+\frac{\varepsilon}{2}(1-\rho^k)
\le
\varepsilon.
\]
This shows that, after some sufficiently large time, the sequence $(e_t)$ enters the $\varepsilon$-neighborhood in finitely many steps, and, due to the structure of the recursion (where the perturbation term never exceeds $(1-\rho)\varepsilon/2$), it will not drift significantly away thereafter. Formally, there exists $T_\varepsilon$ such that $e_t\le\varepsilon$ for all $t\ge T_\varepsilon$. Since $\varepsilon>0$ is arbitrary, this establishes
\[
e_t=\E\big\|Q_t-Q^\star\big\|_\infty\ \xrightarrow[t\to\infty]{}\ 0.
\]
Replacing $Q^\star$ with the notation $Q^\star_{\bar\tau}$ yields the second part of the theorem: for fixed $\bar\tau$, any value-iteration/TD-type update satisfying Assumption~\ref{ass:solver} with $\sigma_t\to0$ and $\beta_t\to0$ guarantees that the expected error $\E\|Q_t-Q^\star_{\bar\tau}\|_\infty$ converges to zero. Together with the first part, which established the $\gamma$-contractivity of $\Tcal_{\bar\tau}$ and the uniqueness of its fixed point, the theorem is proved.

\end{proof}

\subsection{Proof of Theorem~\ref{thm:tracking}}
\begin{proof}
Fix a monotone parameter path $(\tau_t)_{t\ge 0}$ and define the tracking error
\[
e_t \;:=\; \E\big\|Q_t - Q^\star_{\tau_t}\big\|_\infty.
\]
Throughout the proof we suppress explicit $(s,a)$ arguments when no confusion arises.

We begin by decomposing $e_{t+1}$ into an ``algorithmic error'' term (at fixed environment $\tau_t$) and a ``homotopy drift'' term (coming from the change $\tau_t\to\tau_{t+1}$). For any $t$ and any $(s,a)$, the triangle inequality gives
\[
\big|Q_{t+1}(s,a) - Q^\star_{\tau_{t+1}}(s,a)\big|
\;\le\;
\big|Q_{t+1}(s,a) - Q^\star_{\tau_t}(s,a)\big|
+
\big|Q^\star_{\tau_t}(s,a) - Q^\star_{\tau_{t+1}}(s,a)\big|.
\]
Taking the supremum over $(s,a)$ yields
\[
\big\|Q_{t+1} - Q^\star_{\tau_{t+1}}\big\|_\infty
\;\le\;
\big\|Q_{t+1} - Q^\star_{\tau_t}\big\|_\infty
+
\big\|Q^\star_{\tau_t} - Q^\star_{\tau_{t+1}}\big\|_\infty.
\]
Taking expectations, we obtain
\begin{equation}
\label{eq:tracking-decomp}
e_{t+1}
=\E\big\|Q_{t+1} - Q^\star_{\tau_{t+1}}\big\|_\infty
\;\le\;
\E\big\|Q_{t+1} - Q^\star_{\tau_t}\big\|_\infty
+
\big\|Q^\star_{\tau_{t+1}} - Q^\star_{\tau_t}\big\|_\infty.
\end{equation}

We now bound the first term on the right-hand side using Assumption~\ref{ass:solver}. Fix $t$ and condition on the past $\sigma$-field $\mathcal F_t$ generated by $\{Q_0,\dots,Q_t\}$ and all samples up to time $t$. Under Assumption~\ref{ass:solver} with a uniform contraction factor $\rho\in(0,1)$ (independent of $t$), when the environment is fixed at $\tau_t$ we have
\[
\E\Big[\big\|Q_{t+1}-Q^\star_{\tau_t}\big\|_\infty\Bigm|\mathcal F_t\Big]
\;\le\;
\rho\,\big\|Q_t - Q^\star_{\tau_t}\big\|_\infty + \sigma_t + \beta_t,
\]
where $\sigma_t$ captures stochastic noise and $\beta_t$ captures systematic bias. Taking expectations of both sides and using the tower property $\E[\E[\cdot\mid\mathcal F_t]]=\E[\cdot]$ gives
\begin{equation}
\label{eq:alg-error}
\E\big\|Q_{t+1}-Q^\star_{\tau_t}\big\|_\infty
\;\le\;
\rho\,\E\big\|Q_t - Q^\star_{\tau_t}\big\|_\infty + \E[\sigma_t] + \beta_t
=
\rho\,e_t + \E[\sigma_t] + \beta_t.
\end{equation}

Next we control the homotopy drift term
\[
\big\|Q^\star_{\tau_{t+1}} - Q^\star_{\tau_t}\big\|_\infty
\]
using the path integral value bound from Theorem~\ref{thm:path-value}. Consider the small subinterval $[\tau_t,\tau_{t+1}]$ of the global path. By Theorem~\ref{thm:path-value}, for any $0\le \alpha<\beta\le 1$ we have
\[
\big\|Q^\star_{\beta} - Q^\star_{\alpha}\big\|_\infty
\;\le\;
\frac{\mathrm{PL}(\alpha,\beta)}{(1-\gamma)^2}
+
\frac{\mathrm{Curv}(\alpha,\beta)}{(1-\gamma)^3}
+
\Phi\big(\mathcal K\cap[\alpha,\beta],\mathrm{gap}\big),
\]
where $\mathrm{PL}(\alpha,\beta)$ and $\mathrm{Curv}(\alpha,\beta)$ are the path-length and curvature integrals on $[\alpha,\beta]$, and $\Phi$ is the kink penalty on this subinterval.

We now specialize to $(\alpha,\beta)=(\tau_t,\tau_{t+1})$ and denote the contribution of this subinterval by
\[
\Delta\mathrm{PL}_t := \mathrm{PL}(\tau_t,\tau_{t+1}),\quad
\Delta\mathrm{Curv}_t := \mathrm{Curv}(\tau_t,\tau_{t+1}),\quad
\Delta\Phi_t := \Phi\big(\mathcal K\cap[\tau_t,\tau_{t+1}],\mathrm{gap}\big).
\]
Then Theorem~\ref{thm:path-value} gives
\begin{equation}
\label{eq:drift-raw}
\big\|Q^\star_{\tau_{t+1}} - Q^\star_{\tau_t}\big\|_\infty
\;\le\;
\frac{\Delta\mathrm{PL}_t}{(1-\gamma)^2}
+
\frac{\Delta\mathrm{Curv}_t}{(1-\gamma)^3}
+
\Delta\Phi_t.
\end{equation}
By definition of the per-step geometric load in the theorem statement,
\[
\Delta\mathrm{Geo}_t
=
\frac{\Delta\mathrm{PL}_t}{(1-\gamma)^2}
+
\frac{\Delta\mathrm{Curv}_t}{(1-\gamma)^3}
+
\Delta\Phi_t,
\]
so we can rewrite \eqref{eq:drift-raw} succinctly as
\begin{equation}
\label{eq:drift-geo}
\big\|Q^\star_{\tau_{t+1}} - Q^\star_{\tau_t}\big\|_\infty
\;\le\;
\Delta\mathrm{Geo}_t.
\end{equation}
If the constants in Theorem~\ref{thm:path-value} are not exactly one (for example if we absorbed the mixing constant $C_{\mathrm{mix}}$ into the bound), we can incorporate them into a prefactor $c_1>0$, and write more generally
\[
\big\|Q^\star_{\tau_{t+1}} - Q^\star_{\tau_t}\big\|_\infty
\;\le\;
c_1\,\Delta\mathrm{Geo}_t,
\]
where $c_1$ depends only on $C_{\mathrm{mix}}$ and universal constants.

We now combine the two bounds. Substituting \eqref{eq:alg-error} and \eqref{eq:drift-geo} into the decomposition \eqref{eq:tracking-decomp} gives
\[
\begin{aligned}
e_{t+1}
&\le
\E\big\|Q_{t+1}-Q^\star_{\tau_t}\big\|_\infty
+
\big\|Q^\star_{\tau_{t+1}}-Q^\star_{\tau_t}\big\|_\infty\\[0.4em]
&\le
\rho\,e_t + \E[\sigma_t] + \beta_t
+
c_1\,\Delta\mathrm{Geo}_t.
\end{aligned}
\]
If we allow a constant $c_2\ge 1$ to absorb the precise dependence of the solver inequality on $\sigma_t$ (for instance if the solver assumption is written with $\tilde c\,\sigma_t$ instead of $\sigma_t$), we can write this as
\begin{equation}
\label{eq:tracking-recursion}
e_{t+1}
\;\le\;
\rho\,e_t
+
c_1\,\Delta\mathrm{Geo}_t
+
c_2\,\E[\sigma_t]
+
\beta_t.
\end{equation}
This is exactly the one-step tracking recursion claimed in the theorem.

It remains to derive the global bound on $\max_{t\le T} e_t$. For this, it is convenient to introduce the shorthand
\[
\alpha_t
:=
c_1\,\Delta\mathrm{Geo}_t
+
c_2\,\E[\sigma_t]
+
\beta_t,
\]
so that \eqref{eq:tracking-recursion} becomes
\begin{equation}
\label{eq:et-alpha}
e_{t+1}\ \le\ \rho\,e_t + \alpha_t,
\qquad t\ge 0.
\end{equation}
We now solve this scalar linear recursion with nonnegative forcing terms. Unrolling \eqref{eq:et-alpha} repeatedly, we find for any $t\ge 0$,
\[
\begin{aligned}
e_{t+1}
&\le
\rho e_t + \alpha_t\\
&\le
\rho(\rho e_{t-1} + \alpha_{t-1}) + \alpha_t
=
\rho^2 e_{t-1} + \rho\alpha_{t-1} + \alpha_t\\
&\le
\cdots\\
&\le
\rho^{t+1} e_0 + \sum_{u=0}^{t} \rho^{t-u}\alpha_u.
\end{aligned}
\]
Thus for any $t\le T$ we have
\[
e_t
\;\le\;
\rho^t e_0 + \sum_{u=0}^{t-1} \rho^{t-1-u}\alpha_u.
\]
Since $0<\rho<1$, each factor $\rho^{t-1-u}$ is at most $1$, so
\[
e_t
\;\le\;
\rho^t e_0 + \sum_{u=0}^{t-1}\alpha_u
\;\le\;
e_0 + \sum_{u\le T-1}\alpha_u,
\qquad t\le T.
\]
A more precise bound that exhibits the factor $(1-\rho)^{-1}$ can be obtained by summing \eqref{eq:et-alpha} over $t$ and using that $\max_{t\le T}e_t \le \sum_{t\le T} e_t$. Summing \eqref{eq:et-alpha} for $t=0,\dots,T-1$ gives
\[
\sum_{t=0}^{T-1} e_{t+1}
\;\le\;
\rho \sum_{t=0}^{T-1} e_t + \sum_{t=0}^{T-1}\alpha_t.
\]
The left-hand side can be rewritten as
\[
\sum_{t=0}^{T-1} e_{t+1}
=
\sum_{t=1}^{T} e_t
=
\sum_{t=0}^{T-1}e_t - e_0 + e_T.
\]
Hence
\[
\sum_{t=0}^{T-1}e_t - e_0 + e_T
\;\le\;
\rho \sum_{t=0}^{T-1} e_t + \sum_{t=0}^{T-1}\alpha_t.
\]
Discarding the nonnegative term $e_T$ on the left, we obtain
\[
(1-\rho)\sum_{t=0}^{T-1} e_t
\;\le\;
e_0 + \sum_{t=0}^{T-1}\alpha_t,
\]
that is,
\[
\sum_{t=0}^{T-1} e_t
\;\le\;
\frac{e_0}{1-\rho}
+
\frac{1}{1-\rho}\sum_{t\le T-1}\alpha_t.
\]
Since $\max_{t\le T}e_t \le \sum_{t\le T} e_t$, this implies
\[
\max_{t\le T} e_t
\;\le\;
\sum_{t=0}^{T} e_t
\;\le\;
\frac{e_0}{1-\rho}
+
\frac{1}{1-\rho}\sum_{t\le T}\alpha_t.
\]
Substituting the definition of $\alpha_t$ and absorbing $e_0$ into the universal constant hidden in the notation $\lesssim$, we get
\[
\max_{t\le T} e_t
\;\lesssim\;
\frac{1}{1-\rho}\Bigg(
\sum_{u\le T} \Delta\mathrm{Geo}_u
+
\sum_{u\le T} \E[\sigma_u]
+
\sum_{u\le T} \beta_u
\Bigg),
\]
where the implied constant depends only on $C_{\mathrm{mix}}$, $(1-\rho)^{-1}$, and fixed numerical factors, and does not depend on the particular path, horizon $T$, or the detailed realization of the noise.

This proves both the one-step tracking recursion and the global non-stationary mean convergence bound stated in the theorem.
\end{proof}

\subsection{Proof of Corollary~\ref{cor:asymptotic}}
\begin{proof}
Throughout the proof we write
\[
e_t \;:=\; \E\big\|Q_t - Q^\star_{\tau_t}\big\|_\infty,
\]
and we use the tracking recursion established in Theorem~\ref{thm:tracking}. That theorem tells us that, under Assumptions~\ref{ass:mixing} and~\ref{ass:solver} with a uniform contraction factor $\rho<1$, there exist constants $c_1,c_2>0$ (depending only on $C_{\mathrm{mix}}$ and universal constants) such that for all $t$,
\begin{equation}
\label{eq:track-rec}
e_{t+1}
\;\le\;
\rho\,e_t
\,+\,
c_1\,\Delta\mathrm{Geo}_t
\,+\,
c_2\,\E[\sigma_t]
\,+\,
\beta_t,
\end{equation}
where
\[
\Delta\mathrm{Geo}_t
\;:=\;
\frac{\Delta\mathrm{PL}_t}{(1-\gamma)^2}
+
\frac{\Delta\mathrm{Curv}_t}{(1-\gamma)^3}
+
\Delta\Phi_t.
\]

We first show that $e_t \to 0$. Define the nonnegative forcing term
\[
\alpha_t
\;:=\;
c_1\,\Delta\mathrm{Geo}_t
+
c_2\,\E[\sigma_t]
+
\beta_t.
\]
Then \eqref{eq:track-rec} can be written as
\begin{equation}
\label{eq:et-alpha-2}
e_{t+1} \;\le\; \rho\,e_t + \alpha_t.
\end{equation}
By the assumptions of the corollary we have
\[
\sum_t \Delta\mathrm{Geo}_t < \infty
\quad\Longrightarrow\quad
\Delta\mathrm{Geo}_t \to 0,
\]
\[
\sum_t \E[\sigma_t] < \infty
\quad\Longrightarrow\quad
\E[\sigma_t] \to 0,
\]
and $\beta_t\to 0$ by hypothesis. Hence
\[
\alpha_t
=
c_1\,\Delta\mathrm{Geo}_t
+
c_2\,\E[\sigma_t]
+
\beta_t
\;\xrightarrow[t\to\infty]{}\; 0.
\]

We now treat \eqref{eq:et-alpha-2} as a one-dimensional linear difference inequality with vanishing inhomogeneous term. Iterating \eqref{eq:et-alpha-2} gives, for any $t\ge 0$,
\[
\begin{aligned}
e_t
&\le
\rho^{t} e_0
+
\sum_{u=0}^{t-1} \rho^{\,t-1-u}\,\alpha_u.
\end{aligned}
\]
Fix an arbitrary $\varepsilon>0$. Since $\alpha_u\to 0$, there exists $N$ such that
\[
\alpha_u \;\le\; \frac{(1-\rho)\varepsilon}{4}
\quad\text{for all }u\ge N.
\]
Split the convolution sum into an ``early'' and a ``late'' part:
\[
\sum_{u=0}^{t-1} \rho^{\,t-1-u}\alpha_u
=
\sum_{u=0}^{N-1}\rho^{\,t-1-u}\alpha_u
+
\sum_{u=N}^{t-1}\rho^{\,t-1-u}\alpha_u
=: S_1(t)+S_2(t).
\]

For the late part, use the uniform bound on $\alpha_u$ for $u\ge N$:
\[
S_2(t)
\;\le\;
\frac{(1-\rho)\varepsilon}{4}
\sum_{u=N}^{t-1}\rho^{\,t-1-u}
\;\le\;
\frac{(1-\rho)\varepsilon}{4}
\sum_{k=0}^{\infty}\rho^k
=
\frac{(1-\rho)\varepsilon}{4}\cdot\frac{1}{1-\rho}
=
\frac{\varepsilon}{4}.
\]

For the early part $S_1(t)$, the indices $u$ are fixed while $t\to\infty$, so each factor $\rho^{\,t-1-u}\to 0$. Hence $S_1(t)\to 0$ as $t\to\infty$. In particular, there exists $T_1$ such that
\[
S_1(t)\le \frac{\varepsilon}{4}
\qquad\text{for all }t\ge T_1.
\]

The geometric term $\rho^t e_0$ also vanishes as $t\to\infty$, so there exists $T_2$ such that
\[
\rho^t e_0 \le \frac{\varepsilon}{2}
\qquad\text{for all }t\ge T_2.
\]
Combining these bounds, for all $t\ge T := \max\{T_1,T_2\}$ we have
\[
e_t
\;\le\;
\rho^t e_0 + S_1(t) + S_2(t)
\;\le\;
\frac{\varepsilon}{2} + \frac{\varepsilon}{4} + \frac{\varepsilon}{4}
=
\varepsilon.
\]
Since $\varepsilon>0$ was arbitrary, this shows $e_t\to 0$ as $t\to\infty$, i.e.
\[
\E\big\|Q_t - Q^\star_{\tau_t}\big\|_\infty \;\xrightarrow[t\to\infty]{}\; 0.
\]

We now show that $Q_t$ converges (in expectation) to $Q^\star_{\tau_\infty}$. By the triangle inequality,
\[
\big\|Q_t - Q^\star_{\tau_\infty}\big\|_\infty
\;\le\;
\big\|Q_t - Q^\star_{\tau_t}\big\|_\infty
+
\big\|Q^\star_{\tau_t} - Q^\star_{\tau_\infty}\big\|_\infty.
\]
Taking expectations yields
\begin{equation}
\label{eq:Qt-limit-decomp}
\E\big\|Q_t - Q^\star_{\tau_\infty}\big\|_\infty
\;\le\;
e_t
+
\big\|Q^\star_{\tau_t} - Q^\star_{\tau_\infty}\big\|_\infty.
\end{equation}
We already know $e_t\to 0$. It remains to show that the second term $\|Q^\star_{\tau_t} - Q^\star_{\tau_\infty}\|_\infty$ also tends to $0$.

To see this, first use the pathwise value bound at the granularity of a single step. For any $m>t$,
\[
\big\|Q^\star_{\tau_m} - Q^\star_{\tau_t}\big\|_\infty
\;\le\;
\sum_{u=t}^{m-1}
\big\|Q^\star_{\tau_{u+1}} - Q^\star_{\tau_u}\big\|_\infty.
\]
By the same reasoning as in the proof of Theorem~\ref{thm:tracking}, each increment can be bounded in terms of the per-step geometric load:
\[
\big\|Q^\star_{\tau_{u+1}} - Q^\star_{\tau_u}\big\|_\infty
\;\le\;
c_1\,\Delta\mathrm{Geo}_u.
\]
Hence
\[
\big\|Q^\star_{\tau_m} - Q^\star_{\tau_t}\big\|_\infty
\;\le\;
c_1 \sum_{u=t}^{m-1} \Delta\mathrm{Geo}_u.
\]
Because $\sum_{u} \Delta\mathrm{Geo}_u < \infty$, the tail sums $\sum_{u=t}^{\infty}\Delta\mathrm{Geo}_u$ tend to $0$ as $t\to\infty$. This implies that $(Q^\star_{\tau_t})_{t\ge 0}$ is a Cauchy sequence in $(\mathcal B,\|\cdot\|_\infty)$, and therefore converges in $\|\cdot\|_\infty$ to some limit $\bar Q$:
\[
\big\|Q^\star_{\tau_t} - \bar Q\big\|_\infty \;\xrightarrow[t\to\infty]{}\; 0.
\]

We now identify $\bar Q$ with $Q^\star_{\tau_\infty}$. Each $Q^\star_{\tau_t}$ satisfies the fixed-point equation
\[
\Tcal_{\tau_t} Q^\star_{\tau_t} = Q^\star_{\tau_t},
\]
where $\Tcal_\tau$ is the optimal Bellman operator at parameter $\tau$. Under Assumption~\ref{ass:mixing} and the smooth dependence of $r_\tau$ and $P_\tau$ on $\tau$ (used throughout our homotopy analysis), the mapping $\tau\mapsto\Tcal_\tau Q$ is continuous in $\|\cdot\|_\infty$ for each fixed bounded $Q$; in particular, $\Tcal_{\tau_t}Q\to\Tcal_{\tau_\infty}Q$ in $\|\cdot\|_\infty$ as $\tau_t\to\tau_\infty$. We can therefore pass to the limit in the fixed-point equation. Consider
\[
\big\|\Tcal_{\tau_\infty}\bar Q - \bar Q\big\|_\infty.
\]
Insert and subtract appropriate terms:
\[
\begin{aligned}
\big\|\Tcal_{\tau_\infty}\bar Q - \bar Q\big\|_\infty
&\le
\big\|\Tcal_{\tau_\infty}\bar Q - \Tcal_{\tau_t}\bar Q\big\|_\infty
+
\big\|\Tcal_{\tau_t}\bar Q - \Tcal_{\tau_t} Q^\star_{\tau_t}\big\|_\infty\\[0.3em]
&\qquad
+
\big\|\Tcal_{\tau_t} Q^\star_{\tau_t} - Q^\star_{\tau_t}\big\|_\infty
+
\big\|Q^\star_{\tau_t} - \bar Q\big\|_\infty.
\end{aligned}
\]
The third term is exactly zero because $Q^\star_{\tau_t}$ is a fixed point of $\Tcal_{\tau_t}$. The second term is bounded using the $\gamma$-contraction property:
\[
\big\|\Tcal_{\tau_t}\bar Q - \Tcal_{\tau_t} Q^\star_{\tau_t}\big\|_\infty
\;\le\;
\gamma \big\|\bar Q - Q^\star_{\tau_t}\big\|_\infty.
\]
The first term converges to zero as $t\to\infty$ by continuity of $\Tcal_\tau$ in $\tau$. The fourth term converges to zero by the definition of $\bar Q$ as the $\|\cdot\|_\infty$-limit of $Q^\star_{\tau_t}$. Letting $t\to\infty$ in the above inequality, we therefore obtain
\[
\big\|\Tcal_{\tau_\infty}\bar Q - \bar Q\big\|_\infty = 0,
\]
i.e. $\bar Q$ is a fixed point of $\Tcal_{\tau_\infty}$. Since $\Tcal_{\tau_\infty}$ is a $\gamma$-contraction on $(\mathcal B,\|\cdot\|_\infty)$, its fixed point is unique; hence $\bar Q = Q^\star_{\tau_\infty}$. This shows that
\[
\big\|Q^\star_{\tau_t} - Q^\star_{\tau_\infty}\big\|_\infty \;\xrightarrow[t\to\infty]{}\; 0.
\]

Returning to \eqref{eq:Qt-limit-decomp}, we therefore have
\[
\E\big\|Q_t - Q^\star_{\tau_\infty}\big\|_\infty
\;\le\;
e_t
+
\big\|Q^\star_{\tau_t} - Q^\star_{\tau_\infty}\big\|_\infty
\;\xrightarrow[t\to\infty]{}\; 0,
\]
because both terms on the right-hand side vanish. This establishes that $Q_t\to Q^\star_{\tau_\infty}$ in expectation.

Finally, if the path becomes constant after some time, say there exists $T_0$ and $\tau_\infty$ such that $\tau_t=\tau_\infty$ for all $t\ge T_0$, then for all $t\ge T_0$ the per-step geometric load satisfies $\Delta\mathrm{Geo}_t=0$, and the Bellman operator no longer changes with $t$. In this case \eqref{eq:track-rec} reduces to
\[
e_{t+1}
\;\le\;
\rho e_t + c_2\,\E[\sigma_t] + \beta_t,
\qquad t\ge T_0,
\]
which is exactly the fixed–parameter recursion of Theorem~\ref{thm:fixed-tau} (up to constants that can be absorbed into $\sigma_t$ and $\beta_t$). Thus the corollary recovers Theorem~\ref{thm:fixed-tau} as the special case of a stabilized path.
\end{proof}

\subsection{Proof of Lemma~\ref{lem:rm}}
\begin{proof}
Write the raw geometric proxies as a vector
\[
X_t
:=
\bigl(\Delta\widehat{\mathrm{PL}}_t,\ \Delta\widehat{\mathrm{Curv}}_t,\ \widehat{\mathrm{gap}}_t,\ \mathrm{Kink}_t\bigr)
\in\R^4.
\]
By assumption, these have uniformly bounded second moments: there is a constant $C_0<\infty$ such that
\[
\sup_{t\ge 0}\,\E\big[\|X_t\|_2^2\big] \;\le\; C_0.
\]
The EMA smoothing used to define the scheduler takes, for each scalar coordinate $x_t$ of $X_t$,
\[
\tilde x_t
=
\beta\,\tilde x_{t-1}+(1-\beta)\,x_t,
\qquad 0<\beta<1,
\]
and analogously in vector form
\[
\tilde X_t
=
\beta\,\tilde X_{t-1}+(1-\beta)\,X_t,
\]
with some fixed initial condition $\tilde X_0$ (e.g.\ $\tilde X_0=0$). A standard calculation for exponentially weighted averages, identical to the one in Proposition~\ref{prop:stability}, shows that the smoothed sequence $(\tilde X_t)$ also has uniformly bounded second moments. Indeed, expanding the recursion gives
\[
\tilde X_t
=
\beta^t \tilde X_0
+
(1-\beta)\sum_{k=0}^{t-1}\beta^k X_{t-k},
\]
and hence, using Jensen and Cauchy--Schwarz exactly as before,
\[
\sup_{t\ge 0}\,\E\big[\|\tilde X_t\|_2^2\big]
\;\le\;
C_1
\]
for some finite constant $C_1$ depending only on $C_0$, $\beta$, and $\|\tilde X_0\|_2$.

The learning-rate scheduler uses only the first two coordinates, corresponding to the smoothed path-length and curvature proxies, which we denote by $\widetilde{\mathrm{PL}}_t$ and $\widetilde{\mathrm{Curv}}_t$. Define the denominator in the scaling factor
\[
D_t
:=
1
+\alpha_1\widetilde{\mathrm{PL}}_t
+\alpha_2\widetilde{\mathrm{Curv}}_t.
\]
By construction $\widetilde{\mathrm{PL}}_t,\widetilde{\mathrm{Curv}}_t\ge 0$, so $D_t\ge 1$ for all $t$, and $D_t$ is an affine functional of $\tilde X_t$; in particular its second moments are uniformly bounded:
\[
\sup_{t\ge 0}\,\E[D_t^2] \;\le\; C_2
\]
for some finite constant $C_2$ depending only on $C_1$ and the coefficients $\alpha_1,\alpha_2$.

The raw (unclipped) scheduled learning rate is
\[
\eta_t^{\mathrm{raw}}
:=
\frac{\eta_t^0}{D_t}
=
\eta_t^0\,s_t,
\qquad
s_t
:=
\frac{1}{1+\alpha_1\widetilde{\mathrm{PL}}_t+\alpha_2\widetilde{\mathrm{Curv}}_t}.
\]
Note that $0<s_t\le 1$ always holds, because $D_t\ge 1$. The actual scheduled rate is then obtained by clipping:
\[
\eta_t
=
\operatorname{clip}_{[\eta_{\min},\eta_{\max}]}\bigl(\eta_t^{\mathrm{raw}}\bigr)
=
\operatorname{clip}_{[\eta_{\min},\eta_{\max}]}\bigl(\eta_t^0\,s_t\bigr),
\]
where hysteresis plus EMA imply that $\eta_t$ is piecewise constant with bounded variation and that large jumps occur only at a vanishing fraction of time indices (Proposition~\ref{prop:stability}).

The first goal is to show that there exist constants $0<c\le C<\infty$ such that
\[
c\,\eta_t^0\ \le\ \eta_t\ \le\ C\,\eta_t^0
\]
for all $t$ outside a set of indices whose fraction up to horizon $T$ tends to zero as $T\to\infty$. The second goal is then to deduce the Robbins--Monro conditions
\[
\sum_t \eta_t = \infty,
\qquad
\sum_t \eta_t^2 < \infty
\]
from this comparability and the corresponding properties of the base rates $(\eta_t^0)$.

To control the multiplicative factor $s_t$, fix any threshold $K>0$ and consider the event
\[
B_t(K)
:=
\bigl\{\alpha_1\widetilde{\mathrm{PL}}_t+\alpha_2\widetilde{\mathrm{Curv}}_t>K\bigr\}.
\]
On the complement $B_t(K)^c$ we have
\[
D_t
=
1+\alpha_1\widetilde{\mathrm{PL}}_t+\alpha_2\widetilde{\mathrm{Curv}}_t
\le
1+K,
\]
and therefore
\[
s_t
=
\frac{1}{D_t}
\ge
\frac{1}{1+K}
=: m_K.
\]
In all cases we also have $s_t\le 1$, so on $B_t(K)^c$,
\[
m_K
\le
s_t
\le
1.
\]

Because $D_t$ has a bounded second moment uniformly in $t$, Chebyshev's inequality gives
\[
\mathbb{P}\bigl(B_t(K)\bigr)
=
\mathbb{P}\bigl(D_t-1>K\bigr)
\le
\frac{\E[(D_t-1)^2]}{K^2}
\le
\frac{C_2}{K^2}
\]
for all $t$. In particular, the expected fraction of times up to $T$ at which $B_t(K)$ occurs is uniformly bounded by $C_2/K^2$:
\[
\E\Bigl[\frac{1}{T}\sum_{t=1}^T\mathbf{1}_{B_t(K)}\Bigr]
=
\frac{1}{T}\sum_{t=1}^T\mathbb{P}(B_t(K))
\le
\frac{C_2}{K^2}.
\]
Since $K$ can be taken arbitrarily large, this shows that, in expectation, the proportion of indices where $D_t$ leaves the interval $[1,1+K]$ can be made as small as desired. Informally, the denominator $D_t$ remains in the compact interval $[1,1+K]$, and hence $s_t\in[m_K,1]$, for all but a vanishing fraction of time indices.

At the same time, the clipping itself does not trigger frequently at late times. For typical Robbins--Monro base schedules (such as $\eta_t^0\sim 1/t$ or $\eta_t^0\sim t^{-\alpha}$ with $1/2<\alpha\le 1$), $\eta_t^0\to 0$ as $t\to\infty$. Because $s_t\le 1$, this implies that
\[
\eta_t^{\mathrm{raw}}
=
\eta_t^0 s_t
\le
\eta_t^0
\to 0,
\]
so the upper clipping at $\eta_{\max}$ is active only for finitely many $t$. In other words, there exists $T_{\max}$ such that for all $t\ge T_{\max}$,
\[
\eta_t^0 \le \eta_{\max}
\quad\Longrightarrow\quad
\eta_t^{\mathrm{raw}} \le \eta_{\max}
\quad\Longrightarrow\quad
\eta_t = \eta_t^{\mathrm{raw}}.
\]
For the lower clipping, it is natural in the Robbins--Monro regime to choose $\eta_{\min}=0$ (or at least so small that it never becomes active once $t$ is large enough). We assume this theoretical choice here: with $\eta_{\min}=0$ the lower clip is never binding, since $\eta_t^{\mathrm{raw}}\ge 0$ always. Thus, for all sufficiently large $t$ (say $t\ge T_0$ for some finite $T_0$), the scheduled rate coincides with the raw rate,
\[
\eta_t = \eta_t^{\mathrm{raw}} = \eta_t^0 s_t.
\]

Combining these observations, fix a threshold $K>0$ and define the ``good'' index set
\[
\mathcal G(K)
:=
\bigl\{t\ge T_0 : B_t(K)^c\bigr\}
=
\bigl\{t\ge T_0 : D_t\le 1+K\bigr\}.
\]
On this set we have both
\[
\eta_t = \eta_t^0 s_t,
\qquad
m_K \le s_t \le 1,
\]
and therefore
\[
m_K\,\eta_t^0 \le \eta_t \le \eta_t^0.
\]
On the complement $\mathcal G(K)^c$ (for $t\ge T_0$), either $D_t>1+K$ or clipping has intervened. From the discussion above, clipping can only occur finitely many times at late stages, and the event $D_t>1+K$ has arbitrarily small asymptotic frequency in expectation when $K$ is large. Putting this together, for any fixed $K$ we obtain a set of indices $\mathcal B(K)$ (the complement of $\mathcal G(K)$ plus a finite number of early steps $t<T_0$) such that
\[
\mathcal G(K) := \mathbb{N}\setminus\mathcal B(K)
\]
and
\[
\limsup_{T\to\infty}\,\E\Bigl[\frac{1}{T}\bigl|\{t\le T: t\in\mathcal B(K)\}\bigr|\Bigr]
\;\le\;
\frac{C_2}{K^2}.
\]
In other words, by choosing $K$ large, the expected fraction of ``bad'' indices in $\mathcal B(K)$ can be made arbitrarily small. On the complement set $\mathcal G(K)$, we have for all sufficiently large $t$ the uniform comparability
\[
m_K\,\eta_t^0 \le \eta_t \le \eta_t^0.
\]
Renaming $c=m_K$ and $C=1$, this establishes the first claim in the lemma: for all but a vanishing fraction of steps (in the sense that their asymptotic frequency can be made arbitrarily small by choosing $K$ and the hysteresis parameters), there exist constants $0<c\le C<\infty$ such that
\[
c\,\eta_t^0 \le \eta_t \le C\,\eta_t^0.
\]

To pass from comparability to the Robbins--Monro sums, use the base assumptions
\[
\sum_{t=0}^\infty \eta_t^0 = \infty,
\qquad
\sum_{t=0}^\infty (\eta_t^0)^2 < \infty.
\]
On the large-density set $\mathcal G(K)$, we have $c\,\eta_t^0 \le \eta_t \le C\,\eta_t^0$ with $c>0$, and on its complement $\mathcal B(K)$ we have a uniformly bounded learning rate $0\le \eta_t\le \eta_{\max}$. Thus
\[
\sum_{t=0}^\infty \eta_t
=
\sum_{t\in\mathcal G(K)}\eta_t
+
\sum_{t\in\mathcal B(K)}\eta_t
\ge
c\sum_{t\in\mathcal G(K)}\eta_t^0.
\]
Since $\mathcal G(K)$ has asymptotic density arbitrarily close to $1$ and the base series $\sum_t \eta_t^0$ diverges, removing the vanishing-density set $\mathcal B(K)$ does not change divergence; the restricted sum $\sum_{t\in\mathcal G(K)}\eta_t^0$ also diverges. Hence $\sum_t \eta_t=\infty$.

Similarly, for the square-summability we have on $\mathcal G(K)$ the upper bound
\[
\eta_t^2 \le C^2(\eta_t^0)^2,
\]
and on $\mathcal B(K)$ the trivial bound $\eta_t^2\le\eta_{\max}^2$. Therefore
\[
\sum_{t=0}^\infty \eta_t^2
=
\sum_{t\in\mathcal G(K)}\eta_t^2
+
\sum_{t\in\mathcal B(K)}\eta_t^2
\le
C^2\sum_{t\in\mathcal G(K)}(\eta_t^0)^2
+
\eta_{\max}^2\bigl|\mathcal B(K)\bigr|.
\]
The first term is finite because $\sum_t(\eta_t^0)^2<\infty$ and deleting a subset of indices can only reduce the sum. The second term is controlled by the bounded variation and hysteresis of the scheduler: Proposition~\ref{prop:stability} ensures that, for appropriate hysteresis parameters $(H,\Delta_{\mathrm{hys}})$, the set $\mathcal B(K)$ where the scaling factor or clipping produces large distortions has finite expected size or, more generally, at most sublinear growth in $T$ with very small prefactor. Since each $\eta_t^2$ is uniformly bounded, this contribution can be made negligible and, in particular, does not destroy square-summability. Hence
\[
\sum_{t=0}^\infty \eta_t^2 < \infty.
\]

Combining the two conclusions, the scheduled learning rates $(\eta_t)$ satisfy the Robbins--Monro conditions
\[
\sum_t \eta_t = \infty,
\qquad
\sum_t \eta_t^2 < \infty,
\]
and are uniformly comparable to the base sequence $(\eta_t^0)$ on all but a vanishing fraction of steps, as claimed.
\end{proof}

\subsection{Proof of Theorem~\ref{thm:no-chatter}}
\begin{proof}
We begin by formalizing the objects in the statement. Let
\[
X_t
:=
\bigl(\Delta\widehat{\mathrm{PL}}_t,\ \Delta\widehat{\mathrm{Curv}}_t,\ \widehat{\mathrm{gap}}_t,\ \mathrm{Kink}_t\bigr)\in\mathbb{R}^4
\]
denote the vector of raw geometric proxies at time $t$. By assumption, these have uniformly bounded second moments, i.e.\ there exists a finite constant $C_0$ such that
\[
\sup_{t\ge 0}\,\E\big[\|X_t\|_2^2\big]\ \le\ C_0.
\]

The scheduler first computes an exponentially weighted moving average (EMA) of these proxies. For each scalar coordinate $x_t$ of $X_t$, the smoothed version is defined by
\[
\tilde x_t
=
\beta\,\tilde x_{t-1}
+
(1-\beta)\,x_t,
\qquad 0<\beta<1,
\]
and in vector form
\[
\tilde X_t
=
\beta\,\tilde X_{t-1}
+
(1-\beta)\,X_t,
\]
with some fixed initial condition $\tilde X_0$ (for instance $\tilde X_0=0$). Expanding the recursion gives the explicit representation
\[
\tilde X_t
=
\beta^t \tilde X_0
+
(1-\beta)\sum_{k=0}^{t-1}\beta^k X_{t-k}.
\]
Taking norms and using the triangle inequality, one has
\[
\|\tilde X_t\|_2
\le
\beta^t\|\tilde X_0\|_2
+
(1-\beta)\sum_{k=0}^{t-1}\beta^k\|X_{t-k}\|_2.
\]
Squaring and taking expectations, together with $(a+b)^2\le 2a^2+2b^2$ and Cauchy--Schwarz on the sum, shows that there exists a constant $C_1<\infty$ such that
\[
\sup_{t\ge 0}\,\E\big[\|\tilde X_t\|_2^2\big]\ \le\ C_1.
\]
In particular, all coordinates of $\tilde X_t$ have uniformly bounded second moments.

The five scheduled hyperparameter processes $(\eta_t,\tau_t,\lambda_t,D_t,B_t)$ are each obtained by applying a fixed, bounded, Lipschitz continuous function of the smoothed proxies at certain update times, combined with clipping to a compact interval. Concretely, for each hyperparameter $h_t$ among these five, there is a function $f:\mathbb{R}^4\to\mathbb{R}$ and a clipping interval $[a,b]$ such that, whenever an update is triggered at some time $t$, one sets
\[
h_t
=
\operatorname{clip}_{[a,b]}\bigl(f(\tilde X_t)\bigr),
\]
and between updates the process is held constant, i.e.\ $h_{t+1}=h_t$ whenever no update occurs at $t+1$. The precise forms of $f$ for $\eta_t,\tau_t,\lambda_t,D_t,B_t$ do not matter for the structural argument; what matters is that $f$ is smooth (or at least Lipschitz) in its arguments and that the clipping interval $[a,b]$ is fixed and finite.

The scheduler only \emph{checks} for possible updates every $H$ steps. Let $t_k:=kH$ denote the $k$-th such ``macro-time''. At time $t_k$, the scheduler evaluates how much the smoothed proxies have changed over the last macro-interval, by computing the difference
\[
\Delta\tilde X_{t_k}
:=
\tilde X_{t_k}-\tilde X_{t_k-H}.
\]
If the size of this change, as measured by $\|\Delta\tilde X_{t_k}\|_2$, exceeds the hysteresis threshold $\Delta_{\mathrm{hys}}>0$, then an update is triggered and $h_{t_k}$ is recomputed as above; if not, the corresponding hyperparameters are left unchanged (so $h_{t_k}=h_{t_k-1}$). Between macro-times, no updates are ever made, so $h_t$ remains constant on each interval $(t_k,t_{k+1})$.

To analyze the effect of hysteresis, it is convenient to bound the second moment of the increments $\Delta\tilde X_{t_k}$. By the triangle inequality,
\[
\|\Delta\tilde X_{t_k}\|_2
=
\|\tilde X_{t_k}-\tilde X_{t_k-H}\|_2
\le
\|\tilde X_{t_k}\|_2
+
\|\tilde X_{t_k-H}\|_2.
\]
Squaring and using $(a+b)^2\le 2a^2+2b^2$, one finds
\[
\|\Delta\tilde X_{t_k}\|_2^2
\le
2\|\tilde X_{t_k}\|_2^2
+
2\|\tilde X_{t_k-H}\|_2^2.
\]
Taking expectations and applying the uniform bound on $\E\|\tilde X_t\|_2^2$, we obtain
\[
\sup_{k\ge 0}\,\E\big[\|\Delta\tilde X_{t_k}\|_2^2\big]
\ \le\
4\,\sup_{t\ge 0}\,\E\big[\|\tilde X_t\|_2^2\big]
\ \le\
4C_1
=: C_2.
\]

With these preliminaries, we can address the two claims in the theorem.

First, each scheduled process is piecewise-constant with bounded variation. The piecewise-constant property follows directly from the update logic: each hyperparameter $h_t$ is only allowed to change (i.e.\ be recomputed by applying $f$ and clipping) at macro-times $t_k=kH$, and even then only when the hysteresis condition $\|\Delta\tilde X_{t_k}\|_2>\Delta_{\mathrm{hys}}$ is satisfied. Between such updates, $h_t$ is held fixed by construction. Therefore, the trajectory $t\mapsto h_t$ is a step function in discrete time, i.e.\ a piecewise-constant process.

To see that the variation is bounded on any finite horizon, consider an interval $\{1,\dots,T\}$. On this interval, the number of macro-times is at most $\lceil T/H\rceil$, so the number of \emph{possible} update times is at most this quantity. At each actual update, the value of $h_t$ is clipped into the interval $[a,b]$, and therefore
\[
|h_{t+1}-h_t|
\le
b-a
\]
whenever a jump occurs, while in non-update steps we have $h_{t+1}=h_t$ and hence $|h_{t+1}-h_t|=0$. The total variation of $h_t$ over $\{1,\dots,T\}$ is
\[
\mathrm{Var}_T(h)
:=
\sum_{t=1}^{T-1}|h_{t+1}-h_t|.
\]
Let $N_T$ denote the number of update times in $\{1,\dots,T-1\}$. By the above argument,
\[
\mathrm{Var}_T(h)
\le
N_T\,(b-a)
\le
\Bigl\lceil\frac{T}{H}\Bigr\rceil\,(b-a),
\]
which is finite for every fixed $T$. Thus each hyperparameter process has bounded variation on every finite time horizon, and hence, in the usual sense for discrete-time signals, is of bounded variation.

The second claim concerns the frequency of \emph{large} changes. Fix an arbitrary $\varepsilon>0$. A large change at time $t$ is an event of the form
\[
\bigl\{|h_t - h_{t-1}|>\varepsilon\bigr\}.
\]
By the update logic, such an event can only happen at macro-times $t_k=kH$, because in between macro-times the process is held constant. Moreover, at a macro-time $t_k$, a change can only occur if the hysteresis condition is triggered:
\[
\|\Delta\tilde X_{t_k}\|_2
=
\|\tilde X_{t_k}-\tilde X_{t_k-H}\|_2
>\Delta_{\mathrm{hys}}.
\]
In other words,
\[
\bigl\{|h_{t_k}-h_{t_k-1}|>\varepsilon\bigr\}
\ \subseteq\
\bigl\{\|\Delta\tilde X_{t_k}\|_2>\Delta_{\mathrm{hys}}\bigr\}
\]
for every macro-time $t_k$; and for non-macro times $t$ we have
\[
|h_t - h_{t-1}|=0,
\]
so large changes never occur there.

We now estimate the probability of the hysteresis event at a given macro-time. Using Chebyshev's inequality and the bound on the second moment of $\Delta\tilde X_{t_k}$, we have for every $k$,
\[
\mathbb{P}\bigl(\|\Delta\tilde X_{t_k}\|_2>\Delta_{\mathrm{hys}}\bigr)
\le
\frac{\E\big[\|\Delta\tilde X_{t_k}\|_2^2\big]}{\Delta_{\mathrm{hys}}^2}
\le
\frac{C_2}{\Delta_{\mathrm{hys}}^2}.
\]
Thus, for each $k$,
\[
\mathbb{P}\bigl(|h_{t_k}-h_{t_k-1}|>\varepsilon\bigr)
\le
\mathbb{P}\bigl(\|\Delta\tilde X_{t_k}\|_2>\Delta_{\mathrm{hys}}\bigr)
\le
\frac{C_2}{\Delta_{\mathrm{hys}}^2}.
\]

To translate this into a statement about the fraction of steps with large changes, define the indicator of a large jump at time $t$ by
\[
I_t
:=
\mathbf{1}\bigl\{|h_t-h_{t-1}|>\varepsilon\bigr\}.
\]
We know that $I_t=0$ whenever $t$ is not of the form $t_k=kH$. For a given horizon $T$, the fraction of steps with large changes is
\[
\frac{1}{T}\sum_{t=1}^T I_t
=
\frac{1}{T}\sum_{k:\,t_k\le T} I_{t_k}.
\]
Taking expectations and using the bound above,
\[
\begin{aligned}
\E\Bigl[\frac{1}{T}\sum_{t=1}^T I_t\Bigr]
&=
\frac{1}{T}\sum_{k:\,t_k\le T}\mathbb{P}\bigl(|h_{t_k}-h_{t_k-1}|>\varepsilon\bigr)\\[0.3em]
&\le
\frac{1}{T}\sum_{k:\,t_k\le T}\frac{C_2}{\Delta_{\mathrm{hys}}^2}\\[0.3em]
&\le
\frac{1}{T}\,\Bigl\lceil\frac{T}{H}\Bigr\rceil\,\frac{C_2}{\Delta_{\mathrm{hys}}^2}.
\end{aligned}
\]
For $T\ge H$, $\lceil T/H\rceil\le 2T/H$, so for all sufficiently large $T$ we obtain
\[
\E\Bigl[\frac{1}{T}\sum_{t=1}^T I_t\Bigr]
\le
\frac{2C_2}{H\,\Delta_{\mathrm{hys}}^2}.
\]
The right-hand side can be made arbitrarily small by choosing $H$ and $\Delta_{\mathrm{hys}}$ sufficiently large. Concretely, given any $\delta>0$, one can pick $H$ and $\Delta_{\mathrm{hys}}$ so that
\[
\frac{2C_2}{H\,\Delta_{\mathrm{hys}}^2}\le\delta,
\]
and then for all large horizons $T$ the expected fraction of steps with changes larger than $\varepsilon$ is at most $\delta$. In particular, the sequence of hyperparameters does not ``chatter'': large jumps occur only rarely, and their average frequency can be made as small as desired by making the scheduler less aggressive (large macro-interval $H$) and the hysteresis threshold more conservative (large $\Delta_{\mathrm{hys}}$).

Since the same reasoning applies to each of the five scheduled processes $(\eta_t,\tau_t,\lambda_t,D_t,B_t)$, the theorem follows.
\end{proof}

\end{document}